\documentclass[nohyperref]{article}

\usepackage{microtype}
\usepackage{graphicx}
\usepackage{booktabs} 

\usepackage{hyperref}

\usepackage[accepted]{icml2022}

\usepackage{amsmath}
\usepackage{amssymb}
\usepackage{mathtools}
\usepackage{amsthm}
\usepackage{verbatim}

\usepackage{multirow}
\usepackage{caption}
\usepackage{subcaption}

\usepackage{kotex}

\usepackage{amssymb}
\usepackage{pifont}

\usepackage[capitalize,noabbrev]{cleveref}

\theoremstyle{plain}
\newtheorem{theorem}{Theorem}[section]
\newtheorem{proposition}[theorem]{Proposition}
\newtheorem{lemma}[theorem]{Lemma}

\theoremstyle{definition}

\theoremstyle{remark}

\usepackage[textsize=tiny]{todonotes}

\icmltitlerunning{Robust Imitation Learning against Variations in Environment Dynamics}

\begin{document}

\twocolumn[
\icmltitle{Robust Imitation Learning against Variations in Environment Dynamics}




\begin{icmlauthorlist}
\icmlauthor{Jongseong Chae}{kaist}
\icmlauthor{Seungyul Han}{unist}
\icmlauthor{Whiyoung Jung}{kaist}
\icmlauthor{Myungsik Cho}{kaist}
\icmlauthor{Sungho Choi}{kaist}
\icmlauthor{Youngchul Sung}{kaist}

\end{icmlauthorlist}

\icmlaffiliation{kaist}{School of Electrical Engineering, KAIST, Daejeon, South Korea.}
\icmlaffiliation{unist}{Artificial Intelligence Graduate School, UNIST, Ulsan, South Korea}

\icmlcorrespondingauthor{Seungyul Han}{syhan@unist.ac.kr}

\icmlkeywords{Machine Learning, Imitation Learning, Robust Imitation Learning, Multiple Environments Setting, Robust Reinforcement Learning}

\vskip 0.3in
]



\printAffiliationsAndNotice{}  

\begin{abstract}
In this paper, we propose a robust imitation learning (IL) framework that improves the robustness of IL when environment dynamics are perturbed. The existing IL framework trained in a single environment can catastrophically fail with perturbations in environment dynamics because it does not capture the situation that underlying environment dynamics can be changed. Our framework effectively deals with environments with varying dynamics by imitating multiple experts in sampled environment dynamics to enhance the robustness in general variations in environment dynamics. In order to robustly imitate the multiple sample experts, we minimize the risk with respect to the Jensen-Shannon divergence between the agent's policy and each of the sample experts. Numerical results show that our algorithm significantly improves robustness against dynamics perturbations compared to conventional IL baselines.
\end{abstract}

\section{Introduction}
\label{section:introduction}

Reinforcement Learning (RL) is a framework that produces optimal policies for  tasks. Deep neural networks enable RL to handle complex tasks in various simulation environments \cite{rl2:DQN, rl3:DDPGs, rl4:TD3, rl5:SAC, RL6:TRPO, RL7:PPO}. However, current RL still has limitations for deployment into the real world. Two of the main  limitations are \textit{robustness} and \textit{design of reward function}. A typical RL algorithm interacts with a single environment and evaluates the policy with the interaction environment, so the policy becomes specialized to the trained environment and mostly fails when the underlying dynamics are perturbed from the trained environment. In the real world, underlying dynamics are highly likely to be perturbed. For example, consider autonomous driving with RL. The physical dynamics of an autonomous driving car including handling, braking, the road friction coefficient of a rainy day change from those of a clear day. 
To cope with such uncertainty, one could consider learning an expert policy for each of all possible environment dynamics for a given task like car driving, estimating the realized dynamics, and using one of the learned expert policies for the estimated dynamics.  However, estimating  the realized environment dynamics is difficult   because the dynamics of the environment depend on many correlated environment parameters such as gravity, mass, aging, etc. Furthermore, learning a policy for each of all possible dynamics perturbations is infeasible when the dynamics vary continuously.

Robust Reinforcement Learning (Robust RL) is a framework that produces a robust policy against such environment perturbations for a given task.  The aim is to learn a policy that works  well in all possible dynamics perturbations  without estimating the  perturbation. Typical  robust RL  allows the agent to interact in multiple environments \cite{robust3:worst1,robust4:worst2,robust5:worst3} and the policy optimizes the worst case of the expected returns in the multiple interaction environments. This agent can work well both in all the interaction environments and even in an unseen environment with similar dynamics.  Even if such dynamics variation can be handled by robust RL, there still remains the issue of reward function design for many real-world control problems including our example of autonomous driving, since robust RL relies on a well-designed  reward function.
When we observe a human drive, it is difficult to know what reward the driver has for each of the driver's actions. 

Imitation Learning (IL) has been developed to cope with such situations by learning a policy for a given task without a reward function \cite{il1:BC,il3:IRL,il4:Apprentice}. IL uses demonstrations generated from an expert for the task instead of a reward function, and the agent tries to mimic the expert.   GAIL is one of the popular IL algorithms and tries to mimic an expert by matching occupancy measure, which is the unnormalized distribution of state-action pairs  \cite{il5:GAIL}. 
Up to now, however, most IL algorithms have been proposed for a single interaction environment with perfect or non-perfect expert demonstration to yield a policy that is specialized   to the single interaction environment.

In this paper, we consider robust IL to learn a robust policy by IL against continuous environment dynamics perturbation and propose a novel IL framework to learn  a robust policy performing well  over a range of continuous dynamics variation based on demonstrations only at a few sampled dynamics from the continuum, which does not require demonstrations from all the continuum and thus significantly reduces the amount of required  demonstrations. 
The detail of the proposed framework will follow in the upcoming sections.

\section{Related Works}
\label{section:related_works}

\textbf{Imitation Learning:} IL aims to learn a policy by imitating an expert. 
Behavior Cloning (BC) \cite{il1:BC} is an approach of IL  based on   supervised learning. 
\citet{il2:DRIL} alleviated the covariate shift problem of BC. 
Another approach is adversarial imitation learning \cite{il5:GAIL,il6:GAIFO} in which the agent imitates an expert by matching a positive measure.
\citet{il7:AIRL} recovered the reward function using expert demonstration. 
Cross-domain IL \cite{il8:I2L,il24:state_alignment,il25:IRL_dynamics_mismatch,il26:CDIfO} considered the IL problem under dynamics mismatch between the expert and the learner.

The existing robust IL works addressed the IL problem with non-perfect demonstrations \cite{il12:robustIL4,il11:robustIL3} or improved the stability of IL \cite{il9:robustIL1,il10:robustIL2}, and their settings are different from our setting in this paper.
Meta-IL \cite{il13:metaIL1,il14:metaIL2,il15:metaIL3,il16:metaIL4} and Meta-IRL \cite{il17:metaIRL1,il18:metaIRL2} can learn a new task using a few demonstrations by leveraging experiences from similar tasks, whereas our framework doesn't require any demonstrations for test tasks.  
Multi-task IRL \cite{il19:multitaskIRL} proposed a Maximum Causal Entropy IRL framework for multi-task IRL and meta-learning to infer multiple reward functions for each task. \citet{il20:multitaskILsuite} proposed a multi-task benchmark suite for evaluating the robustness of IL algorithms. 
ADAIL \cite{il23:ADAIL} can learn an adaptive policy for environments of varying dynamics, but it assumed that collecting expert demonstrations in multiple
environments is infeasible and used many simulation environments for domain randomization and
environment encoding.

\textbf{Robust Reinforcement Learning:}  Robust RL produces a robust policy over environment perturbations. Robust-MDP \cite{robust1:mdp1,robust2:mdp2} extends uncertainty transition set on  MDP. \citet{robust3:worst1,robust4:worst2,robust5:worst3} estimated the worst case of the expected return among multiple perturbed environments. \citet{robust6:rarl1} addressed the Robust RL problem by using the adversary. \citet{robust7:rarl2,robust8:rarl3} formalized criteria of robustness to action uncertainty.

\section{Preliminaries}
\label{section:preliminaries}

\subsection{Markov Decision Process}
\label{subsection:markov_decision_process}

An MDP is denoted by a tuple $\mathcal{M}=<\mathcal{S},\mathcal{A},\mathcal{P}, r, \gamma>$, where $\mathcal{S}$ is the state space, $\mathcal{A}$ is the action space, $\mathcal{P} : \mathcal{S}\times\mathcal{A}\times\mathcal{S}\to\mathbb{R}^+$ is the transition probability, $r : \mathcal{S}\times\mathcal{A}\to\mathbb{R}$ is the reward function, and  $\gamma\in(0,1)$ is the discount factor. A policy $\pi$ is a (stochastic) mapping $\pi:\mathcal{S}\mapsto\mathcal{A}$. The return $G_t$ is a discounted cumulative sum reward from time step $t$, i.e., $G_t=\sum_{i=t}^{\infty}\gamma^{i-t}r(s_i,a_i)$. The goal  is to learn a policy $\pi$ to maximize the expected return $J(\pi)=\mathbb{E}_{s_0\sim \mu_0, \tau \sim \pi} \left[G_0\right]$ \cite{rl1:sutton}, where $\tau = \{s_0,a_0,s_1,a_1,\ldots\}$ is an episode trajectory and $\mu_0(s)$ denotes the initial state distribution. The occupancy measure $\rho_{\pi}(s,a)=\pi(a|s)\sum_{t=0}^{\infty}\gamma^t\text{Pr}(s_t=s|\pi, \mathcal{P})$ is the unnormalized state-action distribution induced by policy $\pi$, and $\mu_{\pi}(s)=\sum_{t=0}^{\infty}\gamma^t\text{Pr}(s_t=s|\pi, \mathcal{P})$ is the unnormalized state distribution induced by policy $\pi$.

\subsection{Generative Adversarial Imitation Learning}
\label{subsection:generative_adversarial_imitation_learning}

In IL, the agent does not receive a reward for its action. Instead,  the agent learns a policy based on the demonstration of an expert without knowing the explicit expert policy. 
Typically, expert demonstration is  given as a trajectory generated by the expert's policy, $\tau_{E} = \{s_0,a_0,s_1,a_1,\ldots\}$. Generative adversarial imitation learning (GAIL) \cite{il5:GAIL} is one of the popular IL methods using expert demonstration. Based on \cref{proposition1}, GAIL seeks a policy of which occupancy measure is close to that of the expert so that the agent's policy $\pi$ is close to the expert's policy $\pi_E$.

\begin{proposition}[Theorem 2 of \citet{il4:Apprentice} \& Proposition 3.1 of \citet{il5:GAIL}]  \label{proposition1}
In a single environment, the occupancy measure $\rho_{\pi}(s,a)$ satisfies the following Bellman flow constraint for each $(s,a)\in\mathcal{S}\times \mathcal{A}$:
{\scriptsize\begin{align}
    \rho_{\pi}(s,a)=\mu_0(s)\pi(a|s)+\gamma\int_{(s',a')}\mathcal{P}(s|s',a')\rho_{\pi}(s',a')\pi(a|s)
    \label{eq:bellman_flow_constraints2}
\end{align}
}and  the policy $\pi$ whose  occupancy measure is $\rho_{\pi}$ is unique. That is,  the occupancy measure and the policy are in an \textbf{one-to-one} relationship.
\end{proposition}

The policy $\pi$ induces the occupancy measure $\rho_{\pi}$, and $\rho_{\pi}$ maps to the unique policy $\pi$. Therefore, GAIL reproduces the expert's policy from the policy update \eqref{eq:gail_algorithm_objective}, which  matches the occupancy measures of the agent's policy and the expert's policy:
{\small
\begin{align}
 &\min_{\pi}\mathcal{D}_{JS}(\bar{\rho}_{\pi},\bar{\rho}_{\pi_E})       \label{eq:gail_algorithm_objective}\\
         &\overset{(a)}{=}\min_\pi \mathbb{E}_{\rho_{\pi_E}}\left[\log \frac{{\rho}_{\pi_E}}{{\rho}_\pi+\rho_{\pi_E}}\right]+\mathbb{E}_{\rho_{\pi}}\left[\log\frac{{\rho}_{\pi}}{\rho_\pi+\rho_{\pi_E}}\right] \nonumber\\
        &\overset{(b)}{=}\min_{\pi}\max_D\mathbb{E}_{\rho_{\pi_E}}\left[\log D(s,a)\right]+\mathbb{E}_{\rho_{\pi}}\left[\log(1-D(s,a))\right]\nonumber
\end{align}
}where $\mathcal{D}_{JS}$ denotes the Jensen-Shannon (JS) divergence, and $\bar{\rho}_\pi$ and $\bar{\rho}_{\pi_E}$ are the normalized occupancy distributions from ${\rho}_\pi$ and ${\rho}_{\pi_E}$, respectively. Here, (a) is valid since the constant normalizer is irrelevant in minimization, and  (b) is valid because the maximizing $D$ value is given by  $D(s,a)=\frac{\rho_{\pi_E}(s,a)}{\rho_\pi(s,a)+\rho_{\pi_E}(s,a)}$, where  discriminator $D$ distinguishes whether a given  pair $(s,a)$ is from expert or not.

{\bf Gradient Penalty:} A variant of GAIL \cite{il21:DAC} uses the gradient penalty (GP) proposed by \citet{GP:improvedWGAN} as a regularization term to enhance the stability of IL. The discriminator update of GAIL with GP is given by{\small
\begin{align}
    \label{eq:GAIL+GP}
    &\max_D \mathbb{E}_{\rho_{\pi_E}}\left[\log D(s,a)\right]+\mathbb{E}_{\rho_{\pi}}[\log(1-D(s,a))] \\
    &\hspace{0.7cm}+ \kappa\mathbb{E}_{\hat{x}}\left(\|\nabla_{\hat{x}} D(\hat{x})\|_2-1\right)^2\nonumber,
\end{align}
}where $\hat{x}\sim(\epsilon\rho_{\pi}+(1-\epsilon)\rho_{\pi_E})$ with $\epsilon\sim\text{Uniform}[0,1]$, and $\kappa$ is the regularization coefficient to control the GP term. We will call this GAIL+GP.

\section{Motivation}
\label{subsection:motivation_for_using_multiple_environments}

The existing IL methods typically interact with a single nominal environment and try to imitate an expert  that is specialized at the single nominal environment. 
For further discussion, we define three types of environment: the \emph{interaction} environment, the \emph{demonstration} environment and the \emph{test} environment. The interaction environment is the one with which the agent interacts to obtain policy samples during the training,  the  demonstration environment is the one from which the expert demonstration is generated to  train the agent, and the test environment is the actual test environment for the trained agent policy. The interaction environment and the demonstration environment are the same for conventional IL with a single nominal environment (SNE). 
We will refer to this IL training setting as SNE/SNE (interaction environment / demonstration environment). In most cases, IL trained in the SNE/SNE setting fails  when the  actual test environment dynamics are perturbed from the nominal dynamics, as seen in \cref{figure:gail_wg,figure:gail_wm}.  In \cref{figure:gail_wg,figure:gail_wm}, the x-axis value denotes the ratio (in percentage) of the gravity (or mass) of the test environment to that of the nominal interaction/demonstration environment and the y-axis shows the mean return of the policy trained under the SNE/SNE setting at the corresponding $x$ value. It is seen that the performance degrades severely as the test environment dynamics deviate from the nominal interaction/demonstration environment dynamics.

\begin{figure}[t]
    \centering
    \begin{subfigure}[b]{0.48\linewidth}
        \centering
        \includegraphics[width=\textwidth]{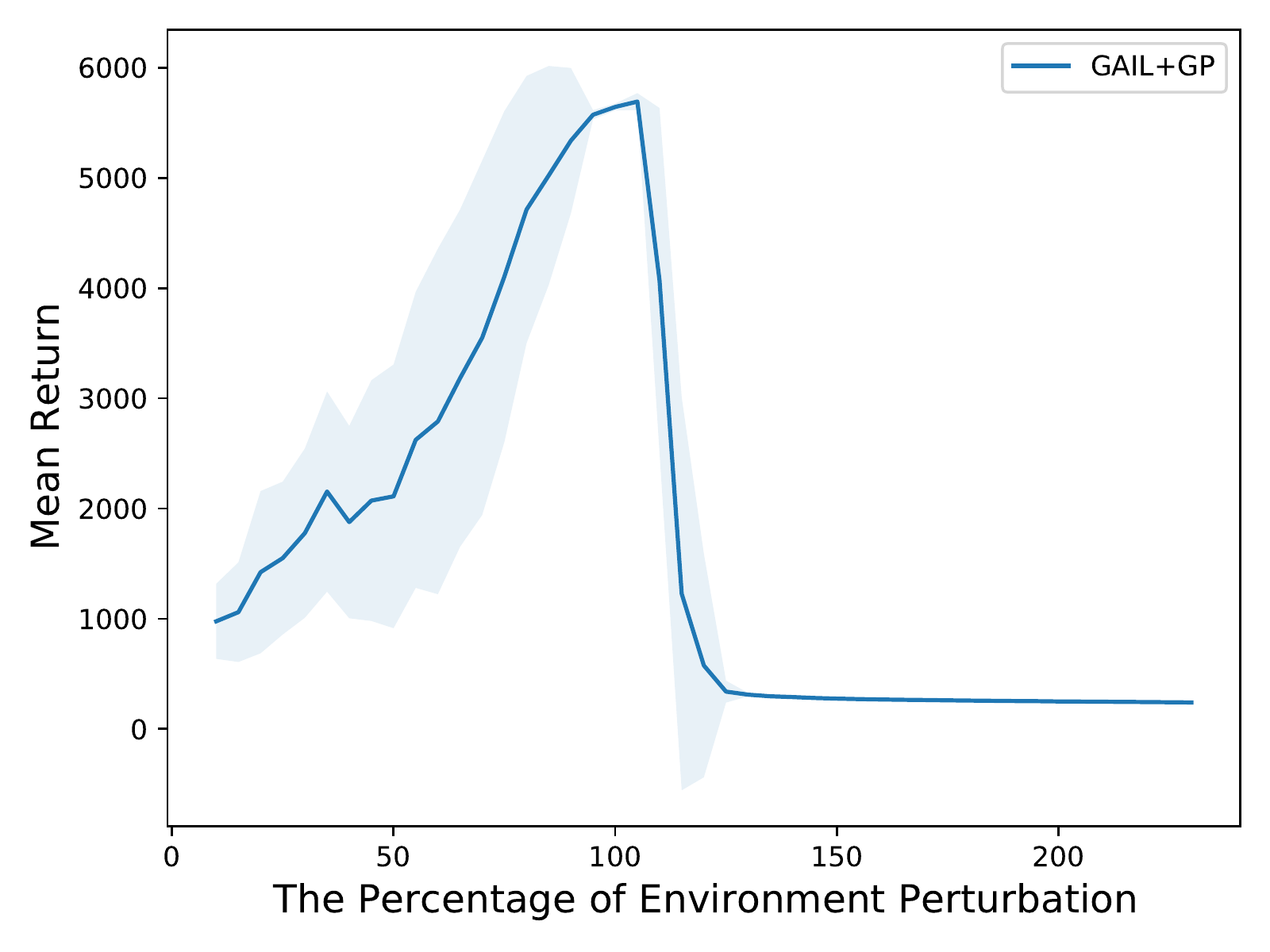}
        \vskip -0.05in
        \captionsetup{justification=centering}
        \caption{Walker2d+Gravity:\\GAIL+GP\label{figure:gail_wg}}
    \end{subfigure}
    \begin{subfigure}[b]{0.48\linewidth}
        \centering
        \includegraphics[width=\textwidth]{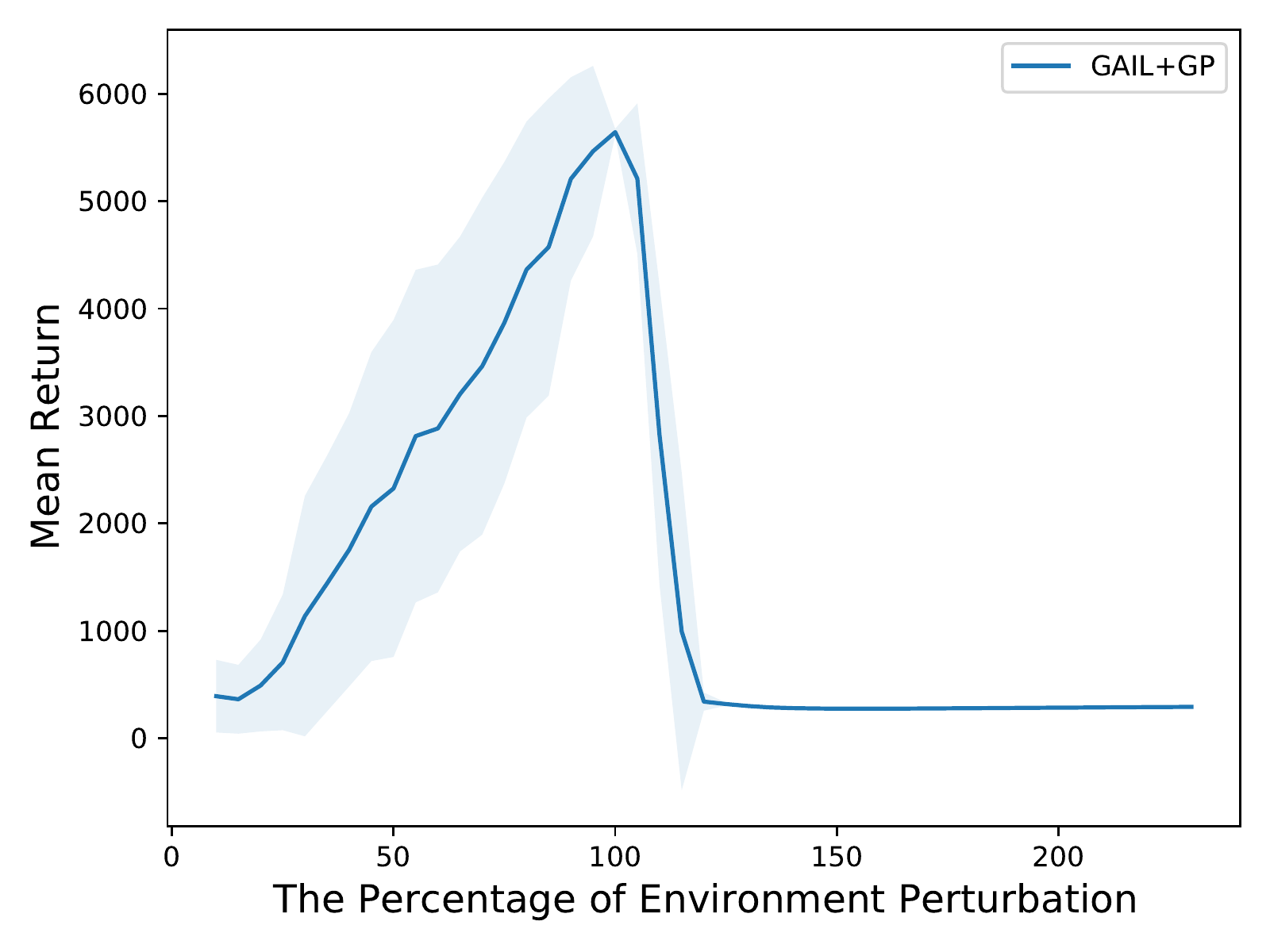}
        \vskip -0.05in
        \captionsetup{justification=centering}
        \caption{Walker2d+Mass:\\GAIL+GP\label{figure:gail_wm}}
    \end{subfigure}
    \begin{subfigure}[b]{0.48\linewidth}
        \centering
        \includegraphics[width=\textwidth]{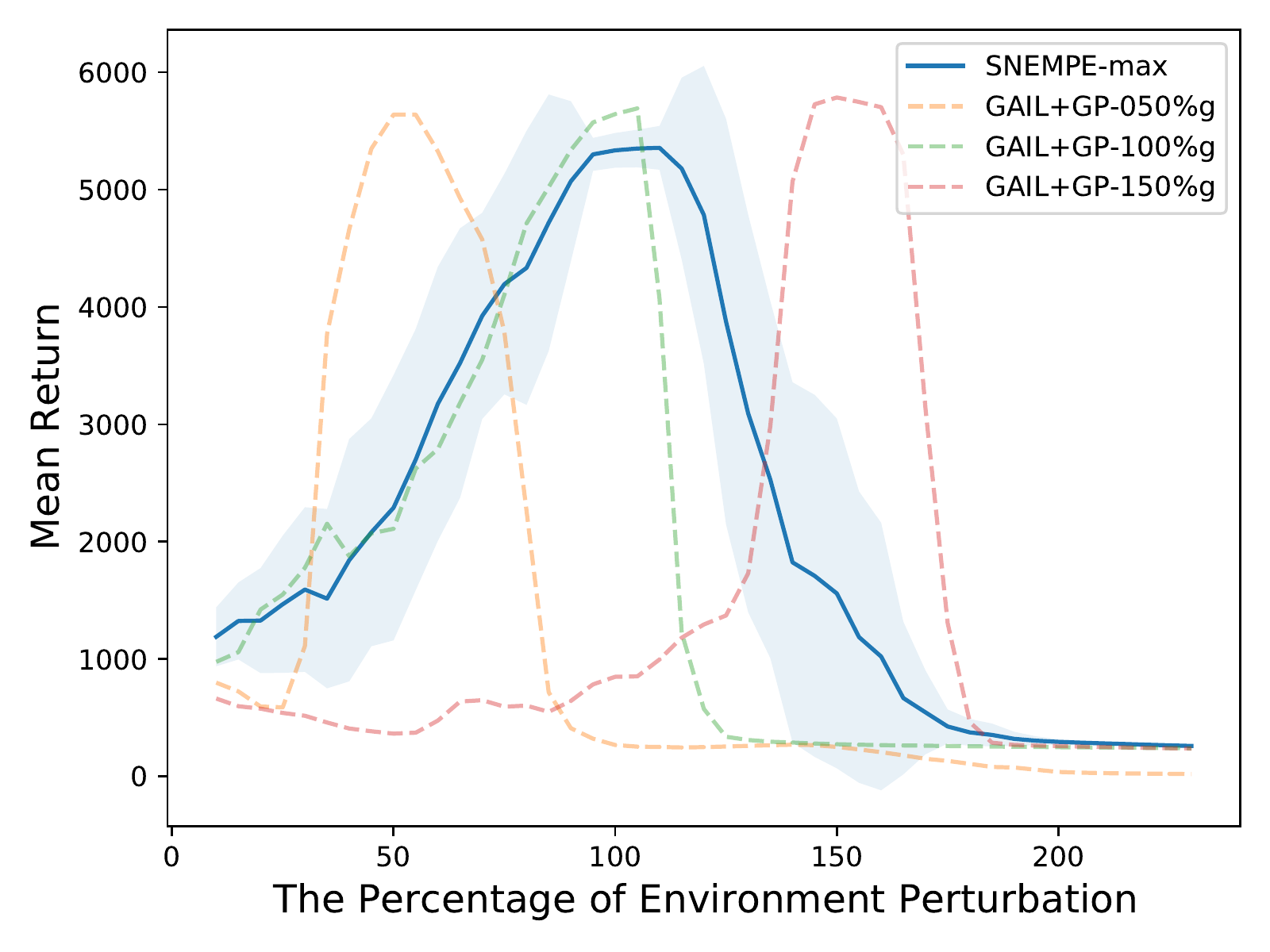}
        \vskip -0.05in
        \captionsetup{justification=centering}
        \caption{Walker2d+Gravity:\\SNEMPE-max \label{figure:snmpe_wg}}
    \end{subfigure}
    \begin{subfigure}[b]{0.48\linewidth}
        \centering
        \includegraphics[width=\textwidth]{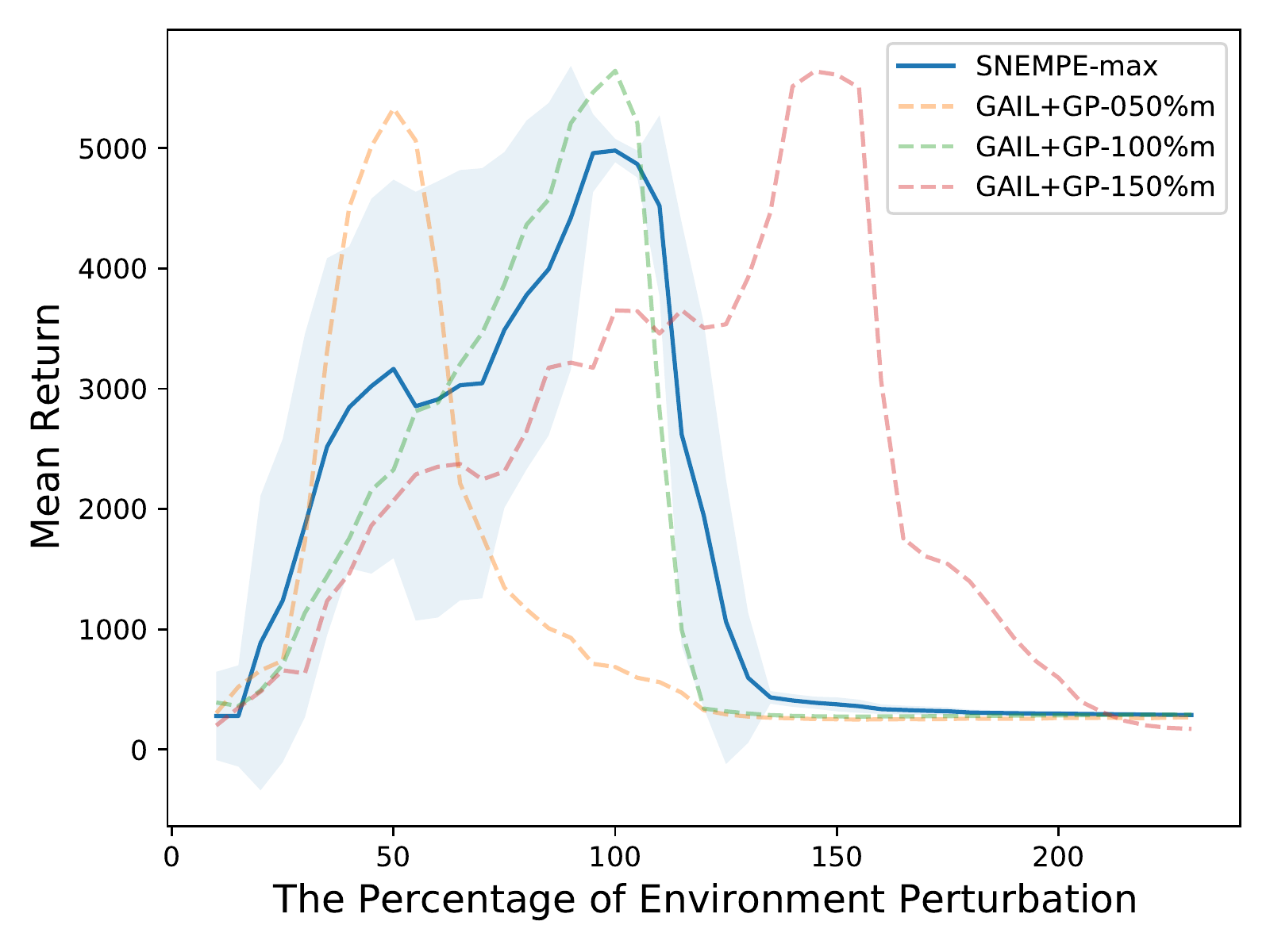}
        \vskip -0.05in
        \captionsetup{justification=centering}
        \caption{Walker2d+Mass:\\SNEMPE-max \label{figure:snmpe_wm}}
    \end{subfigure}
    \vskip -0.1in
    \caption{Performance of the policy trained in a single nominal environment against perturbation (left column -  gravity perturbation and  right column - mass perturbation): (a,b) - performance of GAIL+GP, and (c,d) - SNEMPE-max. \label{figure:figure1_results}} \vskip -0.2in
\end{figure}
To handle such performance degradation, robust RL samples a few environments with perturbed dynamics. For example, in \cref{figure:snmpe_wg}, three environments with gravity 50\%, 150\% and nominal  100\% are sampled.  Then, robust RL allows the agent to interact with these  multiple sampled environments (MPE) so that the agent's policy can capture the various dynamics of the multiple environments. Then, robust RL typically solves  $\max_{\pi}\min_{\mathcal{P}^i\in P}\mathbb{E}_{\pi}[G_t|\mathcal{P}^i]$, where $P=\{\mathcal{P}^i\}$ is the selected environment set.  By maximizing the worst-case  expected return, the agent's policy can capture the varying dynamics in the selected environment set $P=\{\mathcal{P}^i\}$.  However,
robust RL requires a well-designed reward function, which we want to avoid.  

Now, consider robust IL. One simple approach is to apply the above robust RL principle to the IL setting. Here, we obtain  expert demonstrations from multiple sampled demonstration environments  and have a single policy interacting with the single nominal interaction environment. Then, we use discriminators to distinguish the policy samples from each of the multiple sampled expert demonstrations, and train the policy to follow the worst-case, i.e., the expert demonstration that is farthest from the policy sample based on the discriminator outputs. The performance of so learned policy in the perturbed test environment is shown in \cref{figure:snmpe_wg,figure:snmpe_wm} (the corresponding performance is denoted as SNEMPE-max). It is seen that the policy learned in such way improves robustness compared with conventional SNE/SNE IL in \cref{figure:gail_wg,figure:gail_wm}, but the performance is not satisfactory. This degradation implies that policy interaction with the single nominal environment is not enough to capture the dynamics variation even with expert demonstrations from multiple sampled demonstration environments.  Thus, in order to fully capture the dynamics variation, we first sample a few environments with different dynamics from the continuous dynamics distribution and use these multiple sampled environments not only for expert demonstrations but also for policy interaction during the training. We refer to this setting as the MPE/MPE IL setting. In the remainder of this paper, we propose an efficient IL framework based on the MPE/MPE IL setting to yield a policy that  performs robustly against continuous environment dynamics variation based only on a few sampled dynamics for training.

\section{Robust Imitation Learning against Variations in Environment Dynamics}
\label{section:robust_imitation_learning_in_multiple_environments_setting}

\subsection{Problem Formulation}
\label{subsection:problem_setting}

We consider a collection of MDPs $\mathcal{C}=\{\mathcal{M}=<\mathcal{S},\mathcal{A},\mathcal{P}_\zeta,$ $r, \gamma>, ~\zeta \in Z\}$, where the state and action spaces are the same for all members of the collection, the reward is unavailable to the agent,  the transition probability modeling the dynamics is parameterized with  parameter $\zeta$, and the dynamics parameter $\zeta$ can be continuously varied or perturbed from the nominal value $\zeta_0$ within the set $Z$. Among this continuous collection, we sample $N$ MDPs  with dynamics parameters $\zeta_1,\zeta_2,\cdots,\zeta_N$. We denote these $N$ environments with dynamics $\mathcal{P}_{\zeta_1},\cdots,\mathcal{P}_{\zeta_N}$ (simply denoted as $\mathcal{P}^1,\cdots,\mathcal{P}^N$) by $\mathcal{E}_1,\cdots,\mathcal{E}_N$.  We assume that there exists an expert $\pi_E^i$ for each environment $\mathcal{E}_i$, the expert $\pi_E^i$ generates expert demonstration for the agent, but the expert policy $\pi_E^i$  itself is not available to the agent. We also assume that the agent can interact with each of all sampled environments $\mathcal{E}_1,\cdots,\mathcal{E}_N$, and  the initial state distributions of all interaction environments are the same as $\mu_0(s)$.   Thus, according to our definition in the previous section, $\mathcal{E}_1,\cdots,\mathcal{E}_N$ are both demonstration and interaction environments, and the setting is MPE/MPE. Note that the expert demonstrations at $\mathcal{E}_1,\cdots,\mathcal{E}_N$
are partial information about the entire MDP collection $\mathcal{C}$. Our goal is for the agent to learn a policy $\pi$ that performs well for all members in the MDP collection $\mathcal{C}$ based only  on the expert demonstrations from and agent interaction with the sampled environments $\mathcal{E}_1,\cdots,\mathcal{E}_N$.  We will refer to this problem as {Robust Imitation learning with Multiple perturbed Environments} (RIME).

Let us introduce a few more notations. 
$\rho_{\pi}^i(s,a)=\pi(a|s)\sum_{t=0}^{\infty}\gamma^t\text{Pr}(s_t=s|\pi, \mathcal{P}^i)$ denotes  the occupancy measure of $\pi$ in the $i$-th interaction environment $\mathcal{E}_i$. $\mu_{\pi}^i(s)=\sum_{t=0}^{\infty}\gamma^t\text{Pr}(s_t=s|\pi, \mathcal{P}^i)$ denotes the unnormalized state marginal of $\pi$ in the $i$-th interaction environment $\mathcal{E}_i$. For simplicity, we denote $\rho_{\pi_E^j}^j(s,a)$ and $\mu_{\pi_E^j}^j(s)$ by $\rho_E^j(s,a)$ and $\mu_E^j(s)$, respectively. 
The  expert demonstration $\tau_E^i$ is given by the state-action pair trajectory from expert policy $\pi_E^i$ specialized in the $i$-th demonstration environment $\mathcal{E}_i$ with dynamics $\mathcal{P}^i$.  $D_{ij}(s,a):\mathcal{S}\times\mathcal{A}\to[0,1]$ is a discriminator that distinguishes whether a state-action pair $(s,a)$ is from policy $\pi$ interacting with $\mathcal{E}_i$ or from expert $\pi_E^j$.

\begin{figure}[t]
    \centering
    \includegraphics[width=\columnwidth]{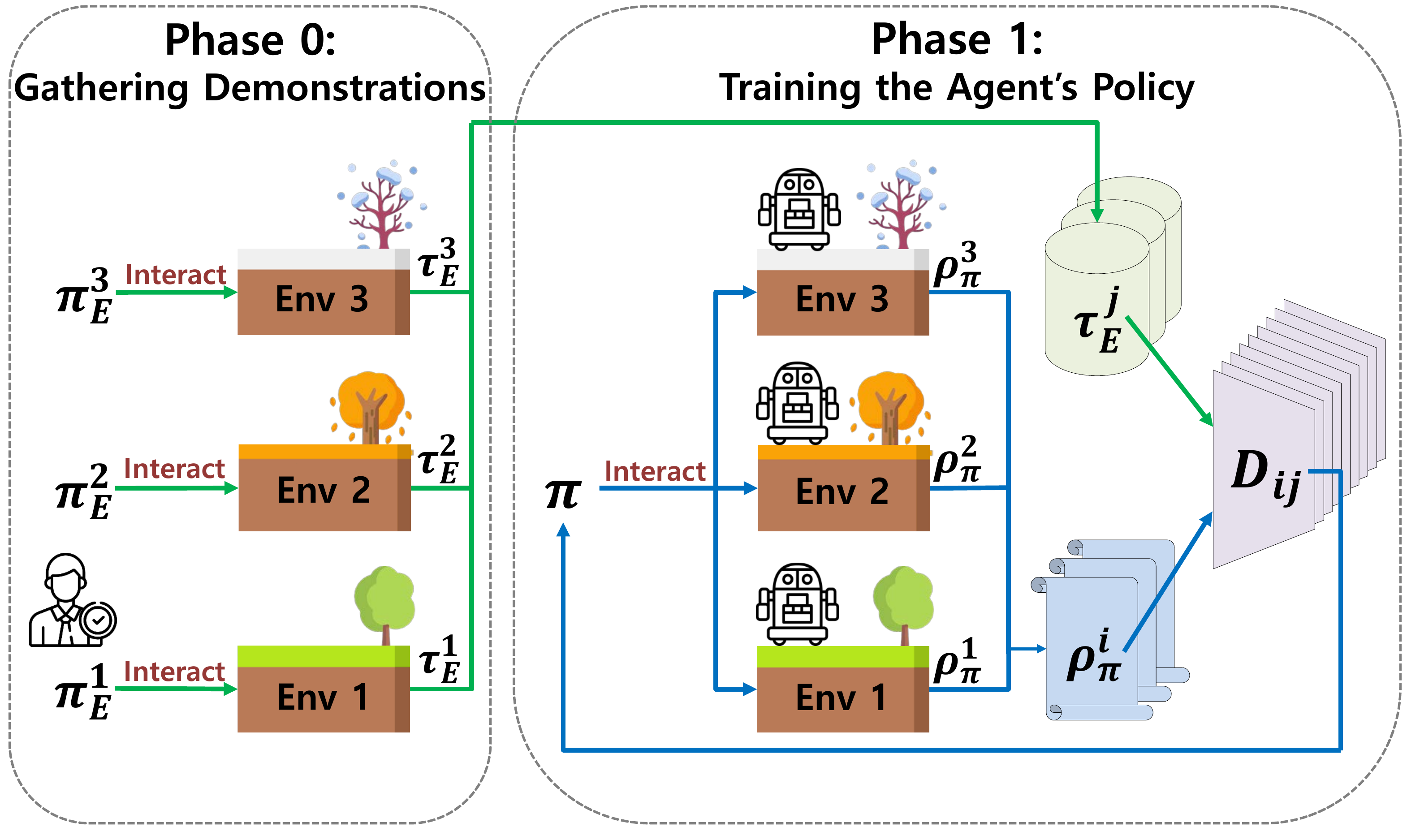}
    \vskip -0.1in
    \caption{Overview of our algorithm. The blue line is the flow of policy samples $\rho_{\pi}^i$, and the green line is the flow of expert demonstrations $\tau_E^j\sim\rho_E^j$. \label{figure:overview}}
    \vskip -0.15in
\end{figure}

\subsection{Direct Optimization in the  Policy Space}
\label{subsection:directly_optimizing_on_policy_space}

In order to solve the RIME problem,
one can consider the occupancy matching technique which is used in GAIL. As mentioned in \cref{subsection:generative_adversarial_imitation_learning}, in the single environment setting, the occupancy measure $\rho_{\pi}$ satisfies the Bellman flow constraint \eqref{eq:bellman_flow_constraints2}, and there exists a one-to-one mapping between  the occupancy measure and the policy.  By this relationship, the agent's policy can imitate the expert by matching its occupancy measure close to that of the expert.
In the multiple environment setting, however, the situation is not so simple as in the single environment case. Suppose that the agent policy $\pi$  interacts uniformly with $N$ environments $\mathcal{E}_1,\cdots,\mathcal{E}_N$ with the same state-action space but different transition probabilities $\mathcal{P}^1,\cdots,\mathcal{P}^N$.
Then, the occupancy measure of $\pi$  becomes the mixture, i.e.,  $\rho_{\pi}=\frac{1}{N}\sum_{i=1}^{N}\rho_{\pi}^i$, and the corresponding Bellman flow equation is given by
\begin{align}
    \label{eq:multi_environment_occupancy_measure}
    &\rho_{\pi}(s,a)=\frac{1}{N}\sum_{i=1}^{N}\rho_{\pi}^i(s,a)=\mu_0(s)\pi(a|s) \nonumber\\
    & ~~~~~+\gamma \frac{1}{N}\sum_{i=1}^{N}\int_{(s',a')}\mathcal{P}^i(s|s',a')\rho_{\pi}^i(s',a')\pi(a|s).
\end{align}
There exists a distinct characteristic in (\ref{eq:multi_environment_occupancy_measure}) from the single-environment equation    (\ref{eq:bellman_flow_constraints2}). 
For simplicity of exposition, suppose that the state space $\mathcal{S}$ and the action space $\mathcal{A}$ are discrete and finite with cardinalities $|\mathcal{S}|$ and $|\mathcal{A}|$, respectively. In the case of (\ref{eq:bellman_flow_constraints2}), we have a linear system of equations with $|\mathcal{S}||\mathcal{A}|$ unknowns $\rho_\pi(s,a),~(s,a)\in \mathcal{S}\times \mathcal{A}$ and $|\mathcal{S}||\mathcal{A}|$ equations. Hence, we have a unique solution $\rho_\pi(s,a)$ if the kernel $\mathcal{P}(s|s',a')$ satisfies certain Markov chain conditions.  In the case of (\ref{eq:multi_environment_occupancy_measure}), on the other hand, we have $N|\mathcal{S}||\mathcal{A}|$ unknowns $\rho^i_\pi(s,a),~i=1,\cdots,N$ but $|\mathcal{S}||\mathcal{A}|$ equations. So, the system is underdetermined,  there exist  infinitely many solutions for the set $\{\rho^i_\pi(s,a),~i=1,\cdots,N\}$, and hence the mixture $\rho_\pi = (1/N)\sum_i \rho_\pi^i$ can be infinitely many. Thus, the mapping from $\pi$ to $\rho_\pi$ can be one-to-many, so there is no guarantee to recover $\pi$ from $\rho_\pi$ unless we prove $\{\rho_\pi\} \cap \{\rho_{\pi'}\} = \emptyset, ~\forall \pi,\pi'$ such that $\pi \ne \pi'$. Hence, there is no guarantee for  policy recovery from  occupancy measure matching, and  we need to consider a new approach to the RIME problem.

Our approach is not to use  the occupancy measure as in GAIL but to use  the policy distribution itself.  For the considered MPE/MPE setting, we propose the following objective function to solve the RIME problem:
{\small
\begin{align}
    \label{eq:objective_for_RIME}
    \min_{\pi}\mathbb{E}_{s\sim\frac{1}{N}\sum_{i=1}^{N}\mu_{\pi}^i}\left[\:\sum_{j=1}^{N}\lambda_j(s)\cdot \mathcal{D}(\pi(\cdot|s), \pi_E^j(\cdot|s))\:\right],
\end{align}
}where  $\mathcal{D}$ is some divergence between two policy distributions, and $\sum_j\lambda_j(s)=1$. The objective function (\ref{eq:objective_for_RIME}) means that we want to design the agent policy $\pi$ to 
appropriately
imitate all expert policies $\pi_E^1,\cdots,\pi_E^N$ on the state samples generated by the agent policy  interacting with all interaction environments. Here, $\lambda_j(s)$ is the weight to determine how much $\pi_E^j(\cdot|s)$ is imitated.  Such an objective has been considered for \emph{integration of expert machines} \cite{AmariBook} and is well suited to our purpose.  The key difference between (\ref{eq:objective_for_RIME}) and (\ref{eq:gail_algorithm_objective}) is that in (\ref{eq:gail_algorithm_objective}), the distance between the occupancy measures of  the agent and the expert is minimized based on \cref{proposition1}, whereas in (\ref{eq:objective_for_RIME}) the distance between the policy distribution of the agent and those of the multiple experts is minimized, not requiring the occupancy measures. However, the key challenge to the objective function   (\ref{eq:objective_for_RIME}) is that the expert policies $\pi_E^1,\cdots, \pi_E^N$ are not available but only their demonstrations are at hand. The following theorem is the first step to  circumvent this difficulty. 

\begin{theorem}
    \label{theorem1}
    If $\rho_{\pi}^i(s,a)>0$, $\lambda_j(s)>0$ for any $i,j\in \{1,\cdots,N\}$, $\gamma\in(0,1)$, and $\mathcal{D}$ in \eqref{eq:objective_for_RIME} is the Jensen-Shannon divergence, then the objective function \eqref{eq:objective_for_RIME} is expressed as
    {\small
    \begin{align}
        \label{eq:result_theorem_for_RIME}
        \min_{\pi}\sum_{i=1}^{N}\sum_{j=1}^{N}\max_{D_{ij}}\Big\{\mathbb{E}_{(s,a)\sim\rho_{\pi}^i}\left[\frac{\lambda_j(s)}{2N}\log(1-D_{ij}(s,a))\right]\nonumber\\
        +\mathbb{E}_{s\sim\mu_{\pi}^i,a\sim\pi_{E}^j}\left[\frac{\lambda_j(s)}{2N}\log(D_{ij}(s,a))\right]\Big\}+\frac{\log2}{1-\gamma},
    \end{align}
    }where $D_{ij}$ is a discriminator that distinguishes whether  $(s,a)$ is from policy $\pi$ interacting with $\mathcal{E}_i$ or from expert $\pi_E^j$
\end{theorem}
\begin{proof}
    See in \cref{appendix:proof_for_theorem}
\end{proof}

\subsection{Practical Methodology}
\label{subsection:practical_methodology}

Due to the second term $\mathbb{E}_{s\sim\mu_{\pi}^i(s), a\sim\pi_E^j(\cdot|s)}[\cdot]$ in \eqref{eq:result_theorem_for_RIME}, which is eventually replaced with sample expectation in implementation, we still require the expert policies $\pi_E^j,~j=1,\cdots,N$. However, $\pi_E^j$ is not available. One way to circumvent this is to reproduce the expert policy $\pi_E^j$ via Behavior Cloning or GAIL+GP by using its  demonstration $\tau_E^j$. However, we found that this method is not so effective. This is due to the classical generalization problem. That is,  the reproduced expert policy $\hat{\pi}_E^j$ based on $\tau_E^j$ does not cover all states induced by $\pi$  (i.e., $s\sim \mu_\pi^i$). For some states, $\hat{\pi}_E^j$ gives inappropriate actions to the agent policy, and these actions lead to learning failure. (The detailed description and experimental results of this approach are in \cref{appendix:description_for_reproduced_expert_policy}.)
To circumvent this,  using importance sampling,  we modify  \eqref{eq:result_theorem_for_RIME} as follows: 

{\small
\begin{align}
    \label{eq:modified_result_for_RIME}
    \min_{\pi}\sum_{i=1}^{N}\sum_{j=1}^{N}\max_{D_{ij}}\Big\{\mathbb{E}_{(s,a)\sim\rho_{\pi}^i}\left[\lambda_j(s)\log(1-D_{ij}(s,a))\right]\nonumber\\+\mathbb{E}_{(s,a)\sim\rho_{E}^j}\left[\frac{\mu_{\pi}^i(s)}{\mu_{E}^j(s)}\lambda_j(s)\log(D_{ij}(s,a))\right]\Big\}, 
\end{align}
}where the last constant term $\log 2/(1-\gamma)$ and the constant scaling factor $1/2N$ in \eqref{eq:result_theorem_for_RIME} are removed. The difference of  (\ref{eq:modified_result_for_RIME}) from  (\ref{eq:result_theorem_for_RIME}) is that for the  expectation in the  second term,  the sample pair $(s,a)$ is drawn from the expert trajectory, which facilitates implementation. Instead, we need the importance sampling ratio $\frac{\mu_{\pi}^i(s)}{\mu_{E}^j(s)}$. 
However, computing $\mu_\pi^i(s)$ and $\mu_E^j(s)$ for a continuous state space by the Bellman flow equation is difficult because we have an infinitely large space, and also the transition dynamics are unknown in the model-free case. In addition, computing $\mu_\pi^i(s)$ and $\mu_E^j(s)$ based on samples is also difficult unless we assume a predefined model distribution.  One can consider applying histogram-based neural network approaches but then  again faces  the generalization issue.
Hence, instead of  computing $\mu_{\pi}^i(s)$ and $\mu_{E}^j(s)$, we directly estimate the ratio $\frac{\mu_{\pi}^i(s)}{\mu_{E}^j(s)}$ by using f-divergence \cite{IS1:LFIW} (detailed implementation and experimental results are in \cref{appendix:ablation_IS}). However, we found that properly estimating $\frac{\mu_{\pi}^i(s)}{\mu_{E}^j(s)}$ and  setting $\frac{\mu_{\pi}^i(s)}{\mu_{E}^j(s)}$ simply to $1$  have almost the same results for most tasks. Thus, for algorithm simplicity,  we  set the importance ratio to $1$ without estimating the ratio. Indeed, similar approaches were used in \cite{il21:DAC,il22:MAILwithCP}.

With the importance sampling ratio set to $1$, the optimization over $\pi$ and $D_{ij}$ in \eqref{eq:modified_result_for_RIME} is tractable. We can apply alternating optimization over $\pi$ and $D_{ij}$. First, consider optimization over $\pi$ for given $D_{ij}$. Note that $\pi$ affects only the first term $\mathbb{E}_{\rho_{\pi}^i}[\cdot]$ in \eqref{eq:modified_result_for_RIME}.
In the first term, we have the weighting factor $\lambda_j(s)$ such that $\sum_{j=1}^N \lambda_j(s)=1$, and determining proper $\lambda_j(s)$ is  cumbersome. Thus, exploiting the fact $\sum_{j=1}^N \lambda_j(s)=1$, we can rewrite the first term for given $D_{ij}$ by pushing $\sum_{j=1}^N$ into the expectation based on the linearity of expectation,  and obtain its upper bound as
{\small
\begin{align}
        &\min_{\pi}\sum_{i=1}^{N}\mathbb{E}_{\rho_{\pi}^i}\left[\sum_{j=1}^{N}\lambda_j(s)\log(1-D_{ij}(s,a))\right]\nonumber\\
    &\leq\min_{\pi}\sum_{i=1}^{N}\mathbb{E}_{\rho_{\pi}^i}\left[\max_{j}\log(1-D_{ij}(s,a))\right],\label{eq:final_objective_policy}
\end{align}
}where $\mathbb{E}_{\rho_{\pi}^i}[\cdot]$ denotes $\mathbb{E}_{(s,a)\sim\rho_{\pi}^i}[\cdot]$, and the inequality is valid because $\sum_{j=1}^N \lambda_j(s)[\cdot]$ can be considered as an expectation ($\max_{D_{ij}}$ does not appear since $D_{ij}$ is given for this step). Then, we optimize the upper bound of the objective function \eqref{eq:final_objective_policy} for policy $\pi$.

Next, consider the optimization of $D_{ij}$ for given $\pi$. This optimization is simplified due to the following theorem:
\begin{theorem}
    \label{theorem2}
    The following maximization problem without the $\lambda_j(s)$ term has the same optimal solution for $D_{ij}$ as \eqref{eq:modified_result_for_RIME} with $\mu_{\pi}^i(s)/\mu_E^j(s)$ set to 1 for given $\pi$:
    {\small
    \begin{align*}
        \max\limits_{D_{ij}}\left\{\mathbb{E}_{\rho_{\pi}^i}\left[\log(1-D_{ij}(s,a))\right]+\mathbb{E}_{\rho_{E}^j}\left[\log(D_{ij}(s,a))\right]\right\}.
    \end{align*}
    }  
\end{theorem}
\begin{proof}
    See in \cref{appendixA:proof_for_theorem2}
\end{proof}
Based on \cref{theorem2} and gradient penalty (GP), we finally derive the objective function of $D_{ij}$  for given $\pi$ as follows:
{\small
\begin{align}
    \max_{D_{ij}}\Big\{&\mathbb{E}_{\rho_{\pi}^i}\left[\log(1-D_{ij}(s,a))\right]+\mathbb{E}_{\tau_{E}^j}\left[\log(D_{ij}(s,a))\right]\nonumber\\
    &+\kappa\mathbb{E}_{\hat{x}}\left(\|\nabla_{\hat{x}} D_{ij}(\hat{x})\|_2-1\right)^2\Big\},
    \label{eq:final_objective_discriminator}
\end{align}
}where $\hat{x}=(s,a)\sim(\epsilon\rho_{\pi}^i+(1-\epsilon)\tau_E^j)$ with $\epsilon\sim\text{Unif}[0,1]$, and $\kappa$ is the weight to control the GP term. 
Note that in \eqref{eq:final_objective_discriminator} we added a gradient penalty term mentioned in \cref{subsection:generative_adversarial_imitation_learning} for stable learning, and $\mathbb{E}_{\rho_E^j}$ is replaced with  $\mathbb{E}_{\tau_E^j}$.

Note that the number of discriminators $D_{ij}$ is given by $N^2$, and increases quadratically as the number $N$ of environments increases.
We can reduce this number by using discriminator weight sharing which makes the discriminator models  share a subset of their weights \cite{WS:CoGAN}. 
The discriminators $D_{i1},\cdots,D_{iN}$ share the weights of their input and hidden layers, and hence they can be implemented as one discriminator with $N$ output nodes. We call this Weight-Shared Discriminator (WSD). For WSD $D_i^{\text{Weight-Shared}}$, the $j$-th output of its $N$ output nodes corresponds to the output of $D_{ij}$, and its objective is given by $\sum_j V_{ij}$, where $V_{ij}$ is the individual objective for $D_{ij}$ in \eqref{eq:final_objective_discriminator}.
Using  WSDs $D_i^{\text{Weight-Shared}}$, $i=1,\cdots,N$, the complexity of discriminators is reduced and is almost $\sim N$.

\subsection{Comparison with Occupancy Measure Matching}
\label{subsection:comparison_amom}

Even without guarantee of the recovery of policy distribution from the occupancy measure in the case of MPE, we can still apply the occupancy measure matching technique to  MPE/MPE. In this case, a reasonable objective is given by 
{\small
\begin{align} \label{eq:OMMEobj}
    \min_{\pi}\sum_{j=1}^{N}\lambda_j\mathcal{D}_{JS}(\bar{\rho}_{\pi},\bar{\rho}_E^j),
\end{align}
}where $\sum_j\lambda_j=1$,  and $\bar{\rho}_\pi$ and $\bar{\rho}_{E}^j$ are the normalized {occupancy} distributions obtained from ${\rho}_\pi$ and ${\rho}_{E}^j$, respectively. (Other objectives are also considered in \cref{section:experiments}.) Then, we can derive an upper bound of \eqref{eq:OMMEobj} as follows:
{\small \begin{align*}
   & \min_{\pi}\sum_{j=1}^{N}\lambda_j\mathcal{D}_{JS}(\bar{\rho}_{\pi},\bar{\rho}_E^j)\\
    & \leq\min_{\pi}\sum_{i=1}^{N}\sum_{j=1}^{N}\frac{\lambda_j(1-\gamma)}{2N}\max_{D_{ij}}\left\{\mathbb{E}_{\rho_{\pi}^i}\left[\log(1-D_{ij}(s,a))\right]\right.\\
    & ~~~~~~~~~~~~~~~+\left.\mathbb{E}_{\rho_E^j}\left[\log D_{ij}(s,a)\right]\right\}+\log2,
\end{align*}
}where the derivation of this upper bound is in \cref{appendix:description_for_a_matching_occupancy_measures}. 
Now consider the optimization of $\pi$ for given $D_{ij}$ in this case. Again, in order to handle $\lambda_j$, we can replace $\sum_j \lambda_j$ with $\max_j$ to yield another upper bound. Then, the objective function of $\pi$ for given $D_{ij}$ is given by

{\small 
\begin{equation} \label{eq:OMMEobj2}
\min\limits_{\pi}\sum_{i=1}^{N}\max_j\mathbb{E}_{\rho_{\pi}^i}[\log(1-D_{ij}(s,a))].
\end{equation}
}We refer to this method as Occupancy measure Matching in Multiple Environments (OMME). The key difference of the objective \eqref{eq:OMMEobj2} from the proposed one in  \eqref{eq:final_objective_policy} is that the operation $\max_j$ is outside the expectation $\mathbb{E}_{\rho_{\pi}^i}[\cdot]$. Note that the order is not interchangeable since $\max_j$ is a nonlinear operation. We will see that this seemingly-slight difference makes a significant performance difference in \cref{section:experiments}.

\begin{table*}[h]
    \caption{Mean return / minimum return over the dynamics parameter range [50\%,150\%] in the $2$ sampled environment case}
    \label{2env_table}
    \centering
    \begin{footnotesize}
    \begin{tabular}{l||c|c|c|c}
    \toprule
    Algorithm & Hopper +Gravity & Walker2d +Gravity & HalfCheetah +Gravity & Ant +Gravity \\
    \midrule
    RIME (ours) & \textbf{2886.7} / \textbf{2332.4} & \textbf{4577.1} / \textbf{4260.9} & \textbf{4268.9} / \textbf{3712.0} & \textbf{4402.2} / \textbf{3909.9} \\
    {RIME+{WSD} (ours)} & \textbf{2857.8} / \textbf{2333.2} & \textbf{4539.3} / \textbf{4235.8} & \textbf{4292.9} / \textbf{3802.5} & \textbf{4388.7} / \textbf{3871.8} \\
    OMME & 2020.3 / 1354.3 & 4467.2 / 3868.7 & 3854.4 / 3352.9 & 3787.8 / 2715.1 \\
    GAIL-mixture & 1797.9 / 959.3 & 3286.7 / 1256.7 & 3688.6 / 2998.4 & 3614.1 / 2856.5 \\
    GAIL-single & 1616.4 / 844.7 & 3210.0 / 1289.0 & 3571.8 / 2673.1 & 3314.3 / 2316.7 \\
    BC & 1129.7 / 648.8 & 971.4 / 313.0 & 1299.5 / -18.3 & 2333.8 / 1988.6 \\
    \midrule\midrule
    Algorithm     & Hopper +Mass & Walker2d +Mass & HalfCheetah +Mass & Ant +Mass \\
    \midrule
    RIME (ours) & \textbf{3535.7} / \textbf{3255.6} & \textbf{4597.0} / 3965.8 & 3959.0 / \textbf{3156.7} & \textbf{4554.5} / \textbf{4417.5} \\
    {RIME+{WSD} (ours)} & \textbf{3499.4} / \textbf{3238.4} & \textbf{4564.7} / \textbf{4174.9} & \textbf{4071.7} / \textbf{3254.7} & \textbf{4539.6} / \textbf{4439.5} \\
    OMME & 3008.6 / 2741.4 & 4046.6 / 3460.0 & 3533.5 / 2732.0 & 4494.6 / 4343.2 \\
    GAIL-mixture & 2334.5 / 1333.5 & 3493.8 / 1425.3 & 2794.7 / 1951.0 & \textbf{4504.9} / 4301.7 \\
    GAIL-single & 2194.1 / 1266.6 & 3031.5 / 1220.4 & 3164.9 / 1685.6 & 4031.1 / 3767.4 \\
    BC & 726.5 / 453.9 & 962.6 / 607.4 & 474.2 / -132.9 & 3923.7 / 3519.0 \\
    \bottomrule
    \end{tabular}
    \end{footnotesize}
\end{table*}
\begin{figure*}[t]
    \centering
    \begin{subfigure}[b]{0.23\textwidth}
        \centering
        \includegraphics[width=\textwidth, height=3.1cm]{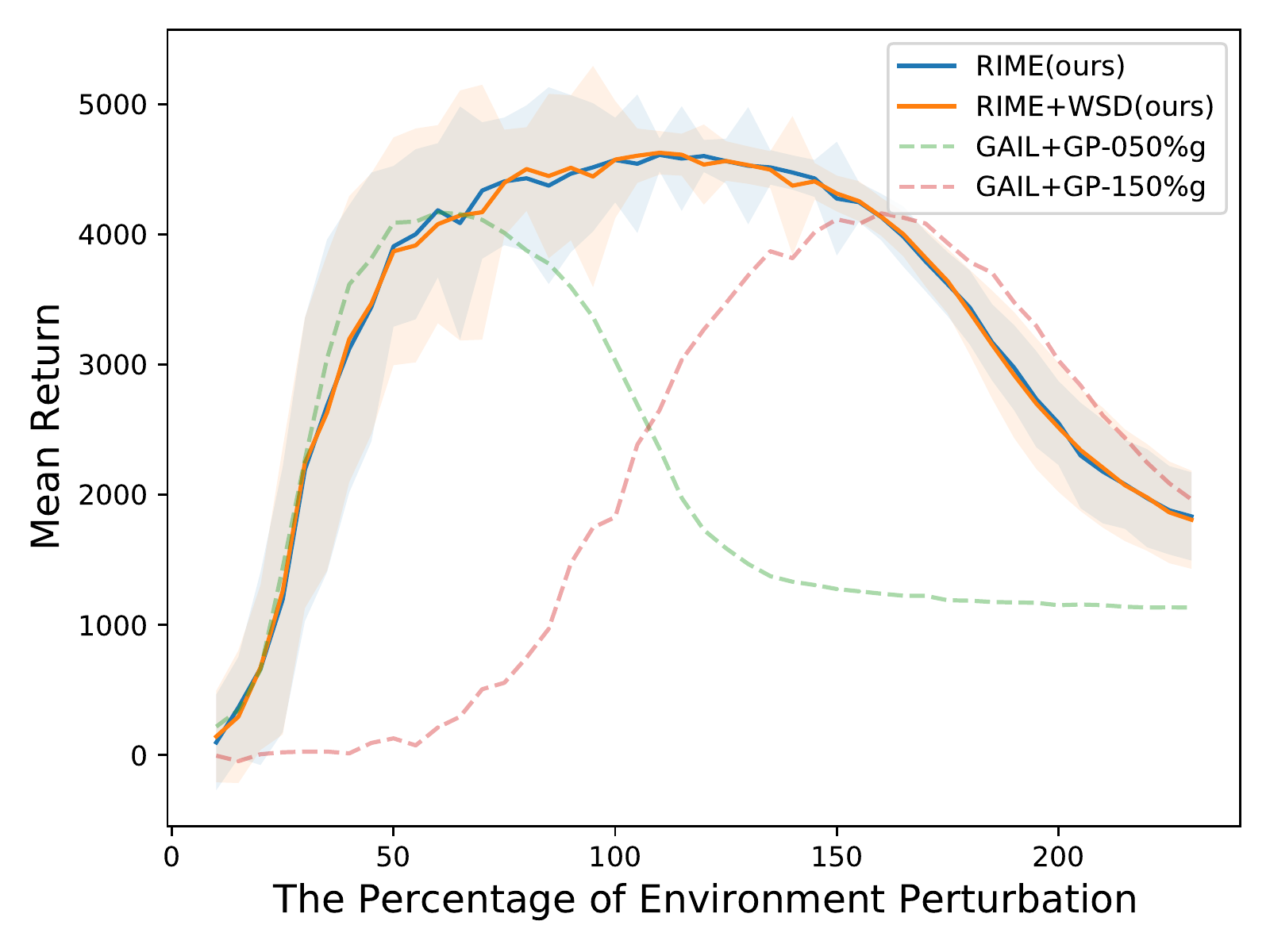}
        \vskip -0.05in
        \captionsetup{justification=centering}
        \caption{Ant+Gravity: performance\label{figure:2envs_results_ag_perf}}
    \end{subfigure}
    \begin{subfigure}[b]{0.23\textwidth}
        \centering
        \includegraphics[width=\textwidth, height=3.1cm]{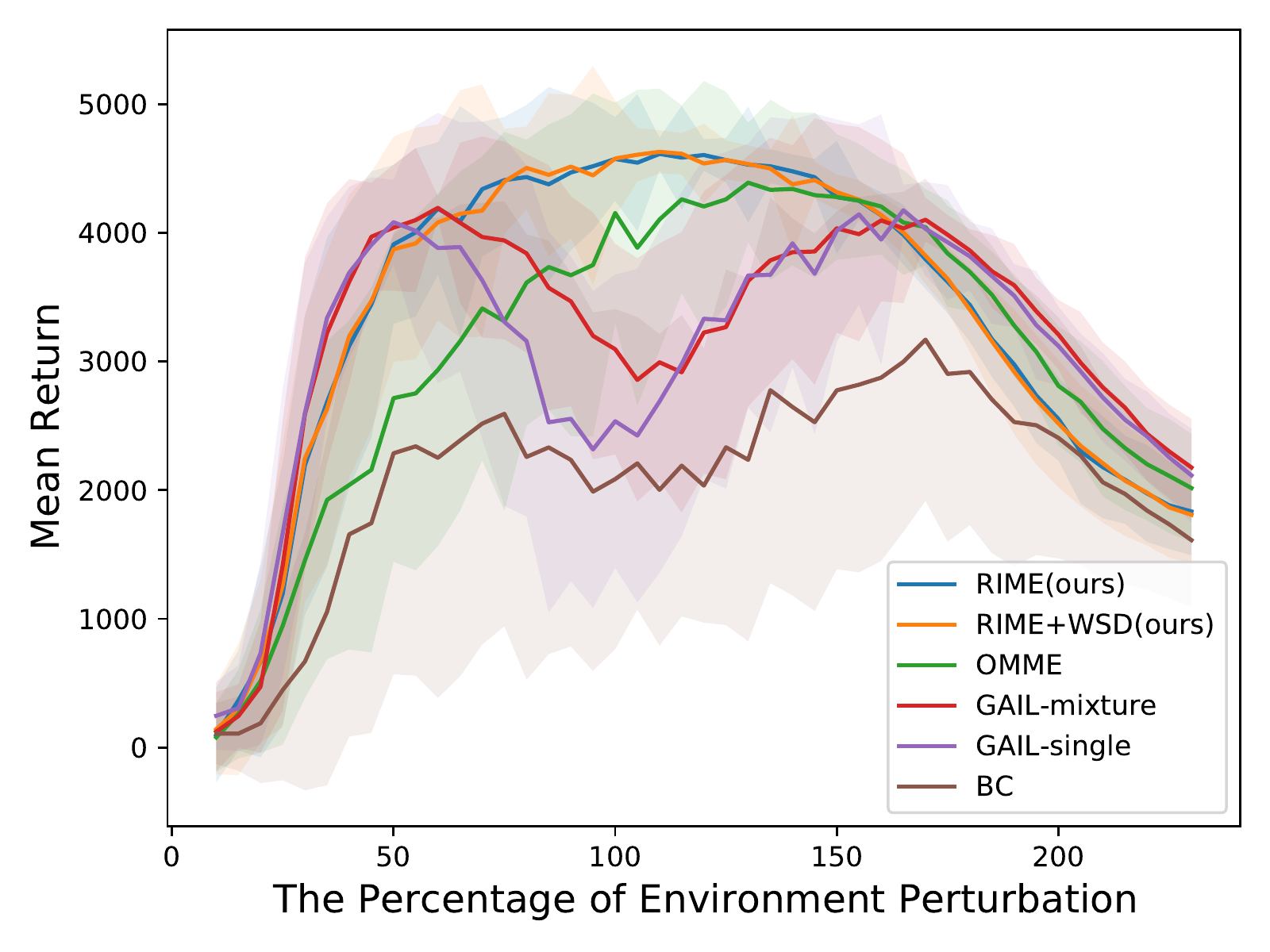}
        \vskip -0.05in
        \captionsetup{justification=centering}
        \caption{Ant+Gravity: comparisons\label{figure:2envs_results_ag_compare}}
    \end{subfigure}
    \begin{subfigure}[b]{0.25\textwidth}
        \centering
        \includegraphics[width=\textwidth, height=3.1cm]{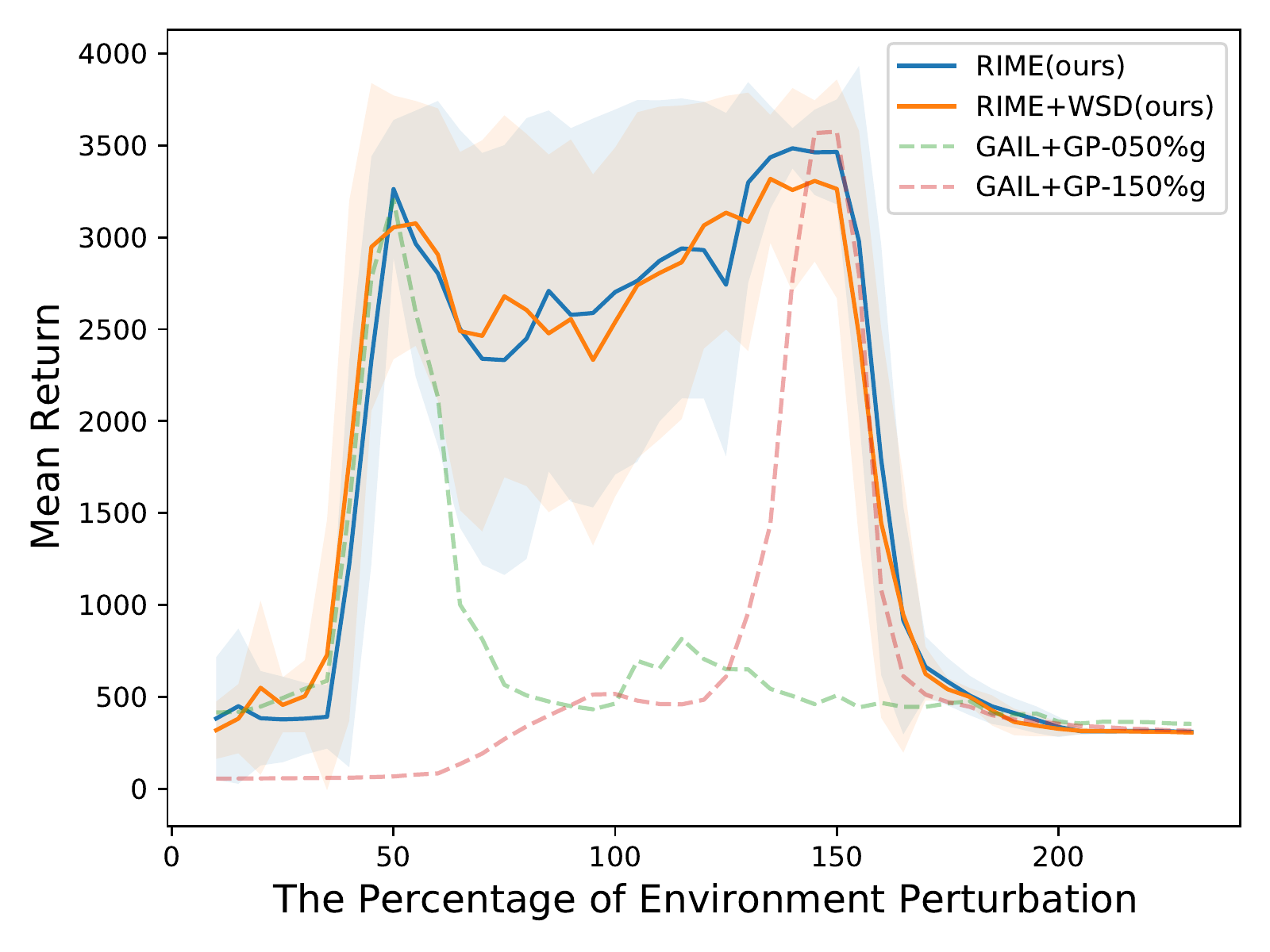}\
        \vskip -0.05in
        \captionsetup{justification=centering}
        \caption{Hopper+Gravity: performance\label{figure:2envs_results_hog_perf}}
    \end{subfigure}
    \begin{subfigure}[b]{0.25\textwidth}
        \centering
        \includegraphics[width=\textwidth, height=3.1cm]{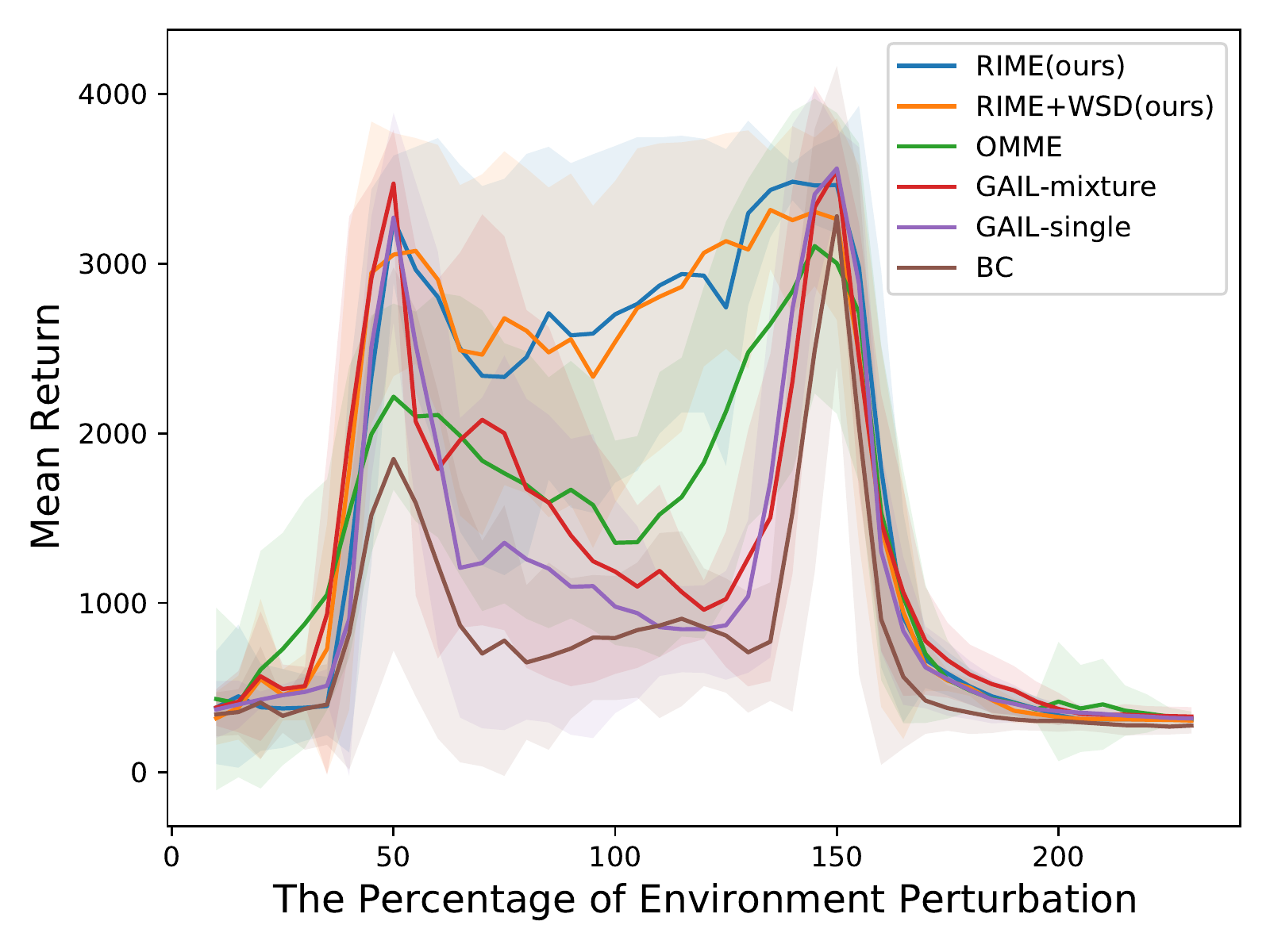}
        \vskip -0.05in
        \captionsetup{justification=centering}
        \caption{Hopper+Gravity: comparisons\label{figure:2envs_results_hog_compare}}
    \end{subfigure}
    \vskip -0.1in
    \caption{Performance on the actual test environment with gravity-perturbed dynamics (the graphs with mass perturbation are in \cref{Appendix:results_in_2_learning_environmnets})  \label{figure:2envs_results}}
    \vskip -0.2in
\end{figure*}

\section{Experiments}
\label{section:experiments}

\subsection{Experimental Settings}
\label{subsection:experimental_settings}

We considered our algorithm together with  the following baselines:
\\
- Behavior Cloning (BC): The policy is trained by supervised learning until validation errors of all expert demonstrations stop decreasing.
\\   
- GAIL-mixture: It is a variant of GAIL+GP for MPE. In this case, we have a single discriminator, and this discriminator distinguishes between all $\bar{\rho}_{\pi}^i$'s and all $\bar{\rho}_E^j$'s. Its objective function for $\pi$ is $\min_{\pi}\mathcal{D}_{JS}(\sum_i\bar{\rho}_{\pi}^i/N,\sum_j\bar{\rho}_{E}^j/N)$. 
\\
- GAIL-single: It is another variant of GAIL+GP for MPE.  In this case, we have  multiple discriminators, and the objective function for $\pi$ is  $\min_{\pi}\sum_i\mathcal{D}_{JS}(\bar{\rho}_{\pi}^i,\bar{\rho}_E^i)$.
\\
- OMME (closest to our algorithm): this is described already. The objective function is given by \eqref{eq:OMMEobj} with \eqref{eq:OMMEobj2}.

Detailed description of the baselines, implementation, expert demonstrations are in \cref{appendix:experimental_information}.  
We considered two versions of the proposed algorithm: RIME and RIME+{WSD}. The only difference between RIME and RIME+{WSD} is the implementation of discriminators $D_{ij}$.  RIME has the discriminators $D_{ij}$ with the objective function \eqref{eq:final_objective_discriminator} and hence the number of the discriminator networks is $N^2$. On the other hand,  RIME+{WSD} uses  {weight-shared discriminator $D_i^{\text{Weight-Shared}}$} described at the end of Section \ref{subsection:practical_methodology}.

We experimented the considered algorithms  on  MuJoCo tasks: Hopper, Walker2d, HalfCheetah and Ant \cite{env:mujoco}.
Each expert demonstration contains $50$ trajectories (i.e., episodes) of state-action pairs generated by the expert and one episode has 1000 timesteps.
We considered   gravity or mass for the considered tasks as our dynamics perturbation parameter $\zeta$. The nominal value $\zeta_0$ means 100\% gravity or mass for each MuJoCo task.  We  trained all algorithms with $10$M timesteps in the case of  experiments with a 1-D dynamics  parameter  and with $5$M timesteps in the case of experiments with 2-D dynamics parameters, and the algorithm for updating the policy is PPO \cite{RL7:PPO,RL8:GAE}.

\begin{table*}[h]
    \caption{Mean return / minimum return over the dynamics parameter range [50\%,150\%] in the $3$ sampled environment case}
    \label{3env_table}
    \centering
    \begin{footnotesize}
    \begin{tabular}{l||c|c|c|c}
    \toprule
    Algorithm & Hopper +Gravity & Walker2d +Gravity & HalfCheetah +Gravity & Ant +Gravity \\
    \midrule
    RIME (ours) & 3164.4 / 2315.5 & \textbf{5197.1} / \textbf{4820.6} & \textbf{5012.8} / 4599.2 & \textbf{4290.8} / \textbf{3485.3} \\
    {RIME+{WSD} (ours)} & \textbf{3281.1} / \textbf{2764.5} & \textbf{5231.7} / \textbf{4894.4} & \textbf{5025.7} / \textbf{4727.9} & \textbf{4168.6} / 3338.2 \\
    OMME & 2878.8 / 2260.6 & 5106.9 / 4484.9 & 4693.4 / 4516.8 & 3689.2 / 2091.5 \\
    GAIL-mixture & 2905.0 / 2289.9 & 4549.7 / 2543.5 & 4746.0 / 4297.6 & 3882.4 / 3431.8 \\
    GAIL-single & 2533.9 / 1276.7 & 4104.0 / 2342.8 & 4512.7 / 4033.1 & 3619.5 / 3088.7 \\
    BC & 798.6 / 448.2 & 791.6 / 594.9 & 1621.9 / 509.8 & 2188.6 / 1129.2 \\
    \midrule\midrule
    Algorithm & Hopper +Mass & Walker2d +Mass & HalfCheetah +Mass & Ant +Mass \\
    \midrule
    RIME (ours) & \textbf{3597.1} / \textbf{3244.6} & \textbf{4752.9} / \textbf{4198.7} & \textbf{5248.2} / 4637.3 & \textbf{4506.1} / \textbf{4384.6} \\
    {RIME+{WSD} (ours)} & \textbf{3585.7} / \textbf{3198.2} & 4704.2 / \textbf{4260.7} & \textbf{5308.5} / \textbf{4868.8} & \textbf{4417.0} / 4202.1 \\
    OMME & 3109.3 / 2815.8 & 4495.2 / 3782.7 & 4802.9 / 4077.7 & 4268.0 / 4036.8 \\
    GAIL-mixture & 3100.1 / 2525.6 & \textbf{4824.1} / 3605.3 & 4237.1 / 3223.3 & 4368.6 / 4075.5 \\
    GAIL-single & 1526.0 / 1011.7 & 4663.1 / 3667.6 & 4088.0 / 2901.7 & 4055.3 / 3603.3 \\
    BC & 410.7 / 162.7 & 704.5 / 354.0 & 1046.7 / -385.2 & 4057.1 / 3740.1 \\
    \bottomrule
    \end{tabular}
    \end{footnotesize}
\end{table*}
\begin{figure*}[t]
    \centering
    \begin{subfigure}[b]{0.19\textwidth}
        \centering
        \includegraphics[width=\linewidth, height=3.1cm]{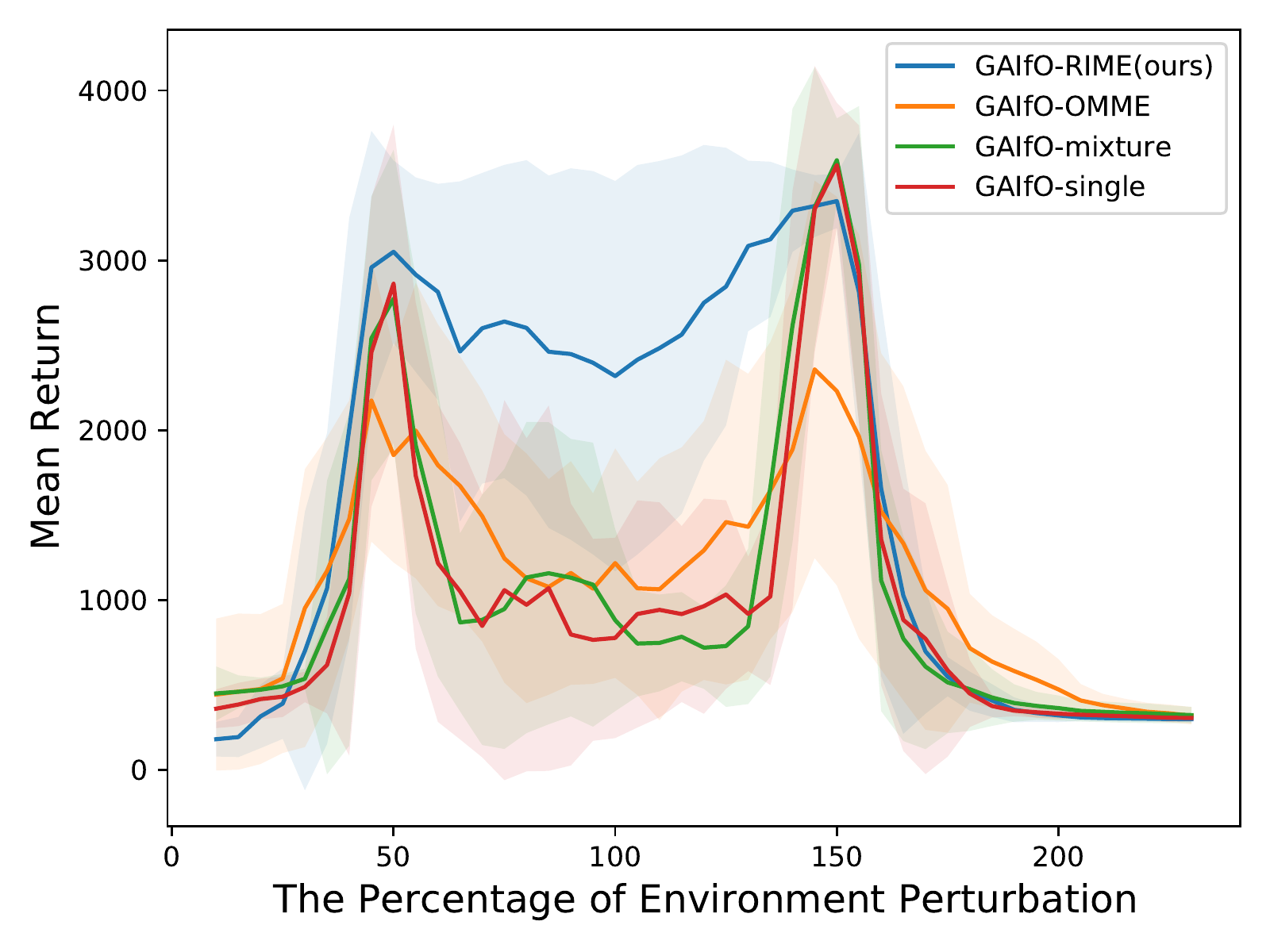}
        \vskip -0.05in
        \captionsetup{justification=centering}
        \caption{Hopper+Gravity:\\variants of GAIfO\label{fig:main_ablation_gaifo1}}
    \end{subfigure}
    \begin{subfigure}[b]{0.19\textwidth}
        \centering
        \includegraphics[width=\linewidth, height=3.1cm]{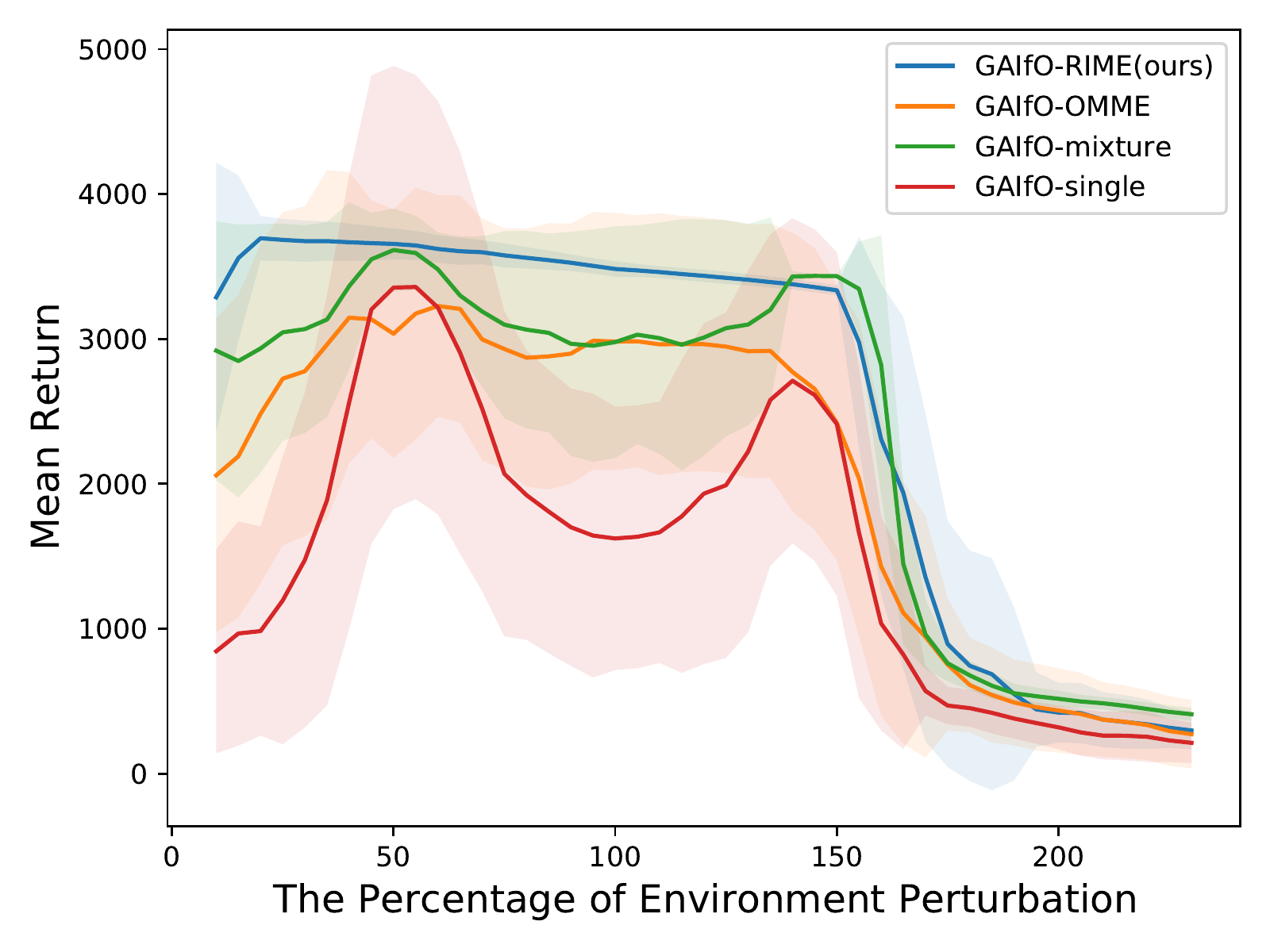}
        \vskip -0.05in
        \captionsetup{justification=centering}
        \caption{Hopper+Mass:\\variants of GAIfO\label{fig:main_ablation_gaifo2}}
    \end{subfigure}
    \begin{subfigure}[b]{0.19\textwidth}
        \centering
        \includegraphics[width=\linewidth, height=3.1cm]{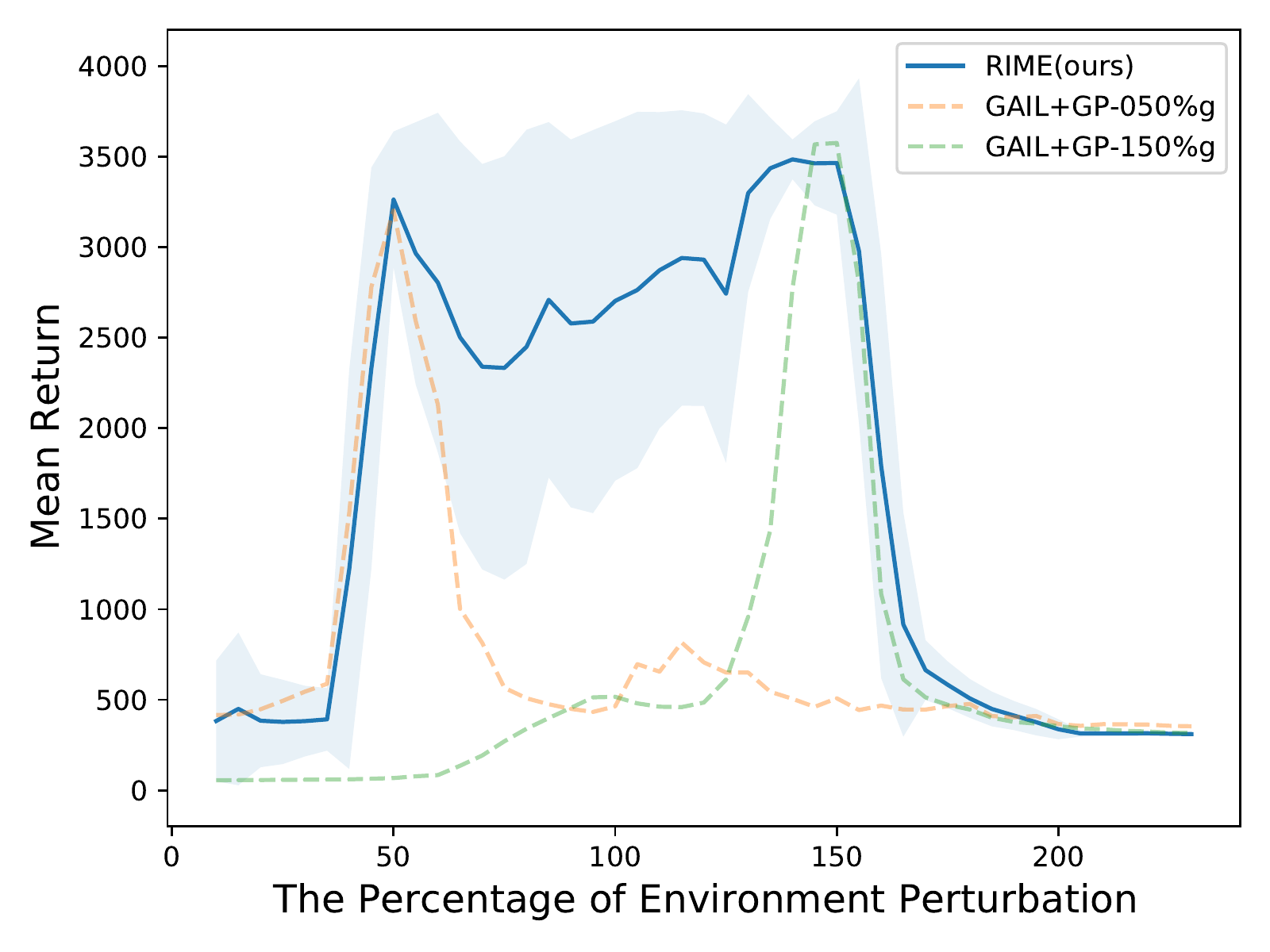}
        \vskip -0.05in
        \captionsetup{justification=centering}
        \caption{RIME in the 2 sampled env. setting\label{fig:main_ablation_wrtN1}}
    \end{subfigure}
    \begin{subfigure}[b]{0.19\textwidth}
        \centering
        \includegraphics[width=\linewidth, height=3.1cm]{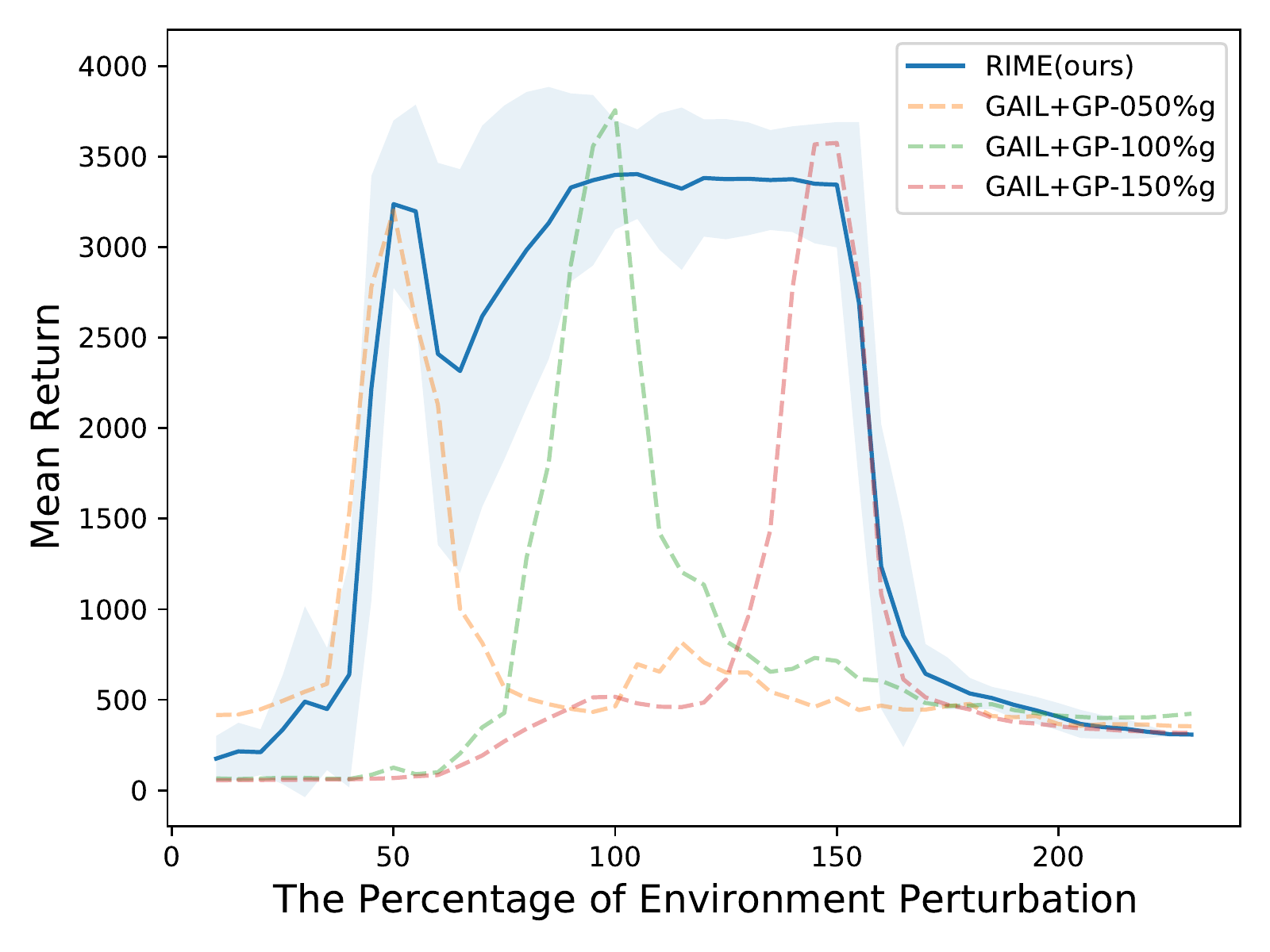}
        \vskip -0.05in
        \captionsetup{justification=centering}
        \caption{RIME in the 3 sampled env. setting\label{fig:main_ablation_wrtN2}}
    \end{subfigure}
    \begin{subfigure}[b]{0.19\textwidth}
        \centering
        \includegraphics[width=\linewidth, height=3.1cm]{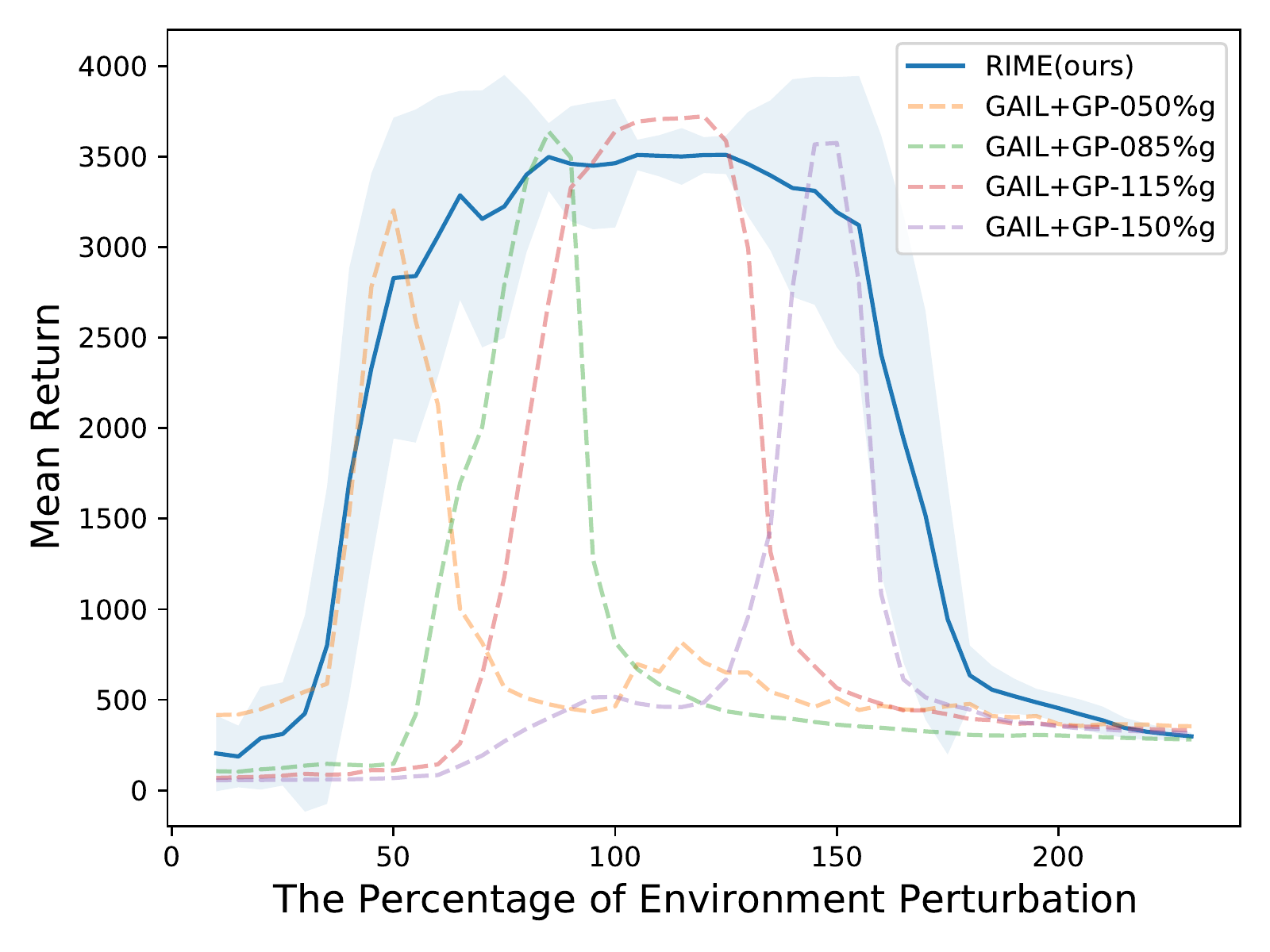}
        \vskip -0.05in
        \captionsetup{justification=centering}
        \caption{RIME in the 4 sampled env. setting\label{fig:main_ablation_wrtN3}}
    \end{subfigure}
    \vskip -0.1in
    \caption{(a-b): Performances when the algorithms use  state-only expert demonstration (the graphs for other tasks are in \cref{appendix:ablation_GAIfO}), (c-e): The performance of RIME with respect to $N$ for Hopper+Gravity task \label{figures:ablation_GAIfO}}
    \vskip -0.2in
\end{figure*}

\subsection{Results}
\label{subsection:results_of_algorithms}

For the same task, we  conducted 3 experiments.  The first two correspond to the case in which a single dynamics parameter (gravity or mass) is perturbed from the nominal value, and the third is the case in which both gravity and mass parameters are perturbed.  The setting for the first is  $N=2$ sampled environments with sampled gravity (or mass) parameters  $50\%\zeta_0$ and $150\%\zeta_0$, and the setting for the second is  $N=3$ sampled environments with sampled gravity (or mass) parameters  $50\%\zeta_0, ~100\% \zeta_0$ and $150\%\zeta_0$.  In the third case, we sampled the joint dynamics of gravity and mass as
$50\%\zeta_{0,g}50\%\zeta_{0,m}, ~50\%\zeta_{0,g}150\%\zeta_{0,m}, ~150\%\zeta_{0,g}50\%\zeta_{0,m}$ and $150\%\zeta_{0,g}150\%\zeta_{0,m}$ with $N=4$. Note that in the third case, we want to cover the variation from 50\% to 150\% for both parameters and only sampled the four corner points in the joint gravity-mass parameter space.

With the sampled $N$ environments, we trained the agent by applying the IL algorithms considered in Section  \ref{subsection:experimental_settings}. Then, in the 1-D perturbation case,  we tested the trained agent policy in  each of  test environments of which dynamics parameter $\zeta$ varies from 10\%$\zeta_0$ to 230\%$\zeta_0$ with 5\%$\zeta_0$ step, i.e., $10\%\zeta_0, 15\%\zeta_0, \cdots, 230\%\zeta_0$. 
In the 2-D perturbation case,  we tested the trained algorithms for each of test environments with dynamics parameters $[50\%, 70\%, \cdots, 150\%]\zeta_{0,g} \times [50\%, 70\%, \cdots, 150\%]\zeta_{0,m}$.

\textbf{IL with 2 Sampled Environments (50\%, 150\%):}
\cref{figure:2envs_results} shows the result in the case of 2 sampled environments with $\zeta=50\%\zeta_0$ and $150\%\zeta_0$. \cref{figure:2envs_results_ag_perf,figure:2envs_results_ag_compare} show the test environment performance of the trained policies of the considered algorithms on the Ant+Gravity task, where the gravity parameter varies. As seen in \cref{figure:2envs_results_ag_perf}, GAIL+GP trained at 50\%$\zeta_0$ and GAIL+GP trained at 150\%$\zeta_0$  perform well only around the trained dynamics. On the other hand, the proposed algorithm (RIME) performs well across all dynamics variation range between the two trained points. It is seen that in the middle the performance of RIME is even better than the peak of the single-environment-specialized  GAIL+GP policy. \cref{figure:2envs_results_ag_compare} shows the performance of other MPE IL algorithms. It is seen that other MPE IL algorithms' performance degrades for the unseen dynamics. Note that the performance sensitivity with respect to the dynamics parameter is mild in the case of Ant+Gravity. \cref{figure:2envs_results_hog_perf,figure:2envs_results_hog_compare} show the test environment performance for Hopper+Gravity in which the performance sensitivity with respect to the dynamics parameter is high. As seen in \cref{figure:2envs_results_hog_perf}, in this case, GAIL+GP can perform only well in a very narrow region around the trained point. On the other hand, the proposed method performs well in the full unseen region between the two trained points. Note that the test performance of the proposed algorithm is superb in the unseen region as compared to other MPE IL baselines, as seen in \cref{figure:2envs_results_hog_compare}.

Table \ref{2env_table} summarizes the robustness performance. We tested each algorithm at the test dynamics $50\%\zeta_0$, $55\%\zeta_0$, $\cdots$, $150\%\zeta_0$ with 5\% quantization between the two sampled dynamics values 50\% and 150\%. We then averaged the performance over the test values and took the minimum performance over the test values. So, when the average and minimum values are equal, the test performance is flat across the tested region, showing the robustness over the variation. It is seen that the proposed algorithm is superior to other algorithms.

\begin{table*}[h]
    \caption{Mean return / minimum return over the dynamics parameter range $[50\%g,150\%g]\times[50\%m,150\%m]$ in the $4$ sampled environments case with 2-dimensional perturbation parameters}
    \label{4env_table}
    \centering
    \begin{footnotesize}
    \begin{tabular}{l||c|c|c|c}
    \toprule
    Algorithm & Hopper + (G\&M) & Walker2d + (G\&M) & HalfCheetah + (G\&M) & Ant + (G\&M) \\
    \midrule
    RIME (ours) & \textbf{3043.3} / \textbf{2430.8} & 4463.4 / 3824.1 & \textbf{3721.3} / 2753.1 & \textbf{4671.7} / \textbf{4233.5} \\
    {RIME+{WSD} (ours)} & 2936.9 / \textbf{2331.6} & \textbf{4646.4} / \textbf{4000.2} & \textbf{3717.9} / \textbf{2891.7} & \textbf{4651.4} / \textbf{4304.5} \\
    OMME & 2573.4 / 1986.4 & 4488.8 / 3029.3 & 3498.5 / 2502.2 & \textbf{4625.3} / 3594.5 \\
    GAIL-mixture & 1636.4 / 712.0 & 3907.8 / 1245.1 & 3018.6 / 1982.3 & 3994.8 / 2746.1 \\
    GAIL-single & 1684.9 / 840.0 & 3844.8 / 2484.2 & 3199.1 / 2072.6 & 3799.7 / 2194.1 \\
    BC & 500.2 / 317.2 & 330.0 / 211.0 & 1289.3 / 30.2 & 1728.2 / 1032.7 \\
    \bottomrule
    \end{tabular}
    \end{footnotesize}
\end{table*}
\begin{figure*}[h]
    \centering
    \includegraphics[width=\textwidth]{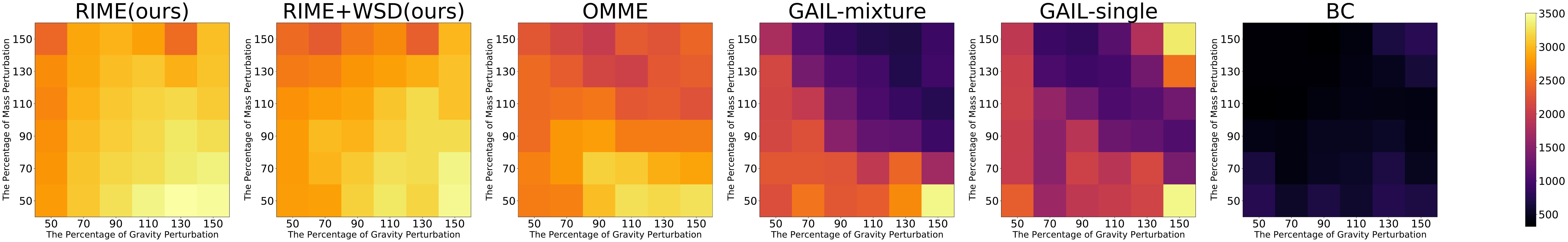}
    \vskip -0.1in
    \caption{Performance on the  test environment with both gravity and mass perturbation for Hopper (the graphs for other tasks are in \cref{Appendix:results_in_4_learning_environmnets})}
    \label{figure:4envs_results_hopper_perf}
    \vskip -0.2in
\end{figure*}

\textbf{IL with 3 Sampled Environments (50\%, 100\%, 150\%):}
Next, we tested the algorithms trained based on $N=3$ with dynamics parameters $50\%\zeta_0$ and $100\%\zeta_0$ and $150\%\zeta_0$.
This setting has more densely-sampled environments  compared to $N=2$. Table \ref{3env_table} shows the corresponding result. (Table \ref{3env_table} was constructed in a similar way to Table \ref{2env_table}.) It is seen that the proposed algorithm is superior to others for a variety of tasks with wide ranges of perturbation.

\textbf{2-D Perturbation Parameter Case:}  \cref{4env_table} summarizes the robustness performance of the algorithms on the test environments with 2-D perturbation (gravity and mass). \cref{figure:4envs_results_hopper_perf} shows the mean-return color plot for the performance of the  algorithms for the Hopper task. It is seen that our proposed algorithm   performs well within the entire 2-D parameter space [50\%,150\%]$\zeta_{0,g}$ $\times$ [50\%,150\%]$\zeta_{0,m}$ by only sampling the four corner points. With this result, we conjecture that even for  higher dimensional perturbation, the proposed method with sampled environments only at the corner points performs well.
Additional experimental results are available in \cref{Appendix:additional_experiments}.

\subsection{Ablation Studies}
\label{subsection:ablation_studies}

\textbf{State-only Expert Demonstration:} \citet{il6:GAIFO} stated that  demonstrations from various resources lack the information on expert's action and addressed the  problem of Imitation from Observation (IfO). We  tested the proposed RIME algorithm  and GAIL variants  in a situation in which   state-only expert demonstrations are available. We trained the algorithms by using state-only demonstrations, which are variants of GAIfO, in the case of $N=2$ sampled environments. The result is shown in \cref{fig:main_ablation_gaifo1,fig:main_ablation_gaifo2}. It is seen that  RIME performs well across the test environment perturbation. This result  indicates that our method can appropriately recover  experts’ preference over the state space.

\textbf{Impact of the Sample Size of Expert Demonstration:}
In the previous section, we used expert demonstrations containing 50 trajectories. However, there may not be sufficient expert demonstrations in the real world. Thus, we performed experiments by reducing the expert demonstration samples gradually from 50 trajectories. Due to space limitation, the result is in  \cref{appendix:ablation_the_size_of_expert_demonstrations}. There, we can see that the proposed robust IL algorithm works quite well even if the expert demonstration length decreases.

\textbf{Tendency over $N$:} From Tables \ref{2env_table} and \ref{3env_table}, we observe that the mean or minimum return performance of the proposed algorithm did not improve monotonically as $N$ changes from 2 to 3. In certain cases, mean return or minimum return slightly decreased as $N$ increases from 2 to 3, although the decrease is not severe. 
For example, in the case of Hopper+Gravity, the mean and minimum return values of 2886.7 and 2332.4 change to 3164.4 and 2315.5, as $N$ increases from 2 to 3.
In order to check the performance tendency with respect to $N$, we further tested the proposed algorithm trained with $N=4$ sampled environments  $\mathcal{E}_1,\cdots,\mathcal{E}_4$ with dynamics parameters $\zeta_1=050\%\zeta_0$, $\zeta_2=085\%\zeta_0$, $\zeta_3=115\%\zeta_0$ and $\zeta_4=150\%\zeta_0$.
\cref{fig:main_ablation_wrtN1,fig:main_ablation_wrtN2,fig:main_ablation_wrtN3} show the performance tendency as $N$ increases. It is hard to say that at every step of $N$ the performance increases as $N$ increases, but there exists a tendency of improvement as $N$ increases. Note that the test performance for $N=4$ is smooth across the variation.

The source code of the proposed algorithm is available at \url{https://github.com/JongseongChae/RIME}.

\section{Conclusion}
\label{section:conclusion}

In this paper, we have considered two issues for the deployment of RL for real-world control problems such as autonomous driving:  robustness and  proper reward design. To address these issues, we have introduced a new framework for robust IL based on multiple environments with dynamics parameters sampled from the continuous range of dynamics parameter variation.  Since it is not obvious that one can recover the policy from the occupancy measure in the case of multiple environments, we have approached the problem by directly optimizing the agent policy in the policy space.  We have formulated the problem as  minimization of the weighted average of divergences from the agent policy to the multiple expert policies. Through a series of manipulations, we have shown that the proposed objective function can be  expressed eventually  as a formula with implementable familiar operations such as expectation,  max and discrimination.  We have evaluated the robustness of the resulting algorithm on MuJoCo tasks by varying  gravity or/and mass parameter(s). Numerical results show that the proposed IL algorithm shows superior performance in robustness across a wide range of dynamics parameter variation based only on training with a few sampled environment dynamics.

\section*{Acknowledgement}

This work was supported by Institute of Information communications Technology Planning Evaluation (IITP) grant
funded by the Korea government (MSIT) (No.2022-0-00469, {Development of Core Technologies for Task-oriented Reinforcement Learning for Commercialization of Autonomous Drones}). Dr. Seungyul Han is currently with Artificial Intelligence Graduate School of UNIST and his work is partly supported by Artificial Intelligence Graduate School support (UNIST), IITP grant funded by the Korea government (MSIT) (No.2020-0-01336).

\bibliography{RIME_reference}
\bibliographystyle{icml2022}

\newpage
\appendix
\onecolumn
\section{Proofs}
\label{appendix:proofs}

\subsection{Proof of Theorem 5.1}
\label{appendix:proof_for_theorem}
\textbf{Theorem 5.1} 
    If $\rho_{\pi}^i(s,a)>0$, $\lambda_j(s)>0$ for any $i,j\in\{1,\cdots,N\}$, $\gamma\in(0,1)$, and $\mathcal{D}$ in eq. \eqref{eq:objective_for_RIME} in the main paper is the Jensen-Shannon divergence, then eq. \eqref{eq:objective_for_RIME} in the main paper is expressed as
    \begin{align*}
        \min_{\pi}\sum_{i=1}^{N}\sum_{j=1}^{N}\max_{D_{ij}}\left\{\mathbb{E}_{(s,a)\sim\rho_{\pi}^i}\left[\frac{\lambda_j(s)}{2N}\log(1-D_{ij}(s,a))\right]+\mathbb{E}_{s\sim\mu_{\pi}^i,a\sim\pi_{E}^j}\left[\frac{\lambda_j(s)}{2N}\log(D_{ij}(s,a))\right]\right\}+\frac{\log2}{1-\gamma}.
    \end{align*}

\begin{proof}
\begin{align*}
    &\min_{\pi}\mathbb{E}_{s\sim\frac{1}{N}\sum_{i=1}^{N}\mu_{\pi}^i}\left[\sum_{j=1}^{N}\lambda_j(s)\mathcal{D}_{JS}(\pi(\cdot|s),\pi_E^j(\cdot|s))\right]\\
    &=\min_{\pi}\int_{s\in\mathcal{S}}\frac{1}{N}\sum_{i=1}^{N}\mu^i_{\pi}(s)\sum_{j=1}^{N}\lambda_j(s)\mathcal{D}_{JS}(\pi(\cdot|s),\pi_E^j(\cdot|s))\\
    &=\min_{\pi}\int_{s\in\mathcal{S}}\frac{1}{N}\sum_{i=1}^{N}\mu^i_{\pi}(s)\sum_{j=1}^{N}\frac{\lambda_j(s)}{2}\left\{\int_{a\in\mathcal{A}}\pi(a|s)\log\frac{2\pi(a|s)}{\pi(a|s)+\pi_E^j(a|s)}+\pi_E^j(a|s)\log\frac{2\pi_E^j(a|s)}{\pi(a|s)+\pi_E^j(a|s)}\right\}\\
    &=\min_{\pi}\int_{s\in\mathcal{S}}\frac{1}{N}\sum_{i=1}^{N}\mu^i_{\pi}(s)\sum_{j=1}^{N}\frac{\lambda_j(s)}{2}\left\{\int_{a\in\mathcal{A}}\pi(a|s)\log\frac{\pi(a|s)}{\pi(a|s)+\pi_E^j(a|s)}+\pi_E^j(a|s)\log\frac{\pi_E^j(a|s)}{\pi(a|s)+\pi_E^j(a|s)}\right\}\\
    &\hspace{1cm}+\int_{s\in\mathcal{S}}\frac{1}{N}\sum_{i=1}^{N}\mu^i_{\pi}(s)\sum_{j=1}^{N}\lambda_j(s)\log2\\
    &\overset{(a)}{=}\min_{\pi}\int_{s\in\mathcal{S}}\frac{1}{N}\sum_{i=1}^{N}\mu^i_{\pi}(s)\sum_{j=1}^{N}\frac{\lambda_j(s)}{2}\left\{\int_{a\in\mathcal{A}}\pi(a|s)\log\frac{\pi(a|s)}{\pi(a|s)+\pi_E^j(a|s)}+\pi_E^j(a|s)\log\frac{\pi_E^j(a|s)}{\pi(a|s)+\pi_E^j(a|s)}\right\}\\
    &\hspace{1cm}+\int_{s\in\mathcal{S}}\frac{1}{N}\sum_{i=1}^{N}\mu^i_{\pi}(s)\log2\\
    &\overset{(b)}{=}\min_{\pi}\sum_{i=1}^{N}\sum_{j=1}^{N}\int_{s\in\mathcal{S}}\int_{a\in\mathcal{A}}\mu^i_{\pi}(s)\frac{\lambda_j(s)}{2N}\left\{\pi(a|s)\log\frac{\pi(a|s)}{\pi(a|s)+\pi_E^j(a|s)}+\pi_E^j(a|s)\log\frac{\pi_E^j(a|s)}{\pi(a|s)+\pi_E^j(a|s)}\right\}+\frac{\log2}{1-\gamma}\\
    &=\min_{\pi}\sum_{i=1}^{N}\sum_{j=1}^{N}\int_{s\in\mathcal{S}}\int_{a\in\mathcal{A}}\Big\{\pi(a|s)\mu^i_{\pi}(s)\frac{\lambda_j(s)}{2N}\log\frac{\pi(a|s)\mu^i_{\pi}(s)}{\pi(a|s)\mu^i_{\pi}(s)+\pi_E^j(a|s)\mu^i_{\pi}(s)}\\
    &\hspace{4cm}+\pi_E^j(a|s)\mu^i_{\pi}(s)\frac{\lambda_j(s)}{2N}\log\frac{\pi_E^j(a|s)\mu^i_{\pi}(s)}{\pi(a|s)\mu^i_{\pi}(s)+\pi_E^j(a|s)\mu^i_{\pi}(s)}\Big\}+\frac{\log2}{1-\gamma}\\
    &=\min_{\pi}\sum_{i=1}^{N}\sum_{j=1}^{N}\int_{s\in\mathcal{S}}\int_{a\in\mathcal{A}}\Big\{\rho_{\pi}^i(s,a)\frac{\lambda_j(s)}{2N}\log\frac{\rho_{\pi}^i(s,a)}{\rho_{\pi}^i(s,a)+\pi_E^j(a|s)\mu^i_{\pi}(s)}\\
    &\hspace{4cm}+\pi_E^j(a|s)\mu^i_{\pi}(s)\frac{\lambda_j(s)}{2N}\log\frac{\pi_E^j(a|s)\mu^i_{\pi}(s)}{\rho_{\pi}^i(s,a)+\pi_E^j(a|s)\mu^i_{\pi}(s)}\Big\}+\frac{\log2}{1-\gamma}\\
    &\overset{(c)}{=}\min_{\pi}\sum_{i=1}^{N}\sum_{j=1}^{N}\max_{D_{ij}}\left\{\mathbb{E}_{(s,a)\sim\rho_{\pi}^i}\left[\frac{\lambda_j(s)}{2N}\log(1-D_{ij}(s,a))\right]+\mathbb{E}_{s\sim\mu_{\pi}^i,a\sim\pi_{E}^j(\cdot|s)}\left[\frac{\lambda_j(s)}{2N}\log(D_{ij}(s,a))\right]\right\}+\frac{\log2}{1-\gamma},
\end{align*}
where (a) holds by the definition of $\lambda_j(s)$, (b) holds due to \cref{appendixA:lemma1} below with the condition of $\gamma<1$, (c) holds due to $D_{ij}\in[0,1]$; for any non-negative $(a,b)\in\mathbb{R}^2\setminus\{0,0\}$, the function $f\to a\log(f)+b\log(1-f)$ has maximum  at $\frac{a}{a+b}$ in $[0,1]$. Thus if we represent $\rho_{\pi}^i(s,a)\cdot\lambda_j(s)/2N$ and $\mu_{\pi}^i(s)\cdot\pi_E^j(a|s)\cdot\lambda_j(s)/2N$ as $g(s,a)$ and $h(s,a)$ respectively, then we have the optimal solution $D_{ij}^*(s,a)=\frac{h(s,a)}{g(s,a)+h(s,a)}=\frac{\pi_E^j(a|s)\mu_{\pi}^i(s)}{\rho_{\pi}^i(s,a)+\pi_E^j(a|s)\mu_{\pi}^i(s)}=\frac{\pi_E^j(a|s)\mu_{\pi}^i(s)}{\pi(a|s)\mu_{\pi}^i(s)+\pi_E^j(a|s)\mu_{\pi}^i(s)}=\frac{\pi_E^j(a|s)}{\pi(a|s)+\pi_E^j(a|s)}$.
\end{proof}

\begin{lemma}[Lemma for proof of \cref{theorem1}]
    \label{appendixA:lemma1}
    Let $f_T^i(s)=\sum_{t=0}^{T}\gamma^t\text{Pr}(s_t=s|\pi,\mathcal{P}^i)$ and $\gamma\in(0,1)$. Then, we have 
    \begin{align*}
        \int_{s\in\mathcal{S}}\mu^i_{\pi}(s)=\frac{1}{1-\gamma}
    \end{align*}
    Therefore, $\int_{s\in\mathcal{S}}\frac{1}{N}\sum_{i=1}^{N}\mu^i_{\pi}(s)=\frac{1}{N}\sum_{i=1}^{N}\int_{s\in\mathcal{S}}\mu^i_{\pi}(s)=\frac{1}{1-\gamma}$.
\end{lemma}

\begin{proof}
    For fixed $s$ and $i$, $0\leq\text{Pr}(s_t=s|\pi,\mathcal{P}^i)\leq1$ because it is a probability. Since $\gamma<1$, we have
    \begin{align*}
        f_T^i(s)=\sum_{t=0}^{T}\gamma^t\text{Pr}(s_t=s|\pi,\mathcal{P}^i)\leq\sum_{t=0}^{T}\gamma^t<\sum_{t=0}^{\infty}\gamma^t=\frac{1}{1-\gamma}<\infty.
    \end{align*}
    Also, by the definition of the discount factor $\gamma$ mentioned in \cref{subsection:markov_decision_process}, its condition $0<\gamma<1$, which implies that $\{f_T^i(s)\}$ is a non-negative and monotone increasing sequence of positive measures with respect to $T$. 
    Hence, by the monotone convergence theorem (Theorem 1.5.7 in \cite{durrett_probability}), $\lim_{T\to\infty}\int_{s\in\mathcal{S}}f_T^i(s) = \int_{s\in\mathcal{S}}\lim_{T\to\infty}f_T^i(s)$.    Therefore, we have
    \begin{align*}
        \int_{s\in\mathcal{S}}\mu_{\pi}^i(s)&=\int_{s\in\mathcal{S}}\lim_{T\to\infty}f_T^i(s)=\lim_{T\to\infty}\int_{s\in\mathcal{S}}f_T^i(s)\\
        &=\lim_{T\to\infty}\int_{s\in\mathcal{S}}\sum_{t=0}^{T}\gamma^t\text{Pr}(s_t=s|\pi,\mathcal{P}^i)\\
        &=\lim_{T\to\infty}\sum_{t=0}^{T}\gamma^t\int_{s\in\mathcal{S}}\text{Pr}(s_t=s|\pi,\mathcal{P}^i)\\
        &=\lim_{T\to\infty}\sum_{t=0}^{T}\gamma^t=\frac{1}{1-\gamma},
    \end{align*}
    where $N,T\in\mathbb{N}$.
\end{proof}

\subsection{Proof of Theorem 5.2}
\label{appendixA:proof_for_theorem2}
\cref{theorem2} can be rewritten as follows:

\textbf{Theorem 5.2}
    The two following maximizing problems have the same optimal solution.
    \begin{align}
        \label{appendixA:eq:theorem2_original}
        &\max_{D_{ij}}\left\{\mathbb{E}_{(s,a)\sim\rho_{\pi}^i}\left[\lambda_j(s)\log(1-D_{ij}(s,a))\right]+\mathbb{E}_{(s,a)\sim\rho_{E}^j}\left[\lambda_j(s)\log(D_{ij}(s,a))\right]\right\}\\
        \label{appendixA:eq:theorem2_simple}
        &\max_{D_{ij}}\left\{\mathbb{E}_{(s,a)\sim\rho_{\pi}^i}\left[\log(1-D_{ij}(s,a))\right]+\mathbb{E}_{(s,a)\sim\rho_{E}^j}\left[\log(D_{ij}(s,a))\right]\right\}.
    \end{align}

\begin{proof}
    \begin{align*}
        &\eqref{appendixA:eq:theorem2_original}=\int_{(s,a)\in\mathcal{S}\times\mathcal{A}}\rho_{\pi}^i(s,a)\lambda_j(s)\log(1-D_{ij}(s,a)) + \rho_E^j(s,a)\lambda_j(s)\log(D_{ij}(s,a))\\
        &\eqref{appendixA:eq:theorem2_simple}=\int_{(s,a)\in\mathcal{S}\times\mathcal{A}}\rho_{\pi}^i(s,a)\log(1-D_{ij}(s,a)) + \rho_E^j(s,a)\log(D_{ij}(s,a)).
    \end{align*}
    For any non-negative $(a,b)\in\mathbb{R}^2\setminus\{0,0\}$, the function $f\to a\log(f)+b\log(1-f)$ has maximum  at $\frac{a}{a+b}$ in $[0,1]$. $\rho_{\pi}^i(s,a)\lambda_j(s)$, $\rho_E^j(s,a)\lambda_j(s)$ can be represented as $g(s,a)$ and $h(s,a)$, respectively. Therefore, the optimal solution of \eqref{appendixA:eq:theorem2_original} $D^*_{ij}(s,a)$ becomes  $\frac{h(s,a)}{g(s,a)+h(s,a)}=\frac{\rho_E^j(s,a)}{\rho_{\pi}^i(s,a)+\rho_E^j(s,a)}$, which is the same as the optimal solution of \eqref{appendixA:eq:theorem2_simple}.
\end{proof}

\newpage

\section{Detailed Descriptions}
\label{appendix:missing_descriptions}

\subsection{Description for Reproduced Expert Policy}
\label{appendix:description_for_reproduced_expert_policy}

In order to optimize \eqref{eq:result_theorem_for_RIME}, expert policies $\pi_E^j$, $j=1,\cdots,N$ are required. However, $\pi_E^j$'s are not available explicitly to us, but  we can use expert demonstration $\tau_E^j$, which is in form of state-action pairs generated by the expert policy $\pi_E^j$ in the $j$-th demonstration environment $\mathcal{E}_j$. In this section, we evaluate an algorithm with the objective function \eqref{eq:result_theorem_for_RIME} in the main paper. In order to compute  the second term $\mathbb{E}_{s\sim\mu_{\pi}^i, a\sim\pi_E^j}[\cdot]$ in the objective function, we reproduce the expert policy $\pi_E^j(\cdot|s)$ by behavior cloning (BC) and GAIL+GP mentioned in \cref{subsection:generative_adversarial_imitation_learning} by using the given expert demonstration. 
Before we optimize the objective function, each expert policy $\pi_E^j$ is first trained in the $j$-th demonstration environment $\mathcal{E}_j$ by using the $j$-th expert demonstration $\tau_E^j$.

With the above experimental setup, we tested the case $N=1$ of the objective function \eqref{eq:result_theorem_for_RIME} as follows: 
\begin{align}
    \label{appendixB:eq:1env_result_theorem}
    \min_{\pi}\max_{D_{11}}\Big\{\mathbb{E}_{(s,a)\sim\rho_{\pi}^1}\left[\frac{\lambda_1(s)}{2}\log(1-D_{11}(s,a))\right]+\mathbb{E}_{s\sim\mu_{\pi}^1,a\sim\pi_{E}^1}\left[\frac{\lambda_1(s)}{2}\log(D_{11}(s,a))\right]\Big\}+\frac{\log2}{1-\gamma},
\end{align}
where $\lambda_1(s)$ is equal to $1$ by the definition of $\lambda_j(s)$. This setting is SNE/SNE. The agent policy is trained in the nominal interaction environment, and the expert $\pi_E^1$ is also trained in the same environment. We evaluated the corresponding performance with 10 random seeds.  \cref{appendixB:figures:reproduced_expert} shows the results of the mean returns of both the expert's and the agent's policies in the nominal test environment. 
In most cases,  the agent policy either has almost the same performance as the expert policy or totally fails to learn. Thus, learning is unstable. 
It implies that if the reproduced expert policy $\hat{\pi}_E^j$ covers the states induced by the agent policy $\pi$, then the agent policy can work well as the expert. On the other hand,  if the reproduced expert $\hat{\pi}_E^j$ does not  cover the  states of the agent policy, then the agent policy fails to learn for the given task.

In practice, it is highly likely that we will have an expert demonstration that covers only a limited region of the entire state-action space. Furthermore, the reproduced expert policy by an IL method would visit a limited region of the entire state space during the training phase. These two reasons can cause  extrapolation error. Due to this error, the reproduced expert policy may sample an action that seems to be a non-expert action for a given state. This inappropriate action will give incorrect information to the agent policy.

\begin{figure}[ht]
    \begin{subfigure}[b]{0.24\textwidth}
        \centering
        \includegraphics[width=\textwidth]{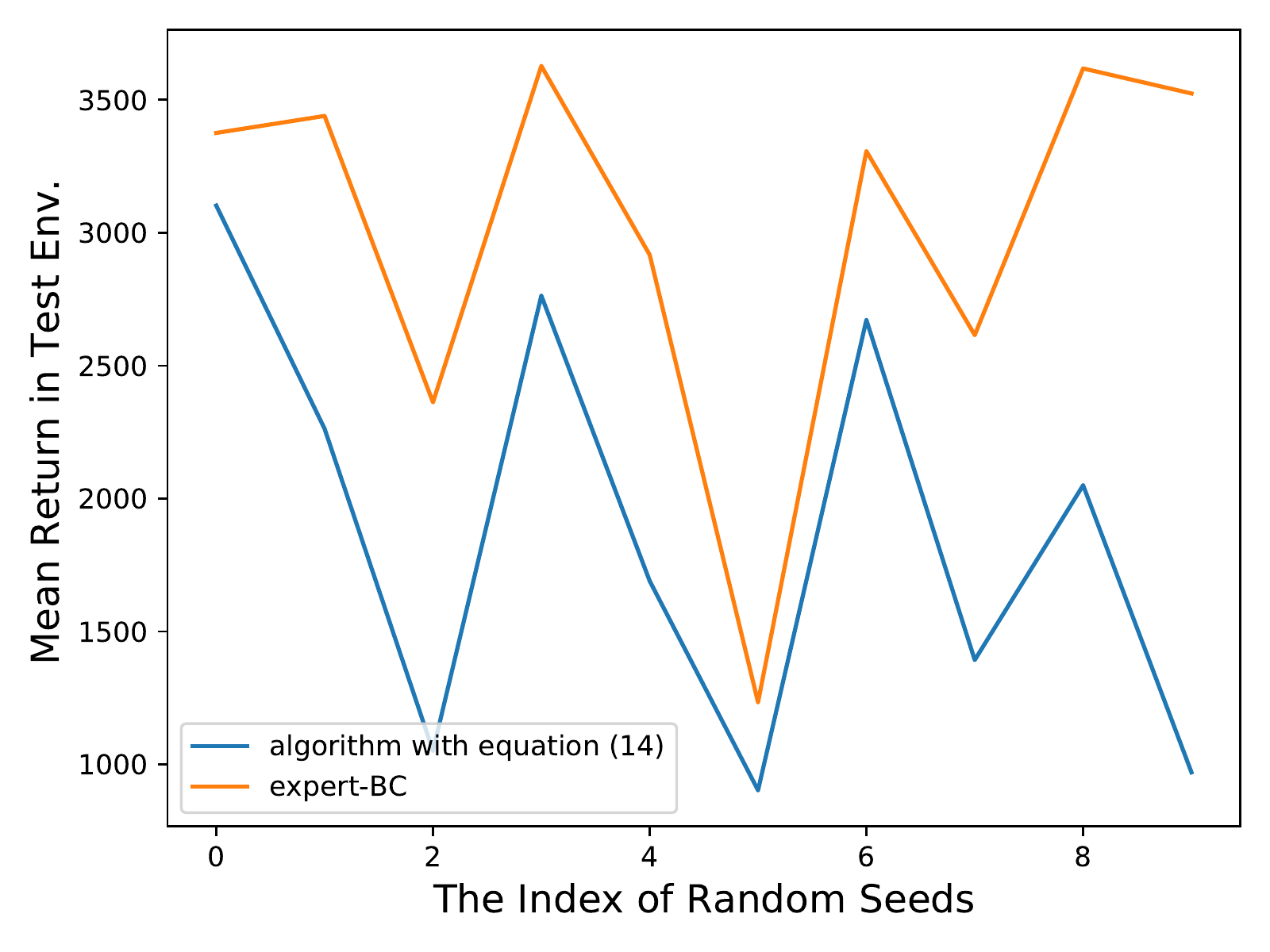}
        \captionsetup{justification=centering}
        \caption{Hopper:\\reproduced expert via BC}
    \end{subfigure}
    \begin{subfigure}[b]{0.24\textwidth}
        \centering
        \includegraphics[width=\textwidth]{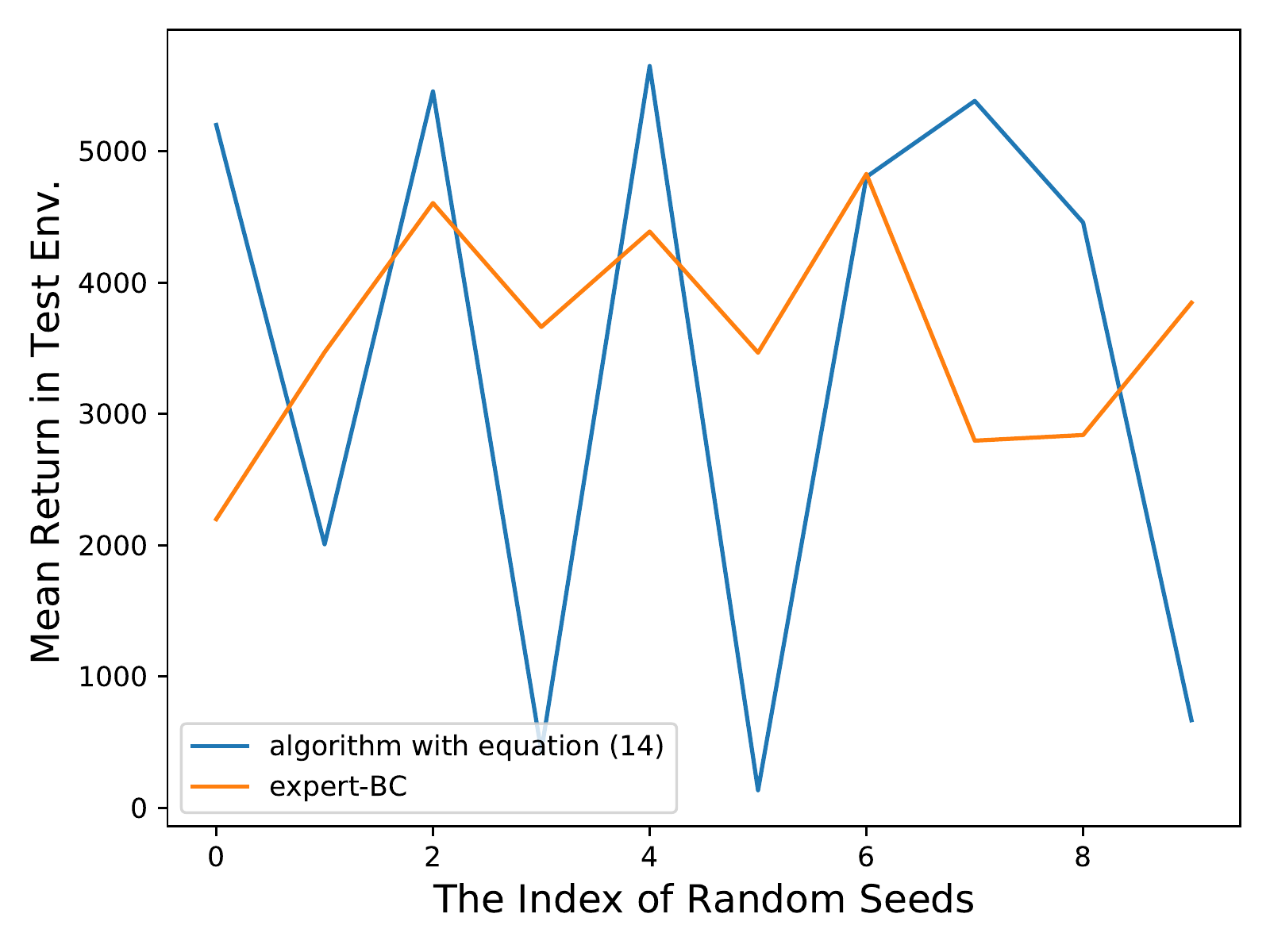}
        \captionsetup{justification=centering}
        \caption{Walker2d:\\reproduced expert via BC}
    \end{subfigure}
    \begin{subfigure}[b]{0.24\textwidth}
        \centering
        \includegraphics[width=\textwidth]{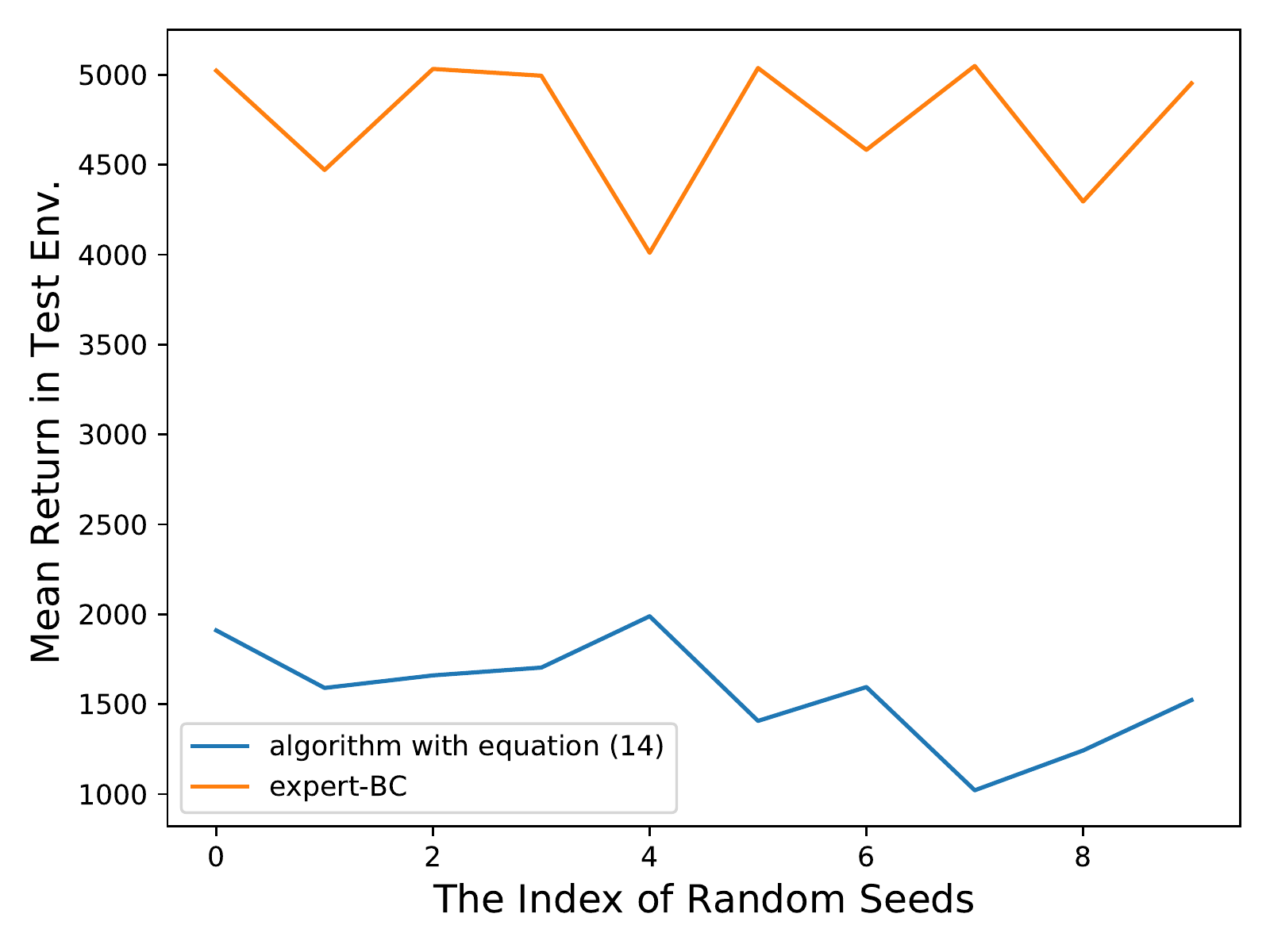}
        \captionsetup{justification=centering}
        \caption{HalfCheetah:\\reproduced expert via BC}
    \end{subfigure}
    \begin{subfigure}[b]{0.24\textwidth}
        \centering
        \includegraphics[width=\textwidth]{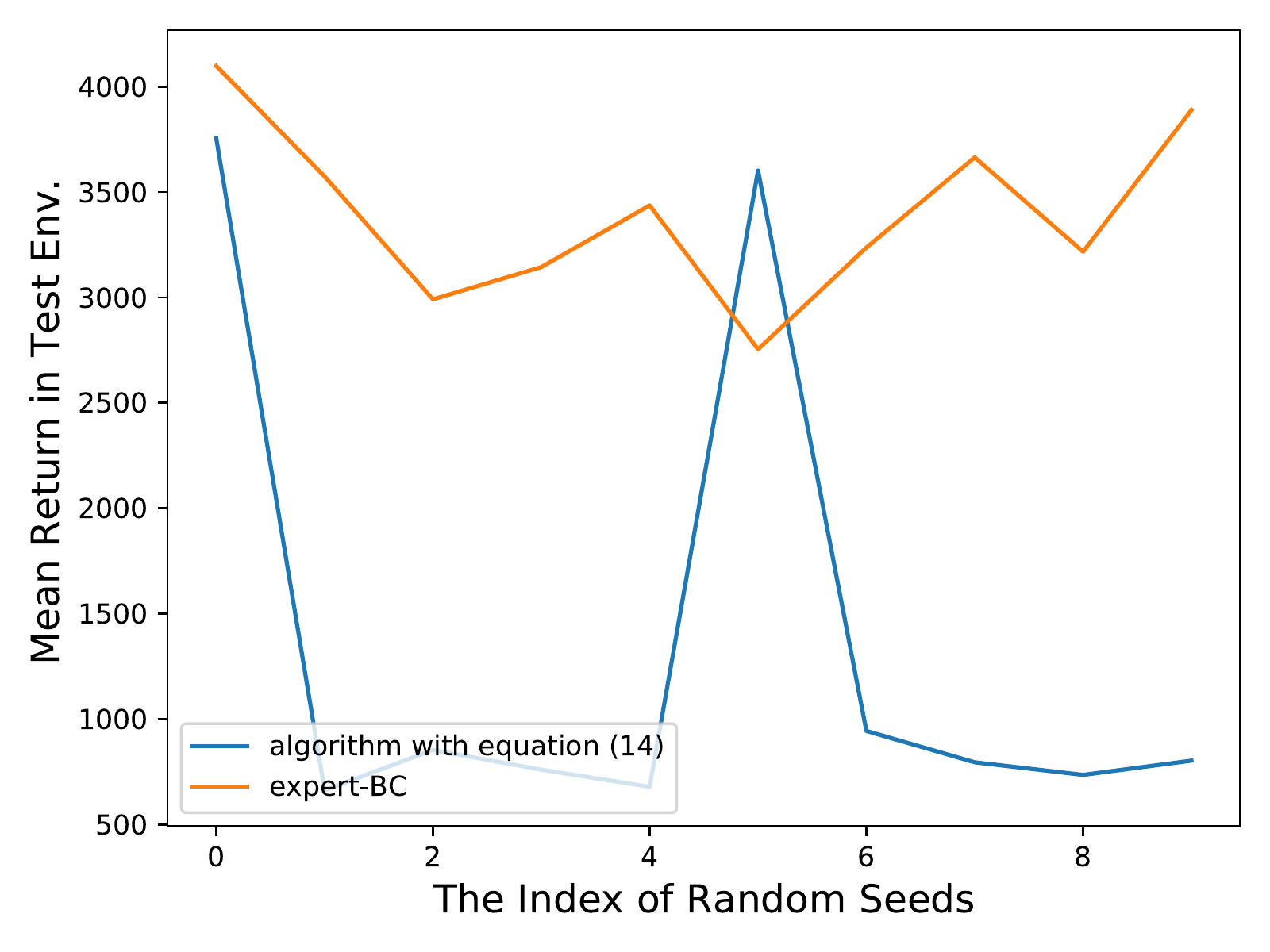}
        \captionsetup{justification=centering}
        \caption{Ant:\\reproduced expert via BC}
    \end{subfigure}

    \begin{subfigure}[b]{0.24\textwidth}
        \centering
        \includegraphics[width=\textwidth]{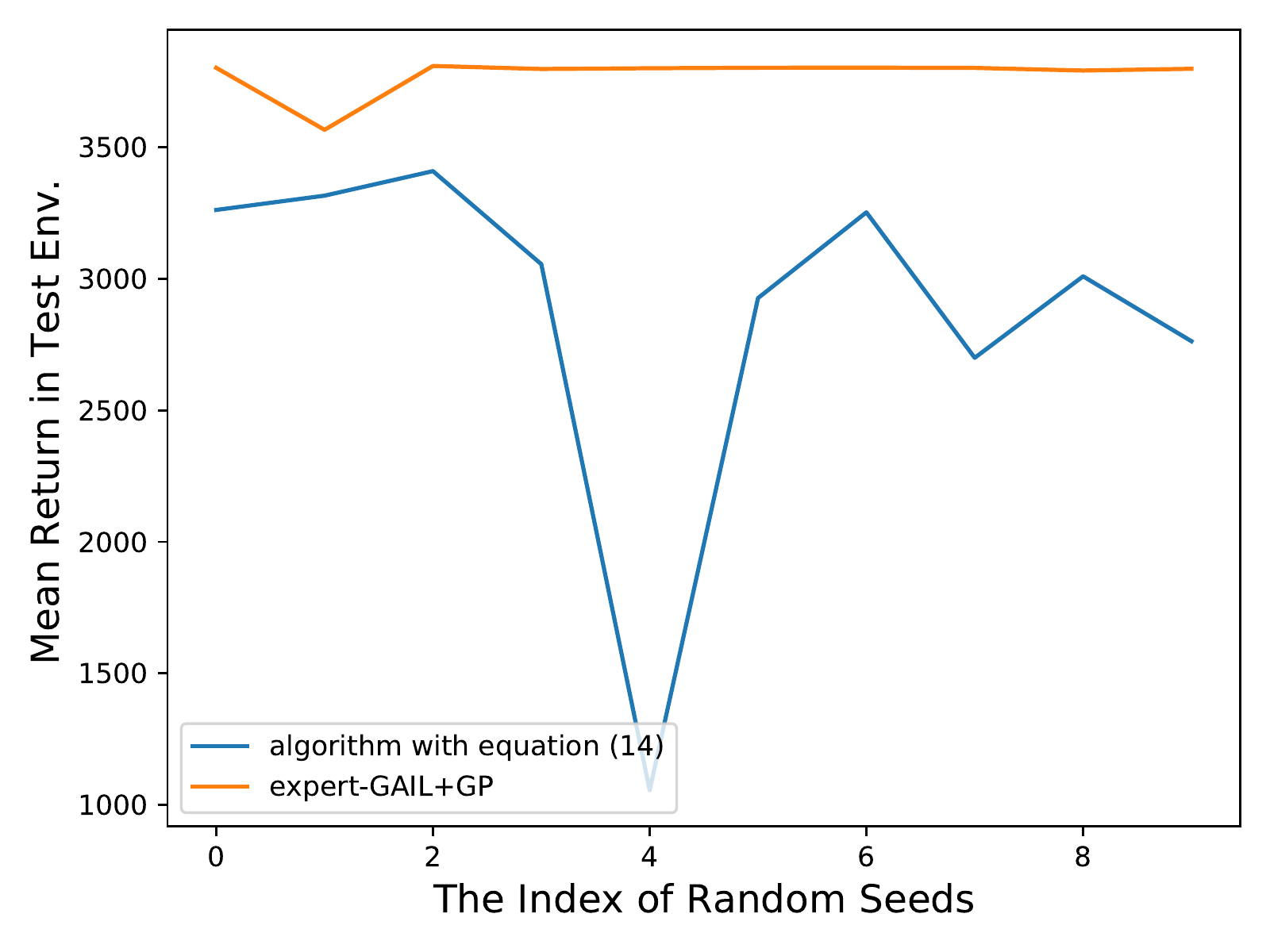}
        \captionsetup{justification=centering}
        \caption{Hopper:\\reproduced expert via GAIL+GP}
    \end{subfigure}
    \begin{subfigure}[b]{0.24\textwidth}
        \centering
        \includegraphics[width=\textwidth]{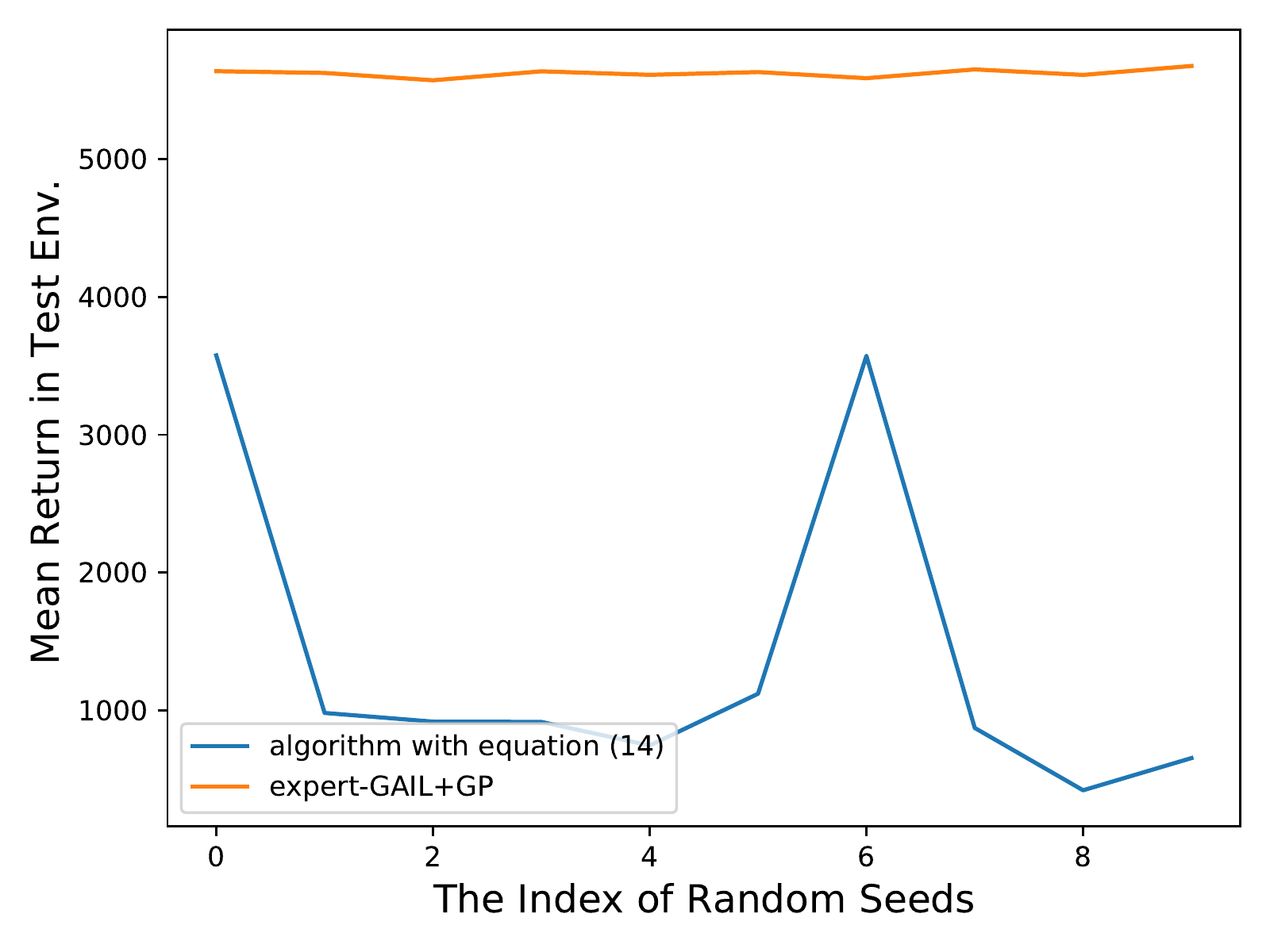}
        \captionsetup{justification=centering}
        \caption{Walker2d:\\reproduced expert via GAIL+GP}
    \end{subfigure}
    \begin{subfigure}[b]{0.24\textwidth}
        \centering
        \includegraphics[width=\textwidth]{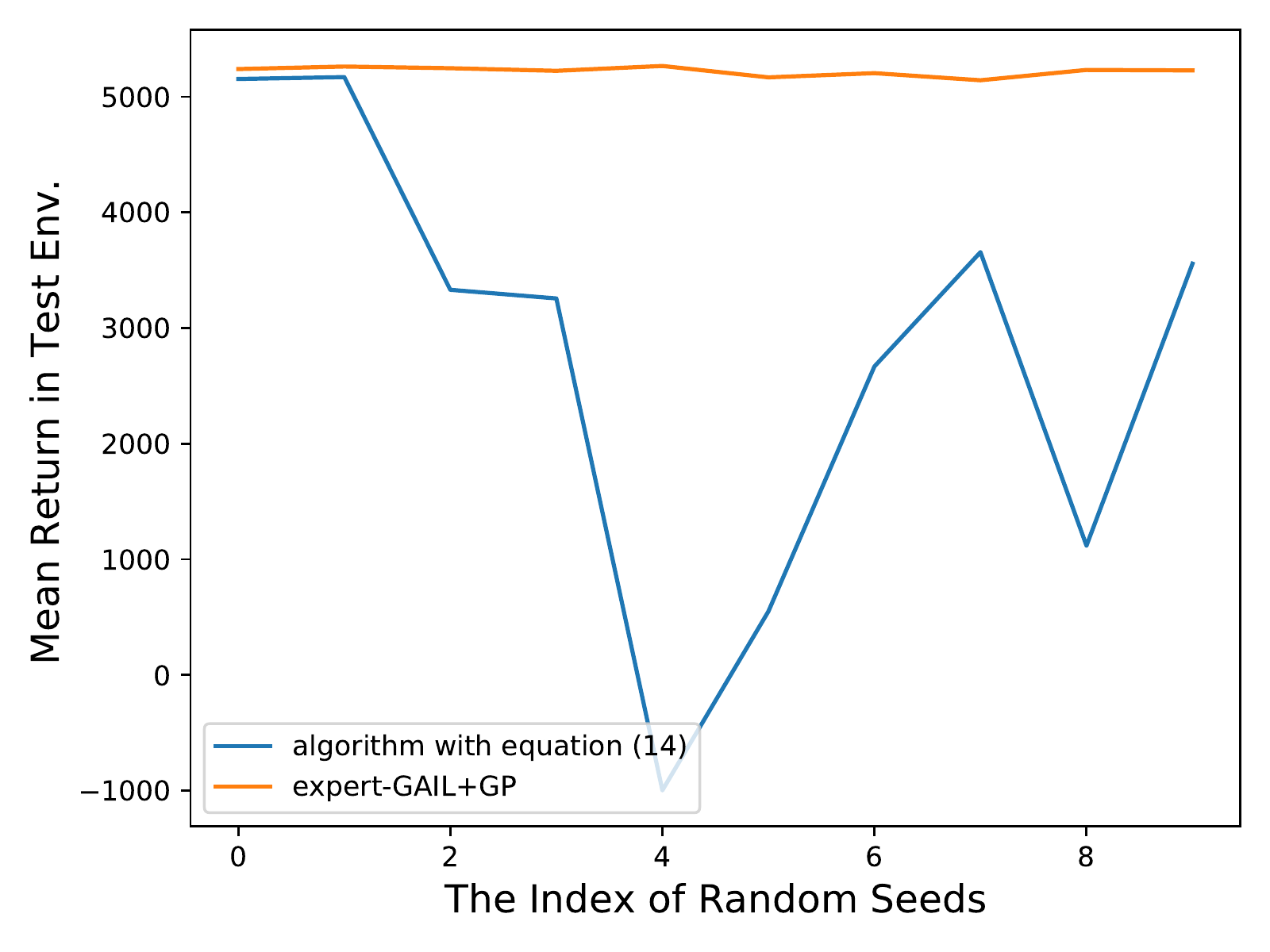}
        \captionsetup{justification=centering}
        \caption{HalfCheetah:\\reproduced expert via GAIL+GP}
    \end{subfigure}
    \begin{subfigure}[b]{0.24\textwidth}
        \centering
        \includegraphics[width=\textwidth]{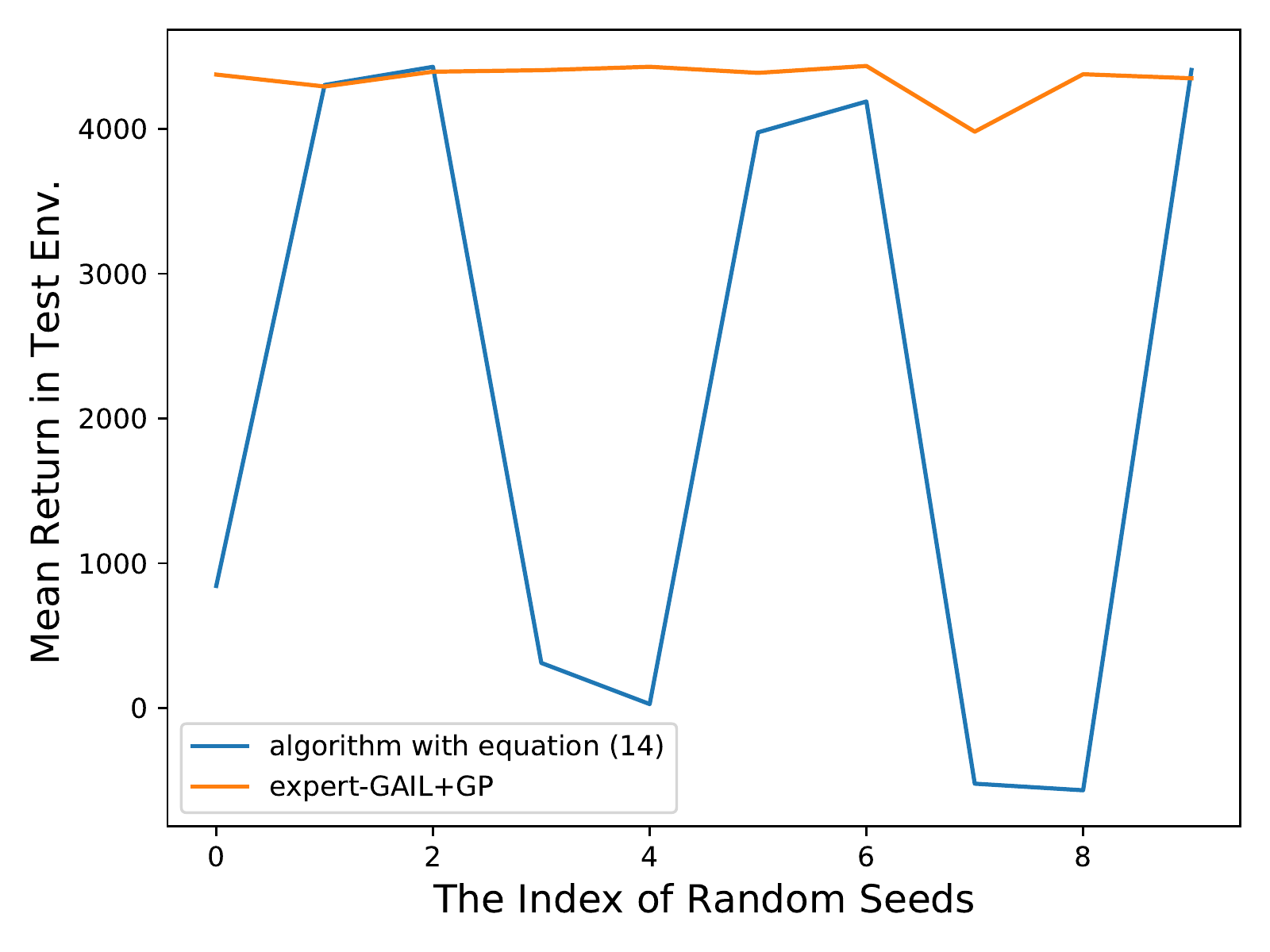}
        \captionsetup{justification=centering}
        \caption{Ant:\\reproduced expert via GAIL+GP}
    \end{subfigure}
    \caption{Mean return performance of IL algorithm solving \eqref{appendixB:eq:1env_result_theorem}  with the reproduced expert policy $\hat{\pi}_E^j$. The x-axis is the index of 10 random seeds and the y-axis is the mean return. The orange lines - the performance of reproduced experts, and  the blue line - IL algorithm solving \eqref{appendixB:eq:1env_result_theorem}  with the reproduced expert policy $\hat{\pi}_E^j$. \label{appendixB:figures:reproduced_expert}}
\end{figure}

\newpage

\subsection{Description for Occupancy measure Matching in Multiple Environments (OMME)}
\label{appendix:description_for_a_matching_occupancy_measures}

Equation \eqref{eq:OMMEobj} in the main paper is rewritten here as
\begin{align}
    \label{appendix:eq:OMME}
    \min_{\pi}\sum_{j=1}^{N}\lambda_j\mathcal{D}_{JS}(\bar{\rho}_{\pi},\bar{\rho}_E^j),
\end{align}
where $\sum_j\lambda_j=1$. We assume $\gamma\in(0,1)$, and as in \cite{il5:GAIL,il12:robustIL4}, $\bar{\rho}_{\pi}^i=(1-\gamma)\rho_{\pi}^i$ and $\bar{\rho}_{E}^j=(1-\gamma)\rho_{E}^j$ are the normalized occupancy distributions from $\pi$ in $\mathcal{E}_i$ and $\pi_E^j$. 

Then, we have
\begin{align*}
    &\min_{\pi}\sum_{j=1}^{N}\lambda_j\mathcal{D}_{JS}(\bar{\rho}_{\pi},\bar{\rho}_E^j)
    =\min_{\pi}\sum_{j=1}^{N}\lambda_j\mathcal{D}_{JS}(\frac{1}{N}\sum_{i=1}^{N}\bar{\rho}_{\pi}^i,\bar{\rho}_E^j)\\
    &\overset{(a)}{\leq}\min_{\pi}\sum_{j=1}^{N}\lambda_j\frac{1}{N}\sum_{i=1}^{N}\mathcal{D}_{JS}(\bar{\rho}_{\pi}^i,\bar{\rho}_E^j)
    =\min_{\pi}\sum_{i=1}^{N}\sum_{j=1}^{N}\frac{\lambda_j}{N}\mathcal{D}_{JS}(\bar{\rho}_{\pi}^i,\bar{\rho}_E^j)\\
    &=\min_{\pi}\sum_{i=1}^{N}\sum_{j=1}^{N}\frac{\lambda_j}{2N}\left\{\int_{(s,a)\in\mathcal{S}\times\mathcal{A}}\bar{\rho}_{\pi}^i(s,a)\log\frac{2\bar{\rho}_{\pi}^i(s,a)}{\bar{\rho}_{\pi}^i(s,a)+\bar{\rho}_{E}^j(s,a)}+\bar{\rho}_{E}^j(s,a)\log\frac{2\bar{\rho}_{E}^j(s,a)}{\bar{\rho}_{\pi}^i(s,a)+\bar{\rho}_{E}^j(s,a)}\right\}\\
    &=\min_{\pi}\sum_{i=1}^{N}\sum_{j=1}^{N}\frac{\lambda_j}{2N}\left\{\int_{(s,a)\in\mathcal{S}\times\mathcal{A}}\bar{\rho}_{\pi}^i(s,a)\log\frac{\bar{\rho}_{\pi}^i(s,a)}{\bar{\rho}_{\pi}^i(s,a)+\bar{\rho}_{E}^j(s,a)}+\bar{\rho}_{E}^j(s,a)\log\frac{\bar{\rho}_{E}^j(s,a)}{\bar{\rho}_{\pi}^i(s,a)+\bar{\rho}_{E}^j(s,a)}\right\}\\
    &\hspace{1cm}+\sum_{i=1}^{N}\sum_{j=1}^{N}\frac{\lambda_j}{2N}\left\{\int_{(s,a)\in\mathcal{S}\times\mathcal{A}}\bar{\rho}_{\pi}^i(s,a)+\bar{\rho}_E^j(s,a)\right\}\log2\\
    &\overset{(b)}{=}\min_{\pi}\sum_{i=1}^{N}\sum_{j=1}^{N}\frac{\lambda_j}{2N}\left\{\int_{(s,a)\in\mathcal{S}\times\mathcal{A}}\bar{\rho}_{\pi}^i(s,a)\log\frac{\bar{\rho}_{\pi}^i(s,a)}{\bar{\rho}_{\pi}^i(s,a)+\bar{\rho}_{E}^j(s,a)}+\bar{\rho}_{E}^j(s,a)\log\frac{\bar{\rho}_{E}^j(s,a)}{\bar{\rho}_{\pi}^i(s,a)+\bar{\rho}_{E}^j(s,a)}\right\}+\log2\\
    &=\min_{\pi}\sum_{i=1}^{N}\sum_{j=1}^{N}\frac{\lambda_j(1-\gamma)}{2N}\left\{\int_{(s,a)\in\mathcal{S}\times\mathcal{A}}\rho_{\pi}^i(s,a)\log\frac{\rho_{\pi}^i(s,a)}{\rho_{\pi}^i(s,a)+\rho_{E}^j(s,a)}+\rho_{E}^j(s,a)\log\frac{\rho_{E}^j(s,a)}{\rho_{\pi}^i(s,a)+\rho_{E}^j(s,a)}\right\}+\log2\\
    &=\min_{\pi}\sum_{i=1}^{N}\sum_{j=1}^{N}\frac{\lambda_j(1-\gamma)}{2N}\max_{D_{ij}}\left\{\mathbb{E}_{(s,a)\sim\rho_{\pi}^i}\left[\log(1-D_{ij}(s,a))\right]+\mathbb{E}_{(s,a)\sim\rho_E^j}\left[\log D_{ij}(s,a)\right]\right\}+\log2,
\end{align*}
where (a) holds by the convexity of the Jensen-Shannon divergence, (b) holds by the definition of $\lambda_j$.

\newpage

\section{Algorithm: Robust Imitation Learning against Variations in Environment Dynamics}
\label{Appendix:pseudo_code}
\begin{algorithm}[h]
   \caption{Robust Imitation learning with Multiple perturbed Environments (RIME)}
   \label{alg:rime_algorithm}
\begin{algorithmic}
   \STATE {\bfseries Input:} The number of sampled environments $N$, sampled environments $\mathcal{E}_1, \ldots, \mathcal{E}_N $, expert demonstrations $\tau_E^1, \ldots, \tau_E^N$, policy parameter $\theta$, parameter of discriminators $\{\phi_{ij}\}$, the number of learning iterations $n_{epoch}$, the weight of GP $\kappa$.
   \STATE Initialize all parameters $\theta$, $\{\phi_{ij}\}$.
   \FOR{$k=1$ {\bfseries to} $n_{epoch}$}
   \FOR{$i=1$ {\bfseries to} $N$}
   \STATE Sample trajectories $\tau_{\pi}^i\sim\pi_{\theta}$ in $\mathcal{E}_i$
   \FOR{$j=1$ {\bfseries to} $N$}
   \STATE Update the discriminator $D_{\phi_{ij}}$ by maximizing \eqref{eq:final_objective_discriminator}
   \ENDFOR
   \ENDFOR
   \FOR{$i=1$ {\bfseries to} $N$}
   \STATE Update the policy $\pi_{\theta}$ by minimizing \eqref{eq:final_objective_policy} using PPO
   \ENDFOR
   \ENDFOR
\end{algorithmic}
\end{algorithm}

\newpage

\section{Ablation Studies}
\label{appendix:ablation_study}

\subsection{Ablation Study for an Algorithm Trained in the SNE/MPE Setting}
\label{appendix:ablationn_for_algorithm_singleenv_multipletau}

To see the effect of interacting with MPE, we evaluated SNEMPE-max described in \cref{subsection:motivation_for_using_multiple_environments} in the perturbed test environments. This algorithm is obtained by simply applying the robust RL principle to the IL setting. Furthermore, it is a variant of our algorithm \eqref{eq:modified_result_for_RIME} applied to the SNE/MPE setting.

We used three expert demonstrations which are generated by their experts in demonstration environments with perturbations $050\%\zeta_0$, $100\%\zeta_0$, $150\%\zeta_0$, where $\zeta_0$ is the nominal dynamics value. With three expert demonstrations, we trained this algorithm in the nominal interaction environment with $\zeta_0$. It has discriminators $D_{1j}$, and the objective function for the discriminator $D_{1j}$ is the same as our discriminator's objective function \eqref{eq:final_objective_discriminator}. The objective function for the policy is given by
\begin{align}
    \label{appendix:eq:snempe-max}
    &\min_{\pi}\mathbb{E}_{(s,a)\sim\rho_{\pi}^1}\left[\max_{j}\log(1-D_{1j}(s,a))\right].
\end{align}

\cref{appendixD:figures:singleenv_multipletau} shows that SNEMPE-max fails when the underlying environment dynamics are perturbed from those of the interaction environment. It is seen that SNEMPE-max trained in a single interaction environment cannot properly capture the diverse  dynamics of multiple demonstration environments.

\begin{figure}[ht]
    \begin{subfigure}[b]{0.24\textwidth}
        \centering
        \includegraphics[width=\textwidth]{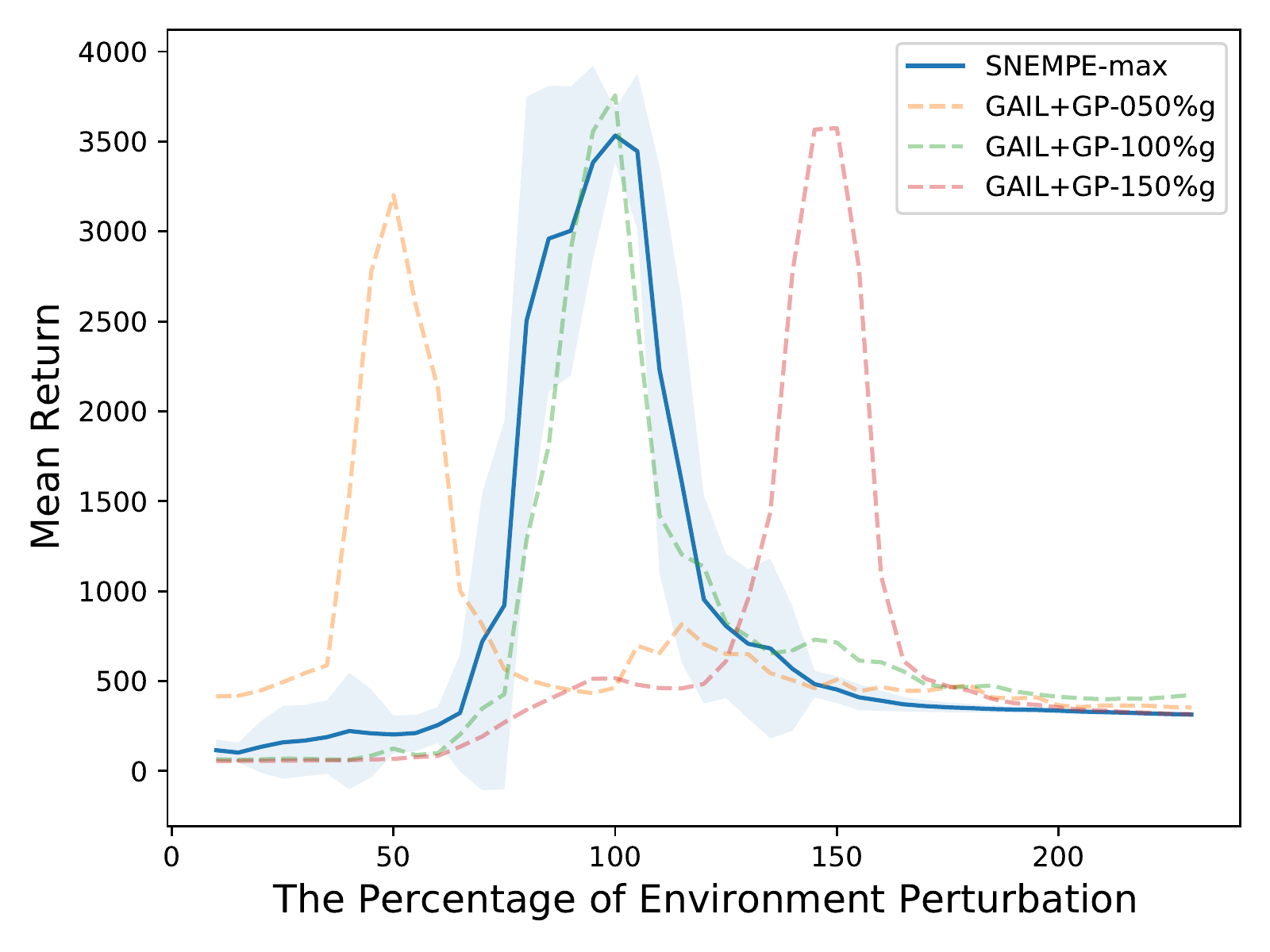}
        \captionsetup{justification=centering}
        \caption{Hopper+Gravity}
    \end{subfigure}
    \begin{subfigure}[b]{0.24\textwidth}
        \centering
        \includegraphics[width=\textwidth]{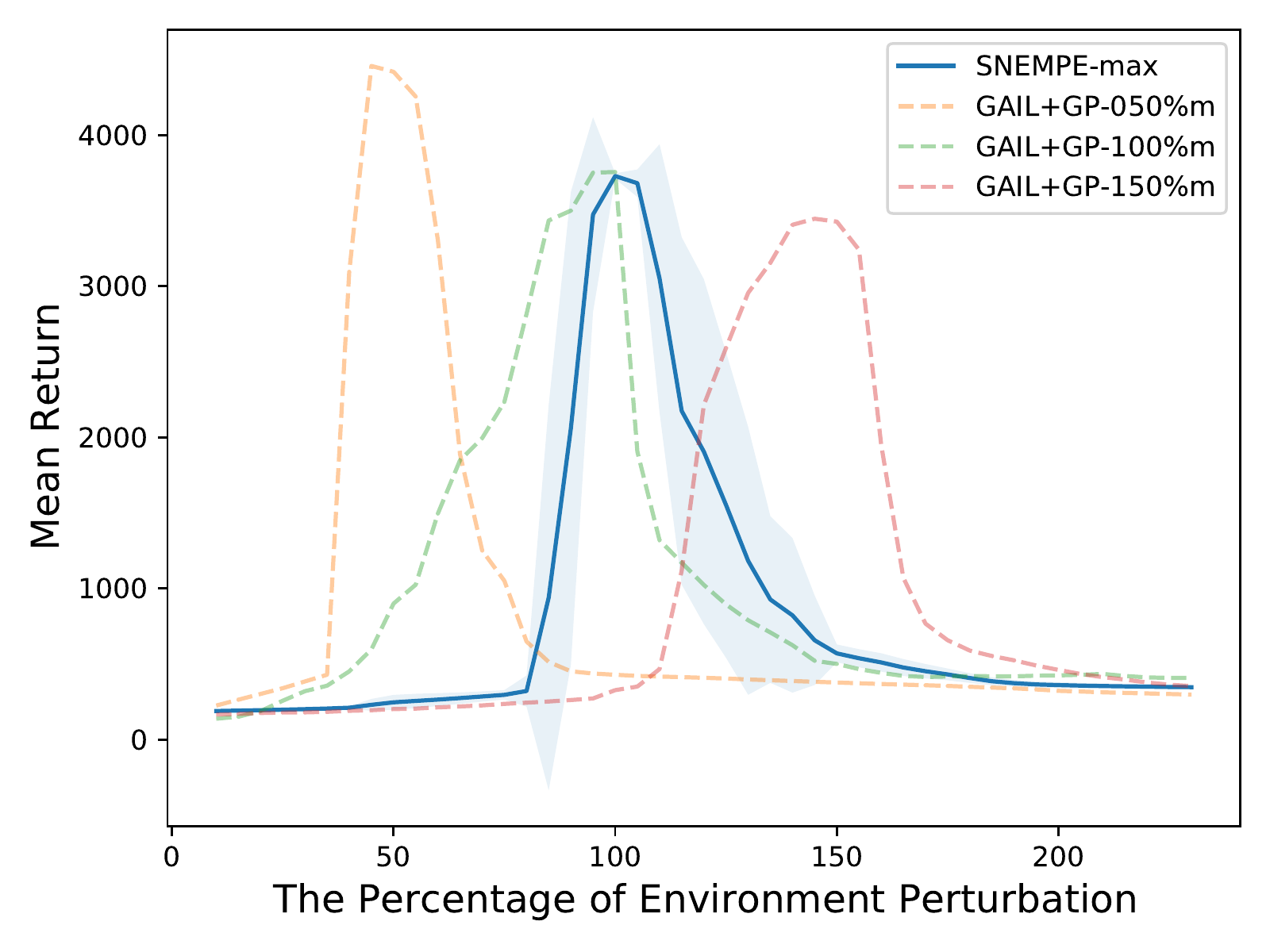}
        \captionsetup{justification=centering}
        \caption{Hopper+Mass}
    \end{subfigure}
    \begin{subfigure}[b]{0.24\textwidth}
        \centering
        \includegraphics[width=\textwidth]{RIME_figures/singleenv_multipletau_wg_perf.pdf}
        \captionsetup{justification=centering}
        \caption{Walker2d+Gravity}
    \end{subfigure}
    \begin{subfigure}[b]{0.24\textwidth}
        \centering
        \includegraphics[width=\textwidth]{RIME_figures/singleenv_multipletau_wm_perf.pdf}
        \captionsetup{justification=centering}
        \caption{Walker2d+Mass}
    \end{subfigure}

    \begin{subfigure}[b]{0.24\textwidth}
        \centering
        \includegraphics[width=\textwidth]{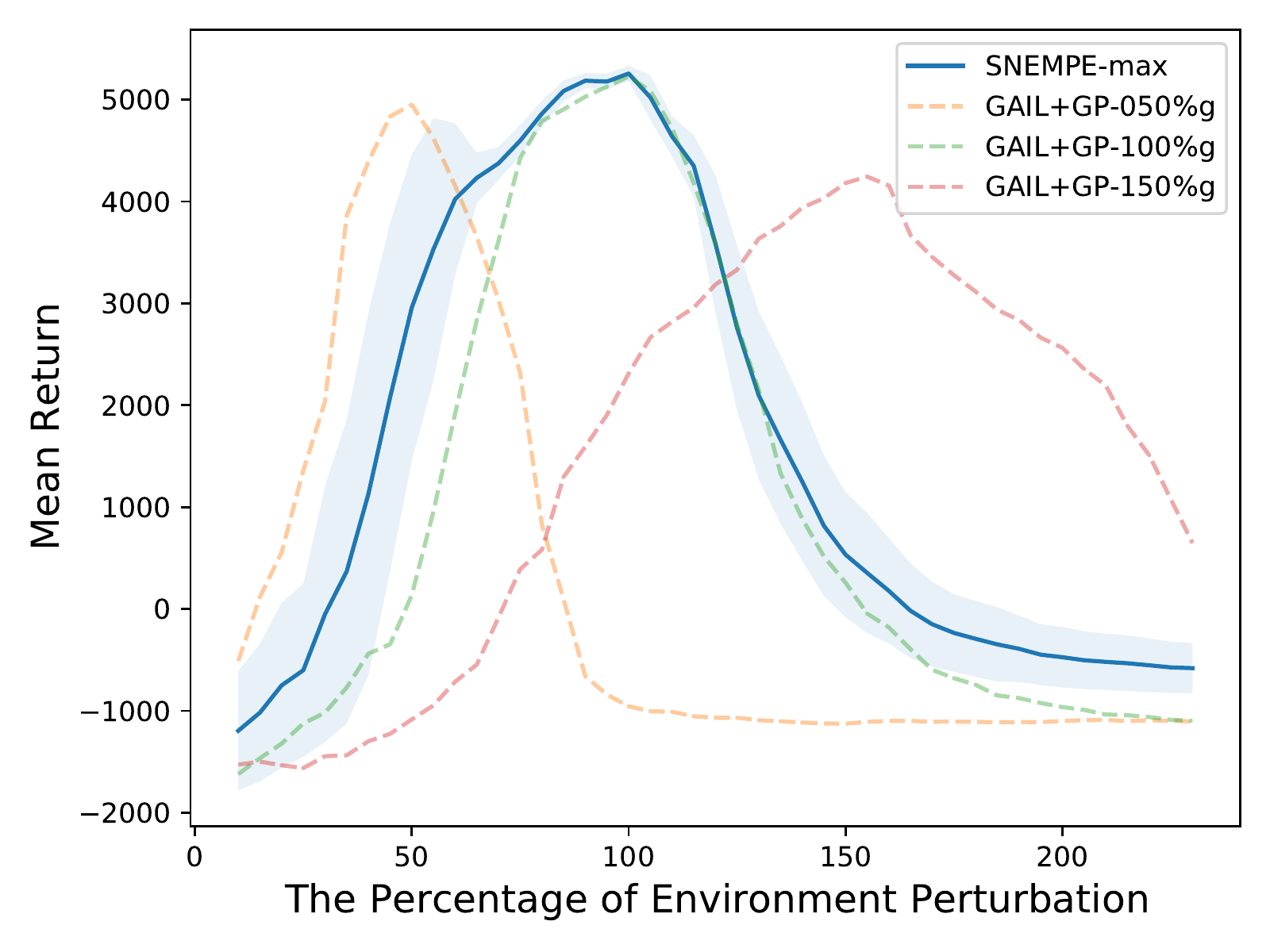}
        \captionsetup{justification=centering}
        \caption{HalfCheetah+Gravity}
    \end{subfigure}
    \begin{subfigure}[b]{0.24\textwidth}
        \centering
        \includegraphics[width=\textwidth]{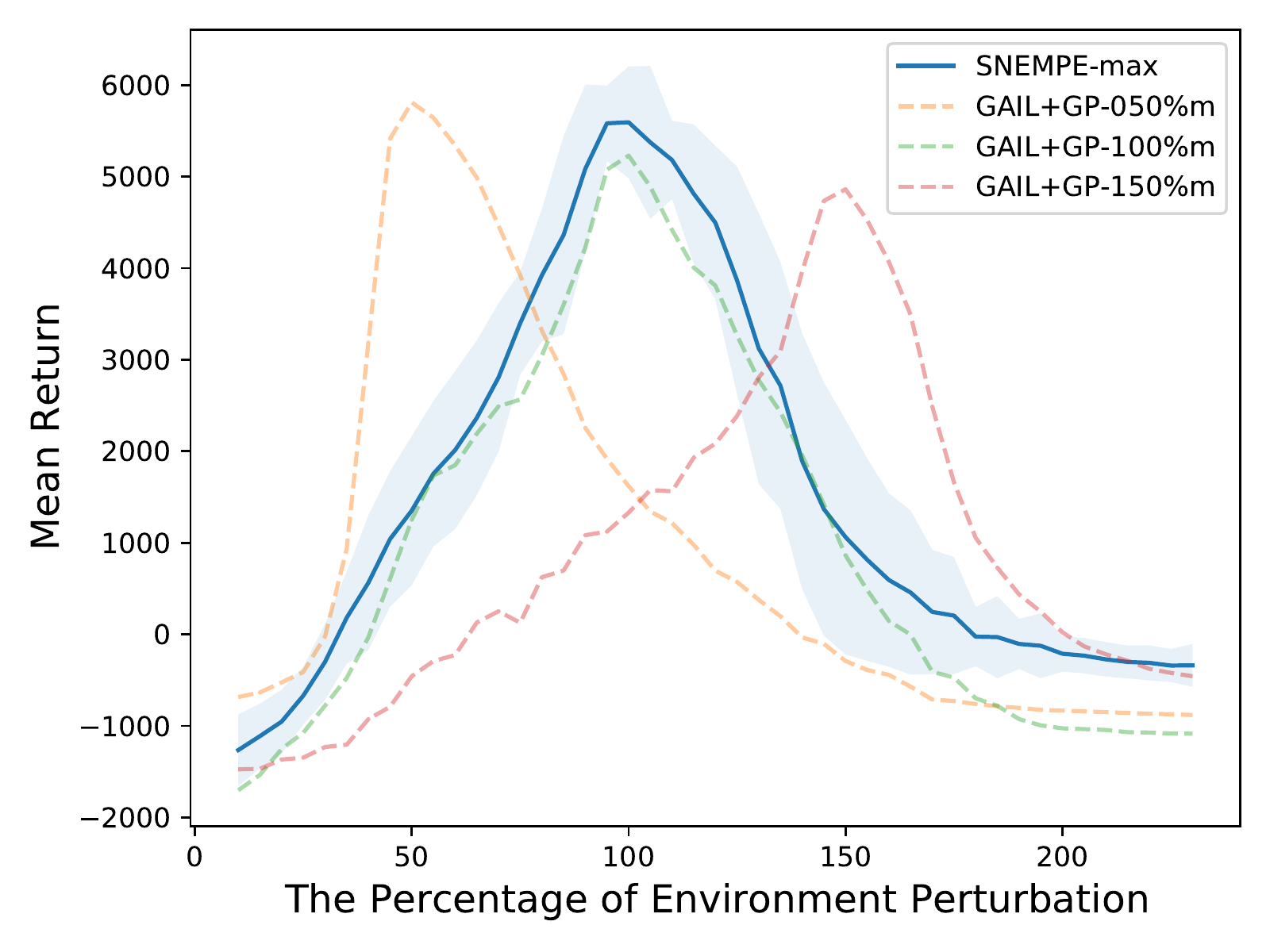}
        \captionsetup{justification=centering}
        \caption{HalfCheetah+Mass}
    \end{subfigure}
    \begin{subfigure}[b]{0.24\textwidth}
        \centering
        \includegraphics[width=\textwidth]{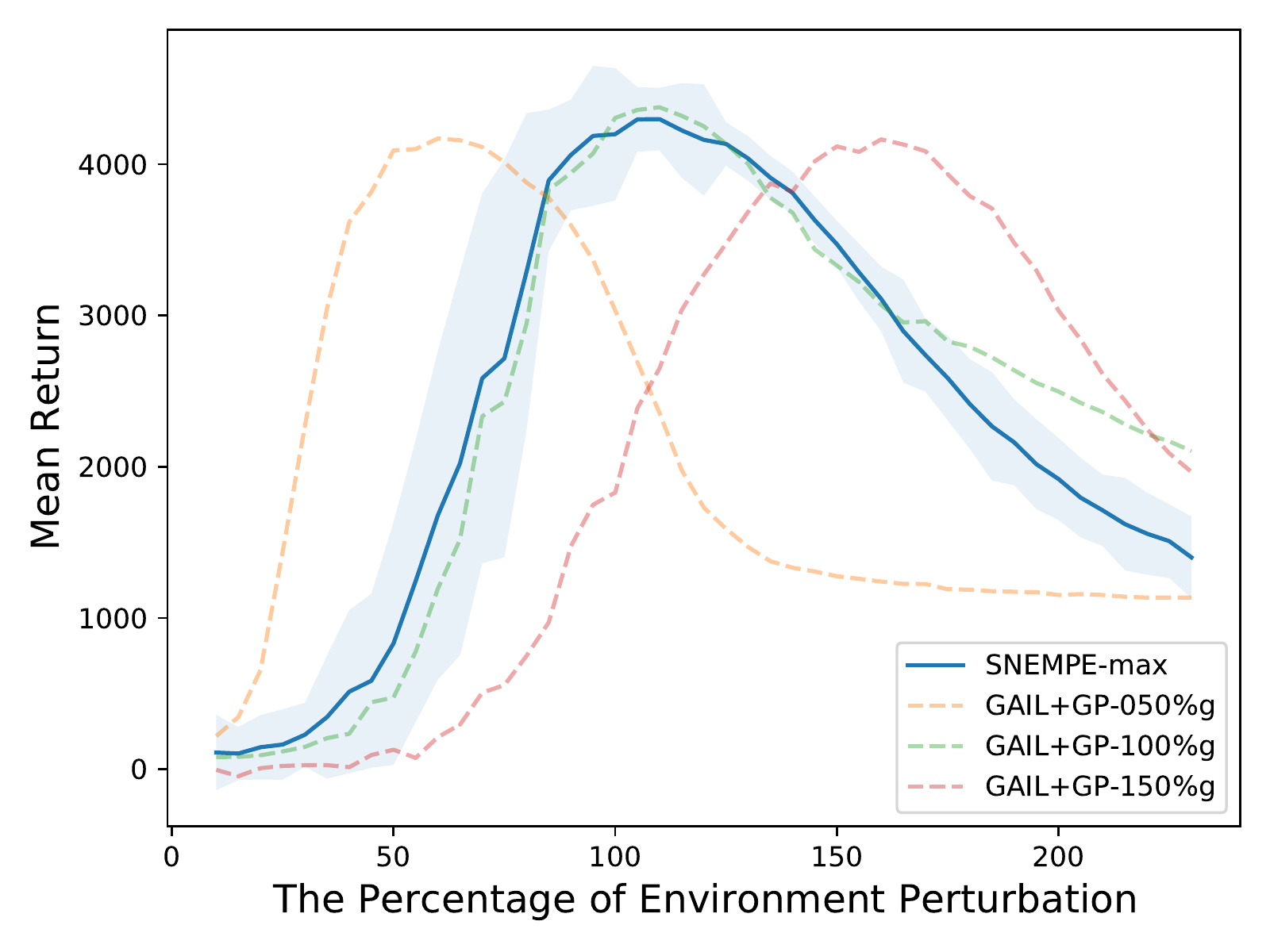}
        \captionsetup{justification=centering}
        \caption{Ant+Gravity}
    \end{subfigure}
    \begin{subfigure}[b]{0.24\textwidth}
        \centering
        \includegraphics[width=\textwidth]{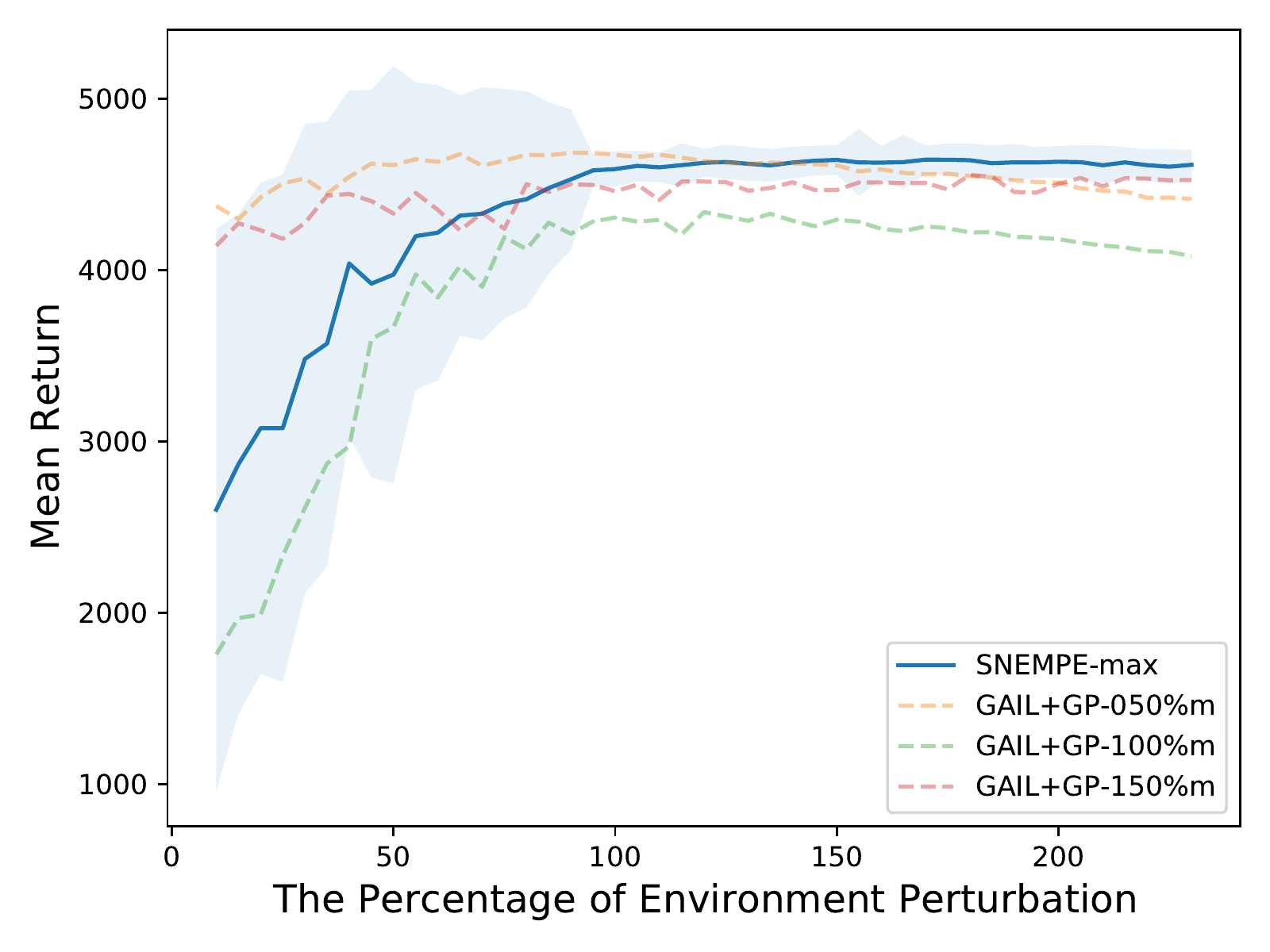}
        \captionsetup{justification=centering}
        \caption{Ant+Mass}
    \end{subfigure}
    \caption{The performance of SNEMPE-max over environment perturbations. \label{appendixD:figures:singleenv_multipletau}}
\end{figure}

\newpage

\subsection{Ablation Study with  Importance Sampling Ratio $\frac{\mu_{\pi}^i(s)}{\mu_E^j(s)}$ Estimator}
\label{appendix:ablation_IS}

To exactly compute $\mu_{\pi}^i(s)$ in the importance sampling ratio $\frac{\mu_{\pi}^i(s)}{\mu_E^j(s)}$ in \eqref{eq:modified_result_for_RIME}, we need many interactions with the interaction environment $\mathcal{E}_i$, which could increase the sample complexity in practice. To avoid this sample complexity issue, we can estimate the ratio $\frac{\mu_{\pi}^i(s)}{\mu_E^j(s)}$ directly.  With an estimated ratio $\frac{\mu_{\pi}^i(s)}{\mu_E^j(s)}$, eq. \eqref{eq:modified_result_for_RIME} is replaced with 
{\small\begin{align}
    \label{appendix_eq:modified_result_for_RIME}
    \min_{\pi}\sum_{i=1}^{N}\sum_{j=1}^{N}\max_{D_{ij}}\left\{\mathbb{E}_{\rho_{\pi}^i}\left[\lambda_j(s)\log(1-D_{ij}(s,a))\right]+\mathbb{E}_{\rho_{E}^j}\left[\tilde{w}_{ij}(s)\lambda_j(s)\log(D_{ij}(s,a))\right]\right\},
\end{align}}where $\tilde{w}_{ij}(s)$ is a given estimator of the ratio $\frac{\mu_{\pi}^i(s)}{\mu_E^j(s)}$. So, the policy in \eqref{appendix_eq:modified_result_for_RIME} affects only the first term $\mathbb{E}_{\rho_{\pi}^i}[\cdot]$ and hence the objective function for the policy update is the same as \eqref{eq:final_objective_policy}. The objective function for the discriminator $D_{ij}$ is given by
\begin{align}
    \label{appendix_eq:objective_discriminator_with_IS}
    \max_{D_{ij}}\mathbb{E}_{\rho_{\pi}^i}\left[\log(1-D_{ij}(s,a))\right]+\mathbb{E}_{\rho_{E}^j}\left[\tilde{w}_{ij}(s)\log(D_{ij}(s,a))\right]
\end{align}
In the same way as in \cref{theorem2}, the optimal discriminator  is given by $D_{ij}^*=\frac{\tilde{w}_{ij}(s)\lambda_j(s)\rho_E^j(s,a)}{\lambda_j(s)\rho_{\pi}^i(s,a)+\tilde{w}_{ij}(s)\lambda_j(s)\rho_E^j(s,a)}=\frac{\tilde{w}_{ij}(s)\rho_E^j(s,a)}{\rho_{\pi}^i(s,a)+\tilde{w}_{ij}(s)\rho_E^j(s,a)}$.

\vspace{1em}

\begin{figure}[h]
    \centering
    \vskip -0.1in
    \begin{subfigure}[b]{0.24\textwidth}
        \centering
        \includegraphics[width=\linewidth,height=3.0cm]{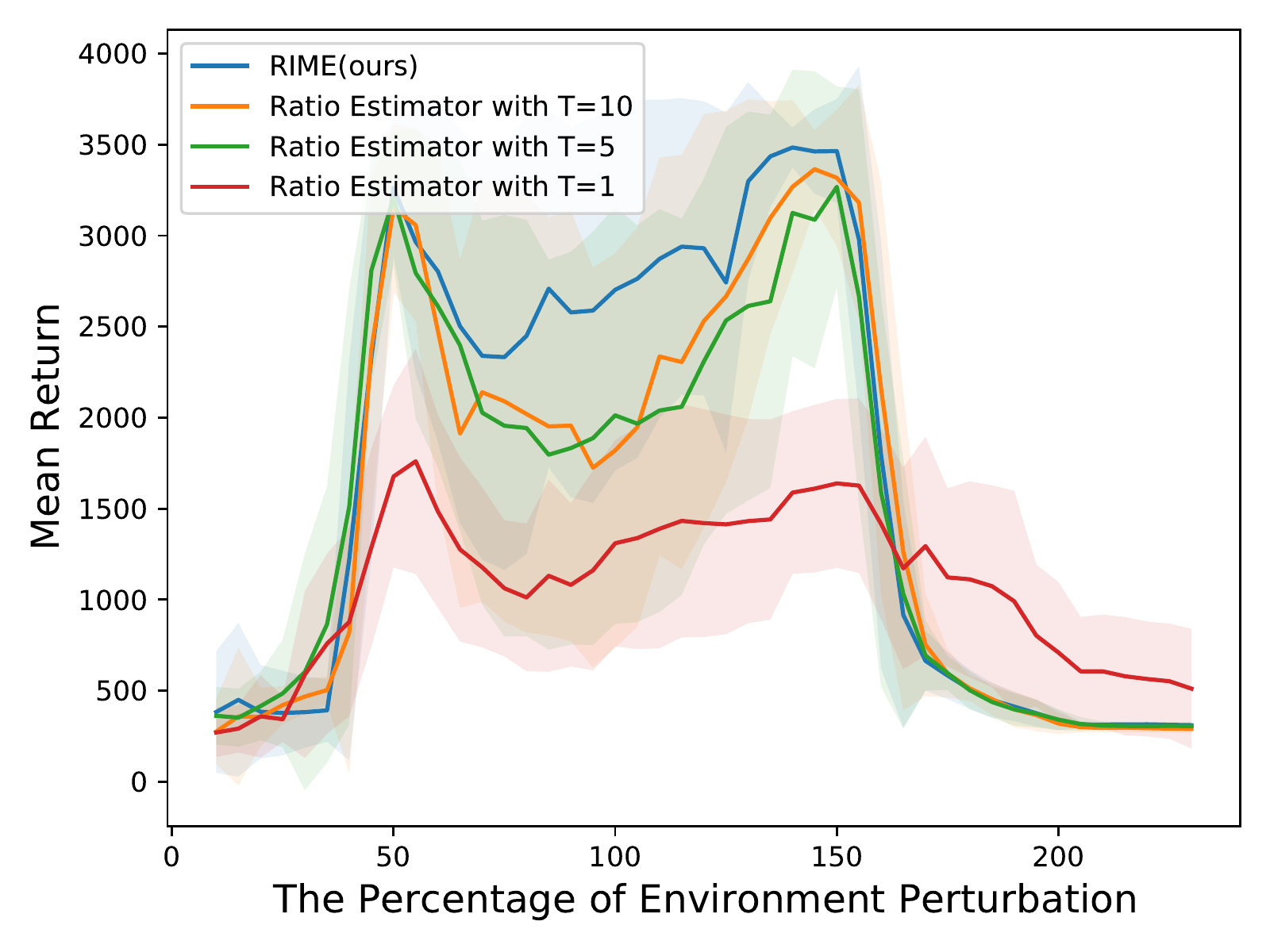}
        \vskip -0.05in
        \captionsetup{justification=centering}
        \caption{Hopper+Gravity}
    \end{subfigure}
    \begin{subfigure}[b]{0.24\textwidth}
        \centering
        \includegraphics[width=\linewidth,height=3.0cm]{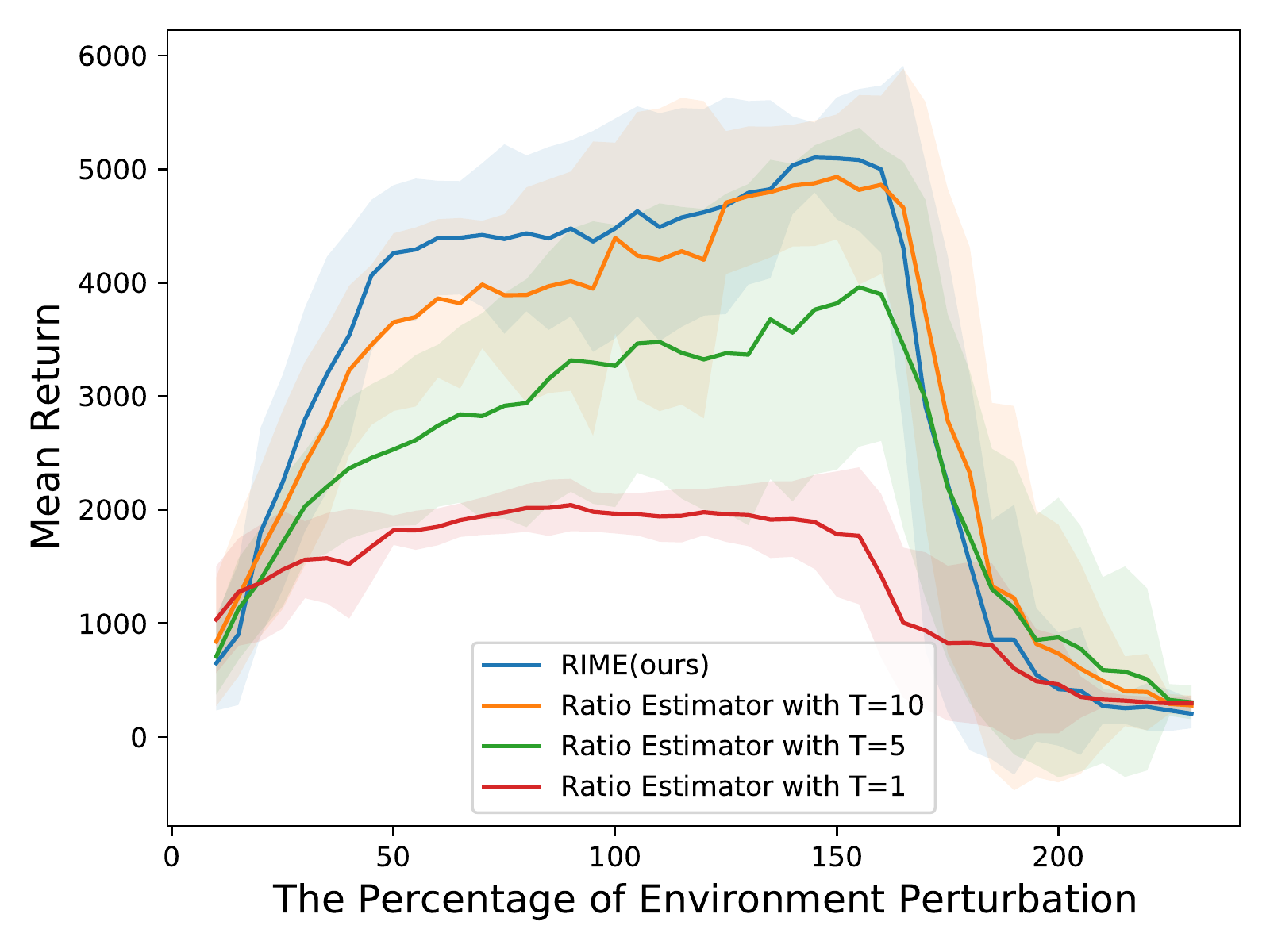}
        \vskip -0.05in
        \captionsetup{justification=centering}
        \caption{Walker2d+Gravity}
    \end{subfigure}
    \begin{subfigure}[b]{0.24\textwidth}
        \centering
        \includegraphics[width=\linewidth,height=3.0cm]{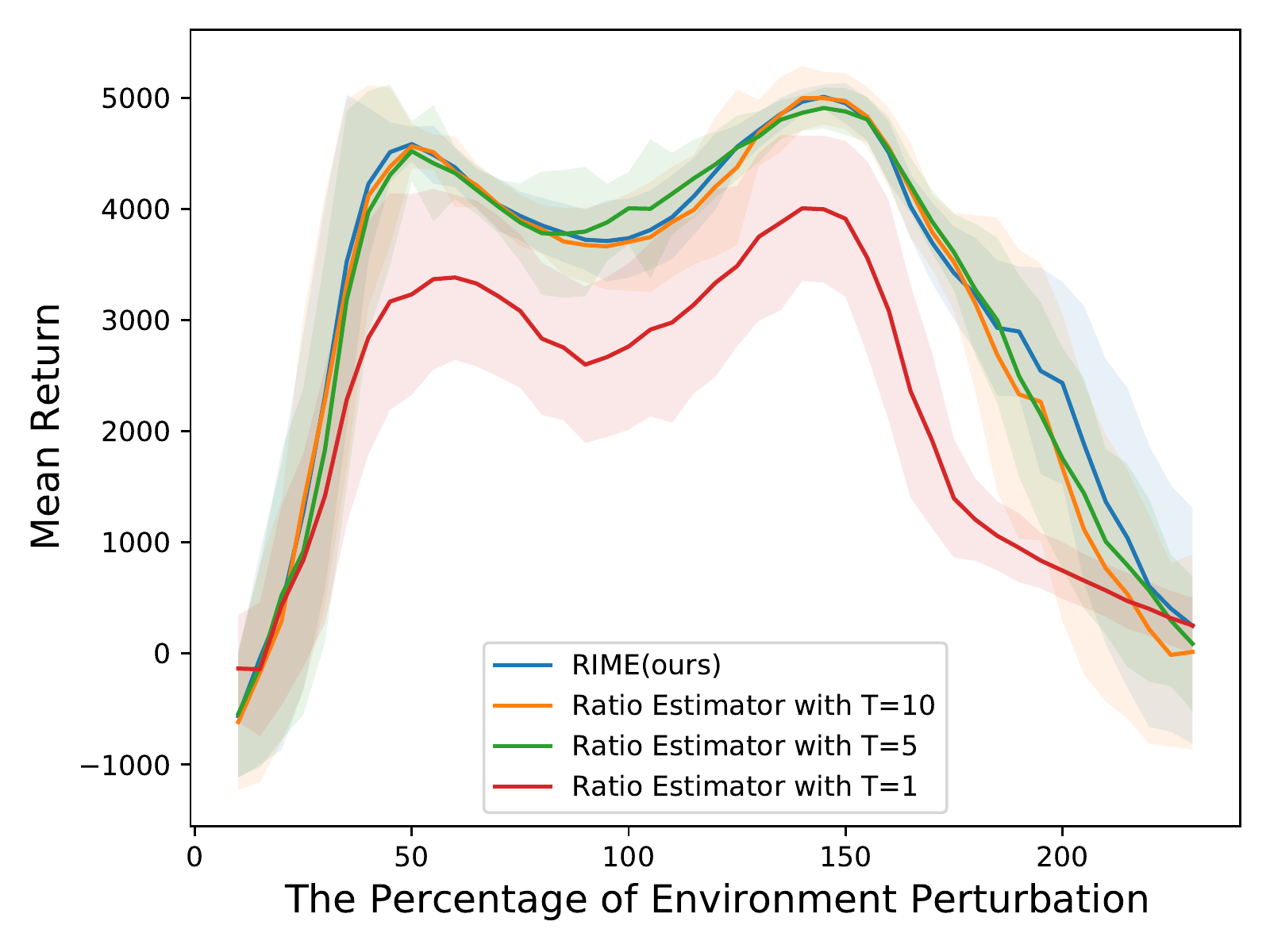}
        \vskip -0.05in
        \captionsetup{justification=centering}
        \caption{HalfCheetah+Gravity}
    \end{subfigure}
    \begin{subfigure}[b]{0.24\textwidth}
        \centering
        \includegraphics[width=\linewidth,height=3.0cm]{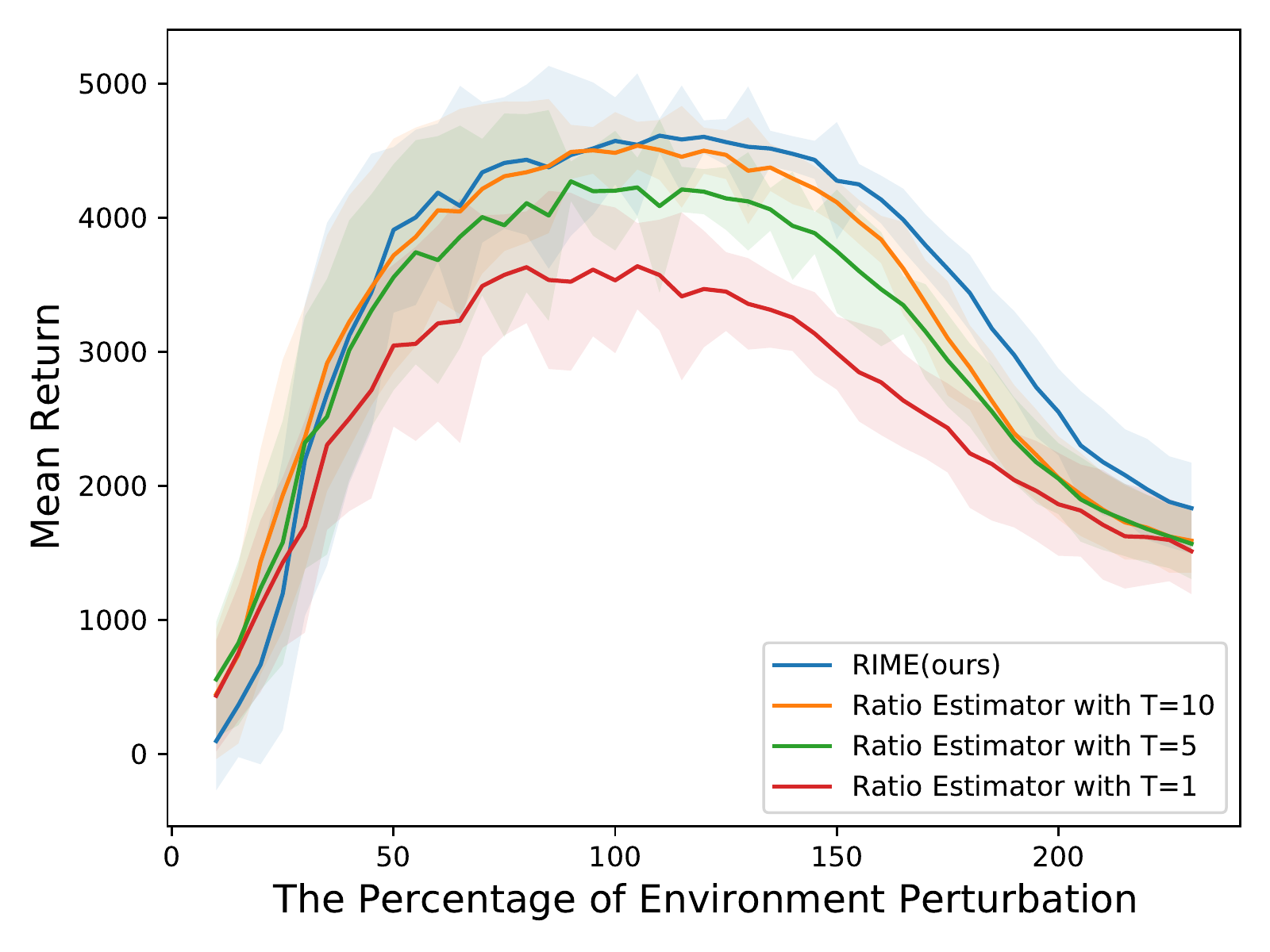}
        \vskip -0.05in
        \captionsetup{justification=centering}
        \caption{Ant+Gravity}
    \end{subfigure}
    \begin{subfigure}[b]{0.24\textwidth}
        \centering
        \includegraphics[width=\linewidth,height=3.0cm]{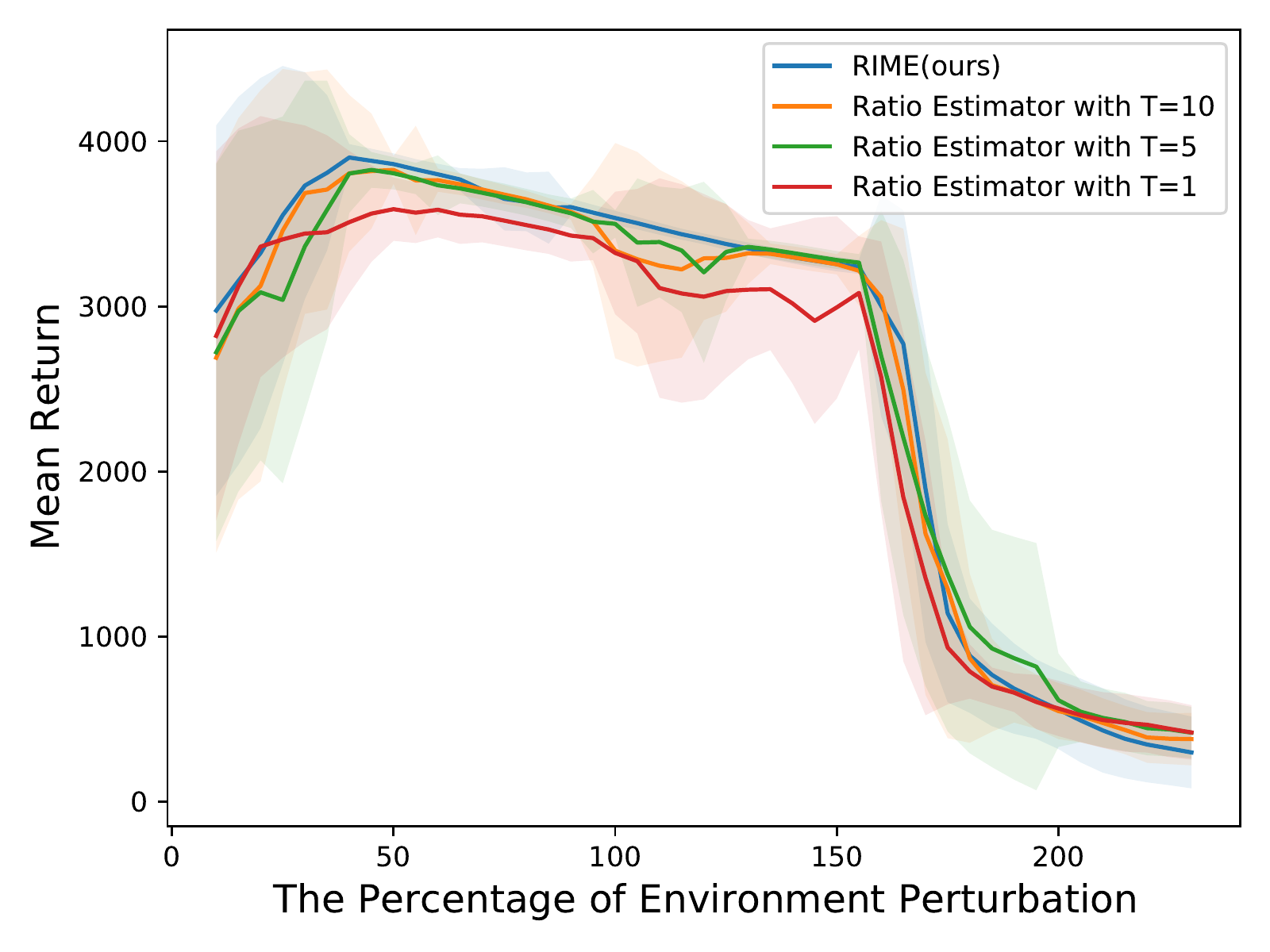}
        \vskip -0.05in
        \captionsetup{justification=centering}
        \caption{Hopper+Mass}
    \end{subfigure}
    \begin{subfigure}[b]{0.24\textwidth}
        \centering
        \includegraphics[width=\linewidth,height=3.0cm]{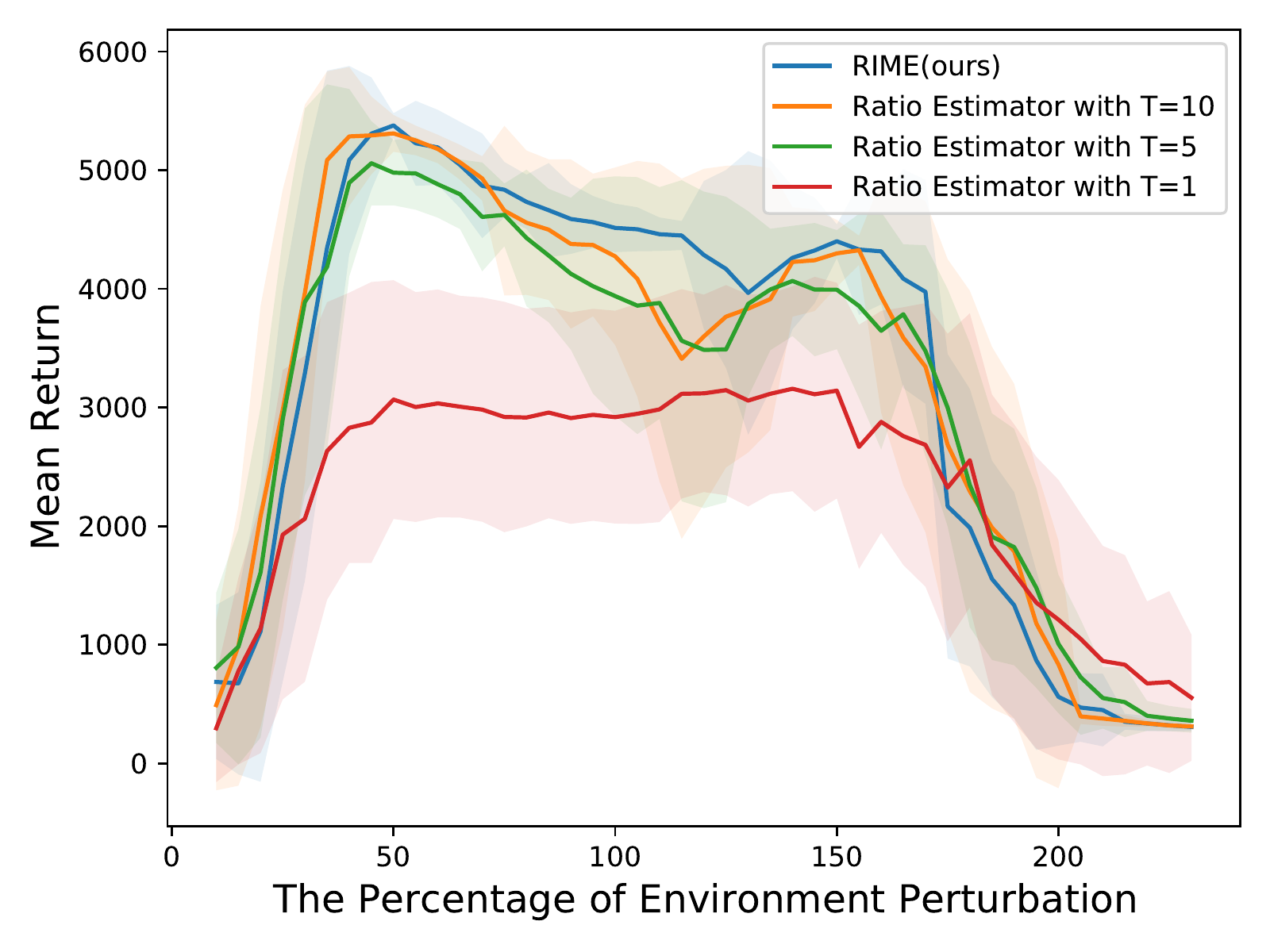}
        \vskip -0.05in
        \captionsetup{justification=centering}
        \caption{Walker2d+Mass}
    \end{subfigure}
    \begin{subfigure}[b]{0.24\textwidth}
        \centering
        \includegraphics[width=\linewidth,height=3.0cm]{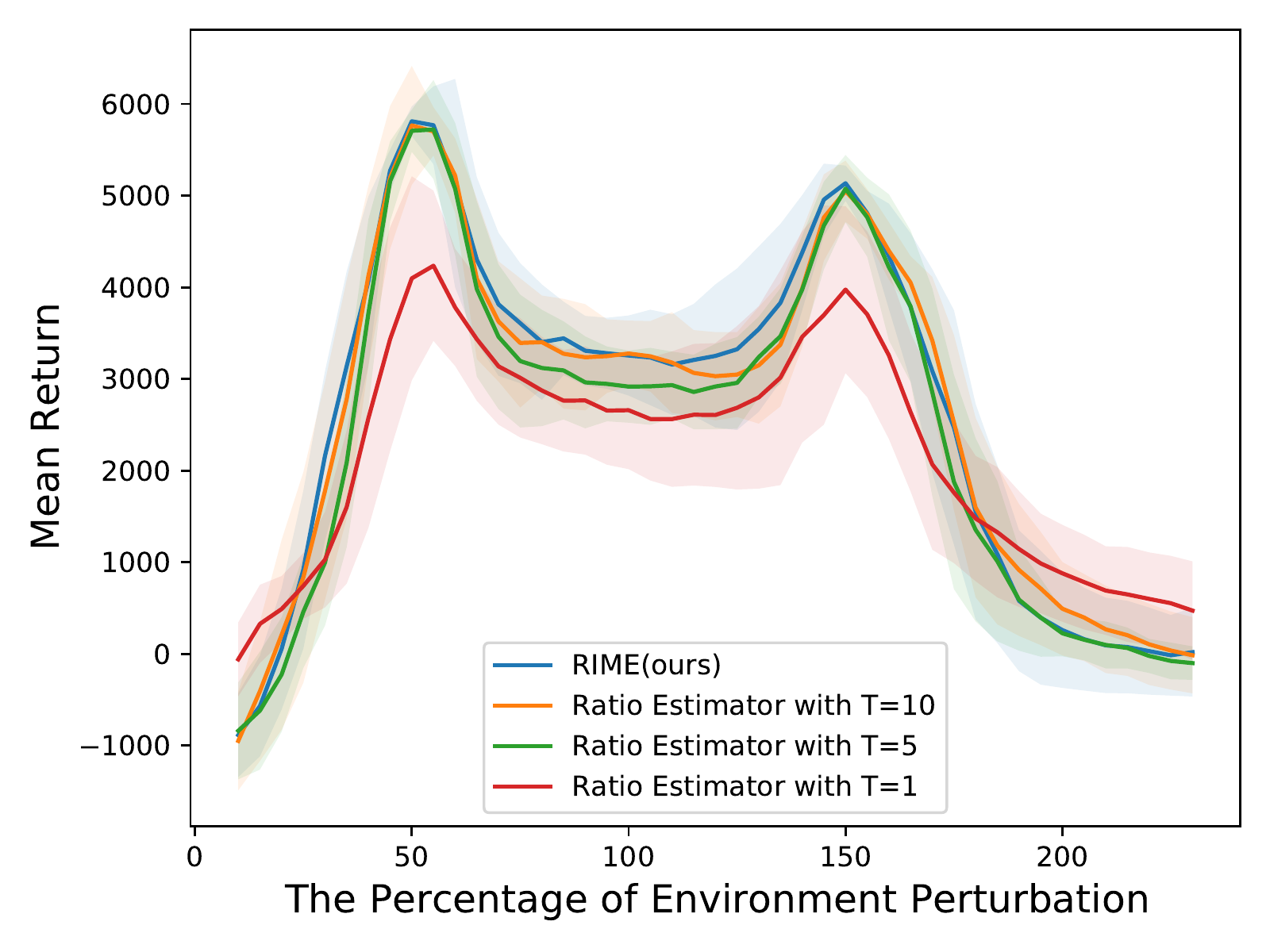}
        \vskip -0.05in
        \captionsetup{justification=centering}
        \caption{HalfCheetah+Mass}
    \end{subfigure}
    \begin{subfigure}[b]{0.24\textwidth}
        \centering
        \includegraphics[width=\linewidth,height=3.0cm]{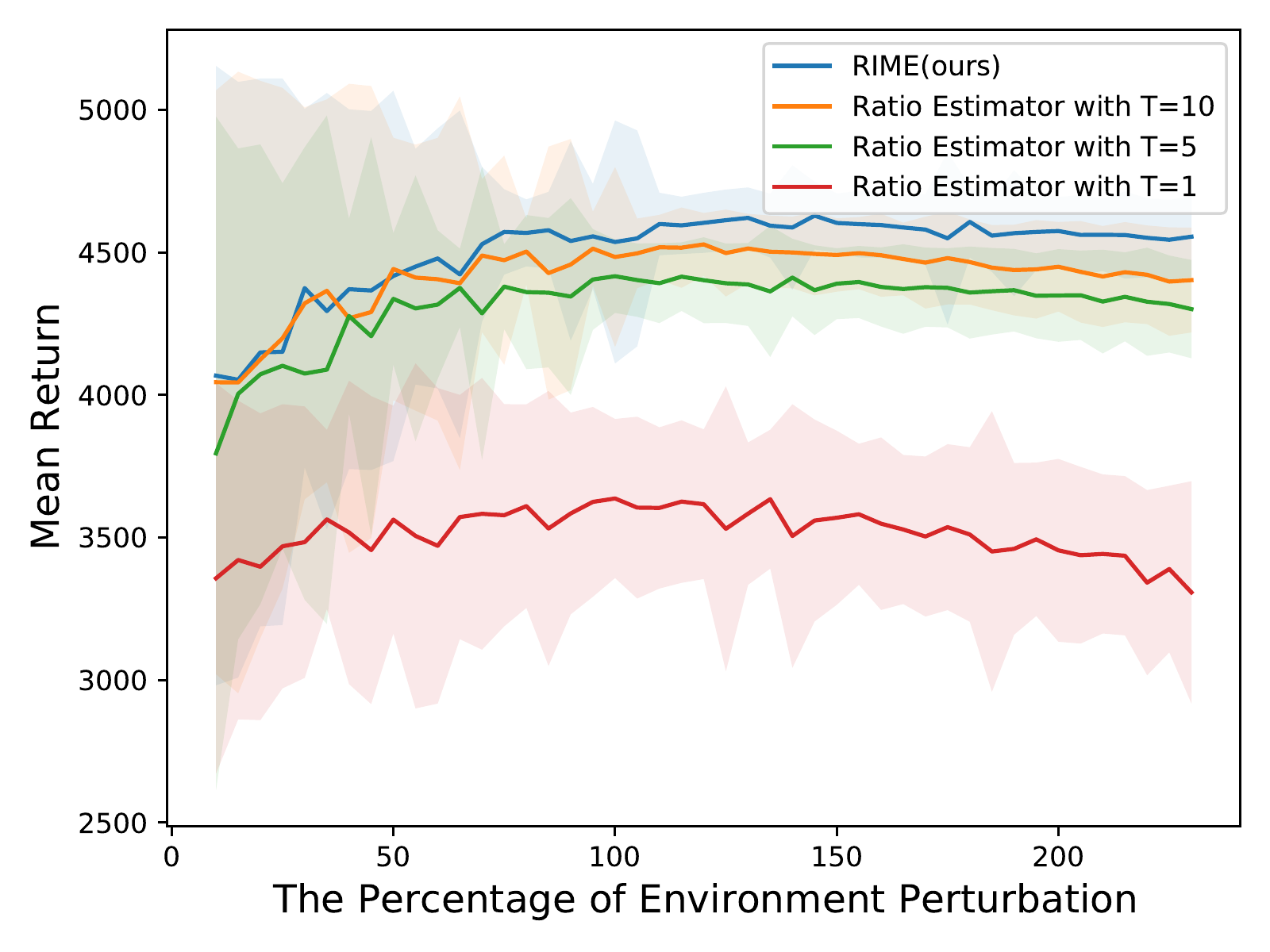}
        \vskip -0.05in
        \captionsetup{justification=centering}
        \caption{Ant+Mass}
    \end{subfigure}
    \vskip -0.1in
    \caption{Comparison with LFIW method in the $N=2$ sampled environments case\label{appendix_fig:LFIW}}
\end{figure}

Now we explain how to estimate the ratio $\tilde{w}_{ij}(s):=\frac{\mu_{\pi}^i(s)}{\mu_E^j(s)}$ by a method of estimating  probability measure ratio proposed in  \cite{IS1:LFIW}. They proposed the Likelihood-Free Importance Weights (LFIW) method, which estimates the ratio of two probability measures by using the lower bound of f-divergence between the two measures. They showed that in \eqref{appendix_eq:f-divergence_bound}, the equality holds at $w=\frac{dP}{dQ}$, so they estimated the probability measure ratio $w(x)$ by maximizing \eqref{appendix_eq:f-divergence_bound}:
\begin{align}
    \label{appendix_eq:f-divergence_bound}
    D_{f}(P||Q)\geq\mathbb{E}_{P}\left[f'(w(x))\right] - \mathbb{E}_{Q}\left[f^*(f'(w(x)))\right],
\end{align}
where $P$ and $Q$ are probability measures, and $D_f$ is an f-divergence. However, directly using the probability measure ratio $w(x)$  may cause learning failure due to the finite sample size issue in practice. To address this issue, the LFIW applies the self-normalization to the probability measures ratio $w(x)$ over $Q$ with a temperature hyperparameter $T$.
\begin{align}
    \tilde{w}(x)=\frac{w(x)^{1/T}}{\mathbb{E}_{Q}\left[w(x)^{1/T}\right]}
\end{align}
By replacing $P$ and $Q$ with $\mu_{\pi}^i$ and $\mu_E^j$, we can estimate the importance sampling ratio $\frac{\mu_{\pi}^i(s)}{\mu_E^j(s)}$. \cref{appendix_fig:LFIW} shows that our proposed method (RIME) which simply sets $\frac{\mu_{\pi}^i(s)}{\mu_{E}^j(s)}$ to $1$ has almost same performance as the proposed method using the estimated  importance sampling ratio by LFIW for all tasks.

\newpage

\subsection{Ablation Study with  State-only Expert Demonstration}
\label{appendix:ablation_GAIfO}

GAIfO \cite{il6:GAIFO} uses state-only expert demonstration and reproduces the expert policy $\pi_E$ by matching the  state-transition occupancy measures induced by the $\pi$ and $\pi_E$. Our algorithm (RIME) and other GAIL variant algorithms can  directly be applied to this setting by using  state-only expert demonstration instead of state-action expert demonstration. We refer to these methods as GAIfO-RIME, GAIfO-OMME, GAIfO-mixture, GAIfO-single.

We tested these GAIfO variants in the $N=2$ sampled environment case (50\% and 150\%). \cref{2env_table_GAIfO_variants} and \cref{appendix_fig:2env_gaifo_variants} show similar results to \cref{2env_table} and \cref{figure:2envs_results} (the case with the state-action expert demonstration) for all the tasks except for Walker2d+Gravity and Ant. For Walker2d+Gravity, GAIfO-mixture and GAIfO-single have good performance around the interaction environments, but they are over-fitted  to these environments and do not  perform near the test  environment with $\zeta_0$. On the other hand,  our method (GAIfO-RIME) performs well near the test  environment with $\zeta_0$. 
Therefore, the experimental results show that our method can properly recover the experts’ preference over the state space.
In the case of  Ant+Gravity and Ant+Mass, all algorithms failed to learn, and we think this is due to the difficulty of optimization due to the large state space of the Ant task.

\begin{table*}[h]
    \caption{Mean return / minimum return of GAIfO variants over the dynamics parameter range [50\%,150\%] in the $N=2$ sampled environment case}
    \label{2env_table_GAIfO_variants}
    \centering
    \begin{footnotesize}
    \begin{tabular}{l||c|c|c}
    \toprule
    Algorithm & Hopper +Gravity & Walker2d +Gravity & HalfCheetah +Gravity \\
    \midrule
    GAIfO-RIME (ours) & \textbf{2758.9} / \textbf{2318.4} & \textbf{3767.7} / \textbf{3331.5} & \textbf{4247.9} / \textbf{3786.1} \\
    GAIfO-OMME & 1491.4 / 1063.3 & 2918.1 / 1935.6 & 3600.4 / 3152.0 \\
    GAIfO-mixture & 1424.8 / 719.6 & 3551.4 / 1808.0 & 3446.6 / 2851.8 \\
    GAIfO-single & 1376.6 / 765.7 & 3346.0 / 1384.0 & 3208.0 / 2252.2 \\
    \midrule\midrule
    Algorithm     & Hopper +Mass & Walker2d +Mass & HalfCheetah +Mass \\
    \midrule
    GAIfO-RIME (ours) & \textbf{3496.5} / \textbf{3335.5} & \textbf{4757.3} / \textbf{4336.3} & \textbf{4247.6} / \textbf{3611.5} \\
    GAIfO-OMME & 2937.7 / 2421.3 & 4154.2 / 3657.6 & 3906.5 / 3172.2 \\
    GAIfO-mixture & 3188.3 / 2953.5 & 3872.0 / 2769.2 & 3239.8 / 2670.5 \\
    GAIfO-single & 2269.2 / 1622.9 & 3478.5 / 1582.6 & 2947.1 / 1653.3 \\
    \bottomrule
    \end{tabular}
    \end{footnotesize}
\end{table*}
\begin{figure*}[h]
    \centering
    \begin{subfigure}[b]{0.24\textwidth}
        \centering
        \includegraphics[width=\textwidth]{RIME_figures/gaifo_2env_compare_hog_compare.pdf}
        \captionsetup{justification=centering}
        \caption{Hopper+Gravity}
    \end{subfigure}
    \begin{subfigure}[b]{0.24\textwidth}
        \centering
        \includegraphics[width=\textwidth]{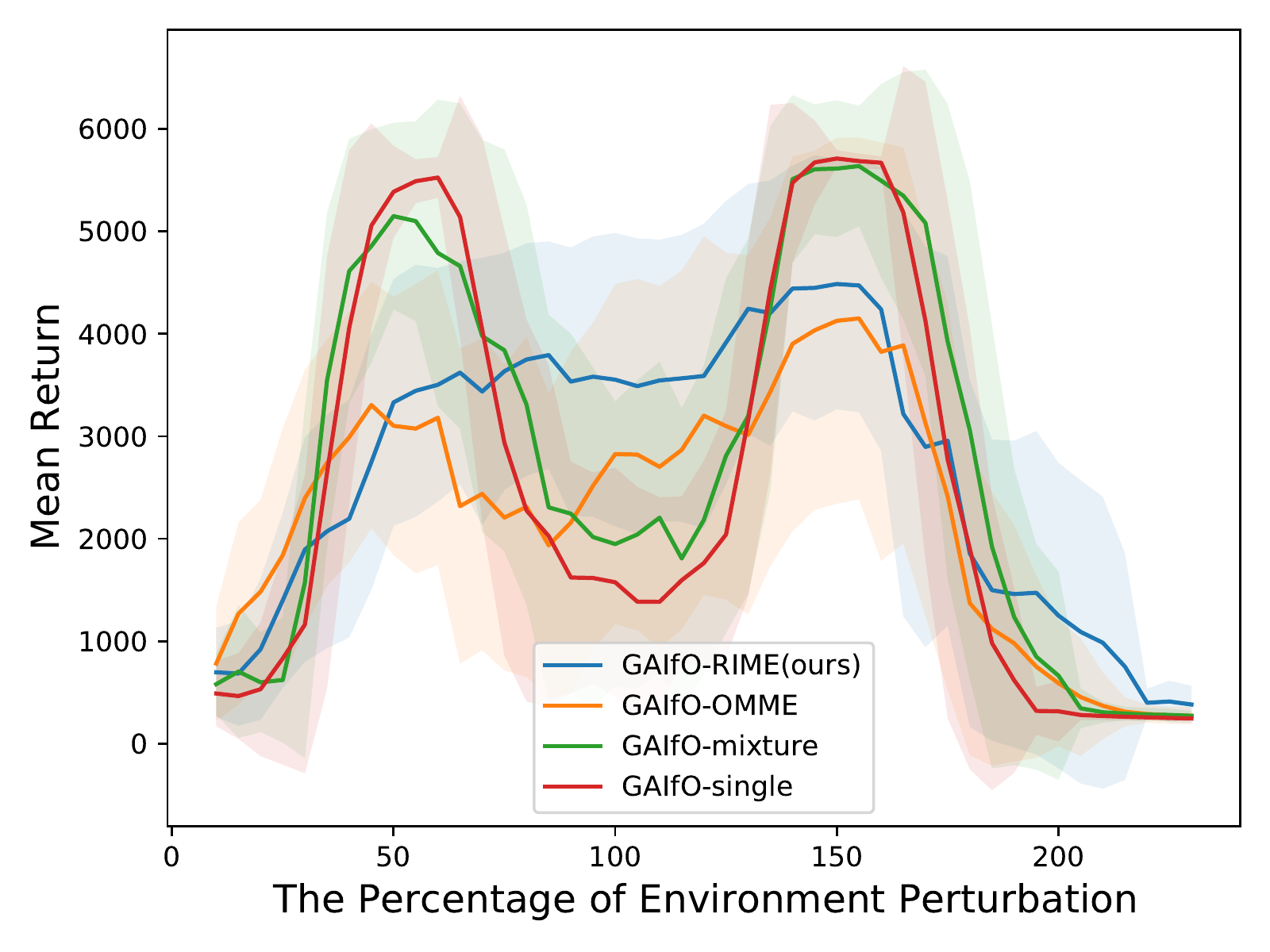}
        \captionsetup{justification=centering}
        \caption{Walker2d+Gravity}
    \end{subfigure}
    \begin{subfigure}[b]{0.24\textwidth}
        \centering
        \includegraphics[width=\textwidth]{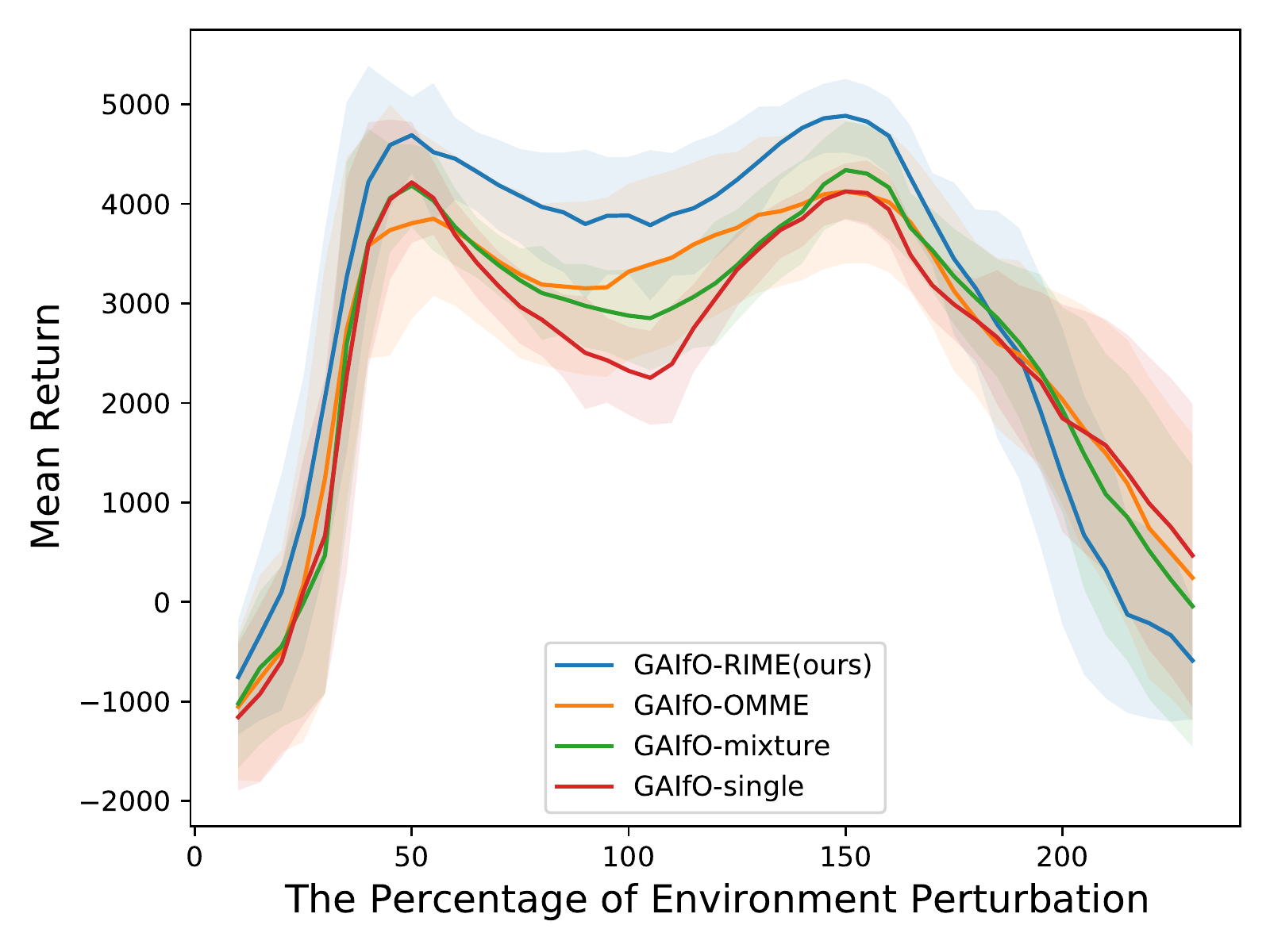}
        \captionsetup{justification=centering}
        \caption{HalfCheetah+Gravity}
    \end{subfigure}
    \vfill
    \begin{subfigure}[b]{0.24\textwidth}
        \centering
        \includegraphics[width=\textwidth]{RIME_figures/gaifo_2env_compare_hom_compare.pdf}
        \captionsetup{justification=centering}
        \caption{Hopper+Mass}
    \end{subfigure}
    \begin{subfigure}[b]{0.24\textwidth}
        \centering
        \includegraphics[width=\textwidth]{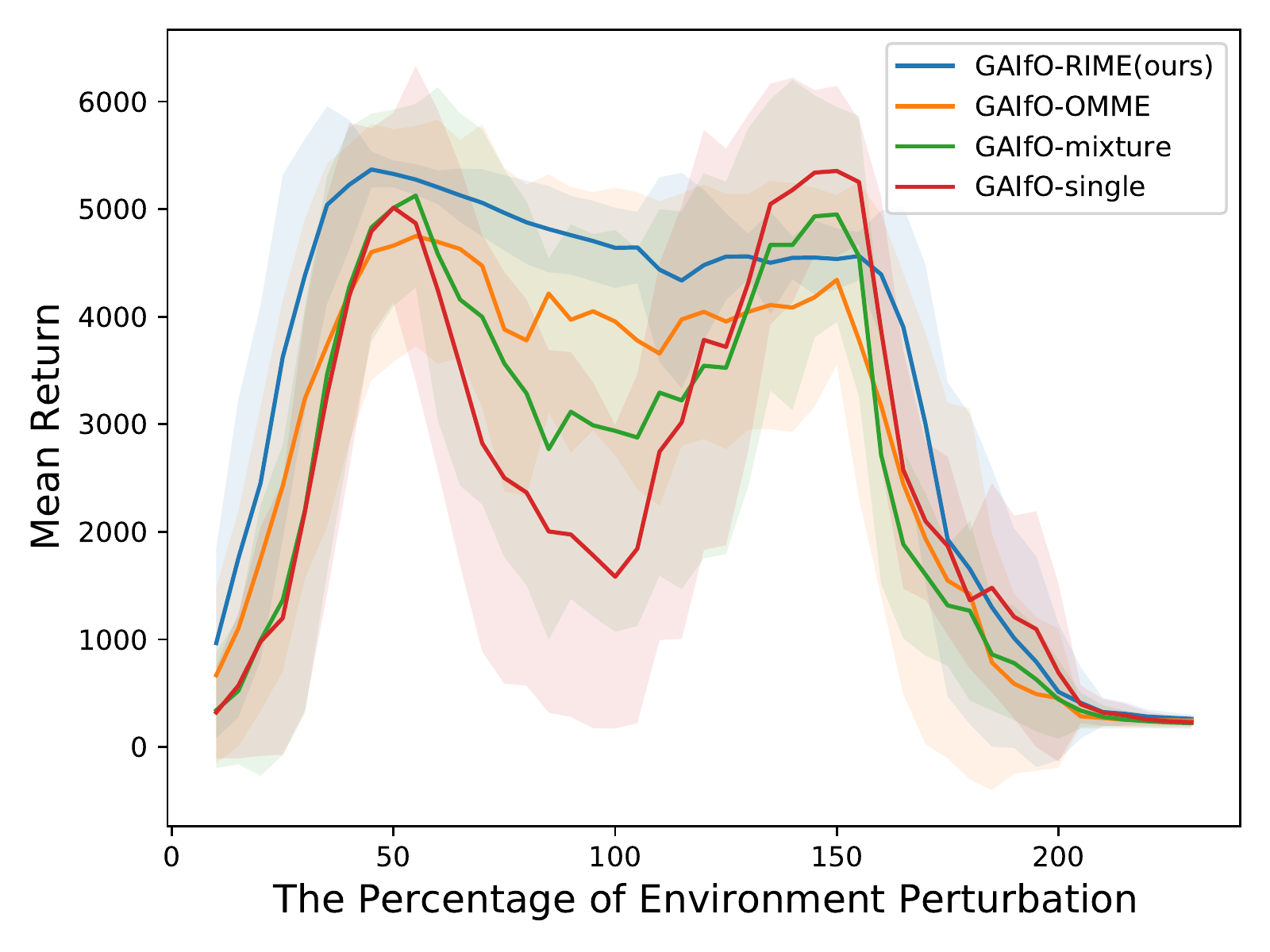}
        \captionsetup{justification=centering}
        \caption{Walker2d+Mass}
    \end{subfigure}
    \begin{subfigure}[b]{0.24\textwidth}
        \centering
        \includegraphics[width=\textwidth]{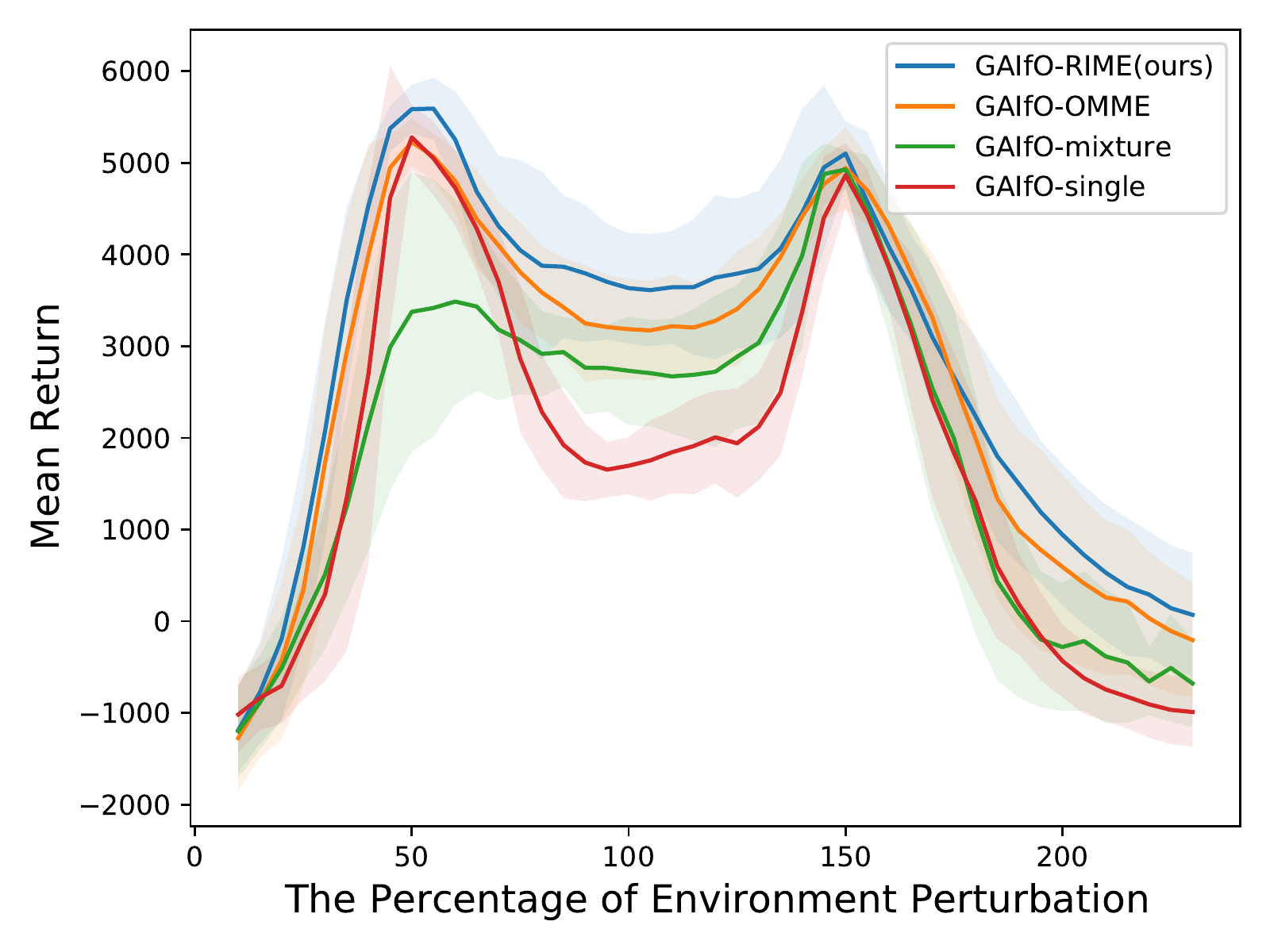}
        \captionsetup{justification=centering}
        \caption{HalfCheetah+Mass}
    \end{subfigure}
    \caption{Comparisons with variants of GAIfO in the $N=2$ sampled environments case \label{appendix_fig:2env_gaifo_variants}}
\end{figure*}

\newpage

\subsection{Ablation Study according to the Size of Expert Demonstration}
\label{appendix:ablation_the_size_of_expert_demonstrations}

Considering the fact that expert demonstrations are costly to obtain,  we tested our algorithm by reducing the amount of expert demonstration from the 50 trajectories (each trajectory with 1000 samples).  

As seen in \cref{appendix:trajs_table}, for Hopper+Gravity, the robustness of our algorithm decreases as the size of expert demonstration decreases. However, for Ant+Gravity and Walker2d+Mass and HalfCheetah+Gravity and Ant+Mass, our algorithm using the reduced amount of expert demonstration still performs well. It seems that the amount of demonstration above a threshold is sufficient. 
\begin{table}[h]
    \caption{Mean return / minimum return over the dynamics parameter range [50\%,150\%] in the $N=2$ sampled environment case with expert demonstrations with various size}
    \label{appendix:trajs_table}
    \centering
    \begin{footnotesize}
    \begin{tabular}{c||c|c|c|c}
    \toprule
    \multirow{2}{4.5em}{\# of expert trajectories} & \multirow{2}{*}{Hopper +Gravity} & \multirow{2}{*}{Walker2d +Gravity} & \multirow{2}{*}{HalfCheetah +Gravity} & \multirow{2}{*}{Ant +Gravity} \\
     & & & & \\
    \midrule
    50 & 2886.7 / 2332.4 & 4577.1 / 4260.9 & 4268.9 / 3712.0 & 4402.2 / 3909.9 \\
    25 & 2774.7 / 2021.0 & 4455.5 / 4044.1 & 4290.3 / 3720.9 & 4445.5 / 3666.5 \\
    10 & 2570.8 / 1811.5 & 4243.1 / 3529.8 & 4244.5 / 3622.7 & 4554.7 / 4038.6 \\
    5 & 2323.2 / 1754.2 & 4514.0 / 3906.0 & 4219.5 / 3578.2 & 4562.7 / 3985.9 \\
    \midrule\midrule
    \multirow{2}{4.5em}{\# of expert trajectories} & \multirow{2}{*}{Hopper +Mass} & \multirow{2}{*}{Walker2d +Mass} & \multirow{2}{*}{HalfCheetah +Mass} & \multirow{2}{*}{Ant +Mass} \\
     & & & & \\
    \midrule
    50 & 3535.7 / 3255.6 & 4597.0 / 3965.8 & 3959.0 / 3156.7 & 4554.5 / 4417.5 \\
    25 & 3510.4 / 3246.3 & 4608.7 / 4134.2 & 4185.2 / 3320.4 & 4602.1 / 4395.6 \\
    10 & 3352.2 / 2895.2 & 4552.2 / 3937.8 & 4025.4 / 3222.4 & 4662.6 / 4524.2 \\
    5 & 3365.5 / 2965.8 & 4614.4 / 4073.2 & 3698.9 / 2718.1 & 4679.9 / 4496.8 \\
    \bottomrule
    \end{tabular}
    \end{footnotesize}
\end{table}

\begin{figure}[h]
    \centering
    \begin{subfigure}[b]{0.24\textwidth}
        \centering
        \includegraphics[width=\textwidth]{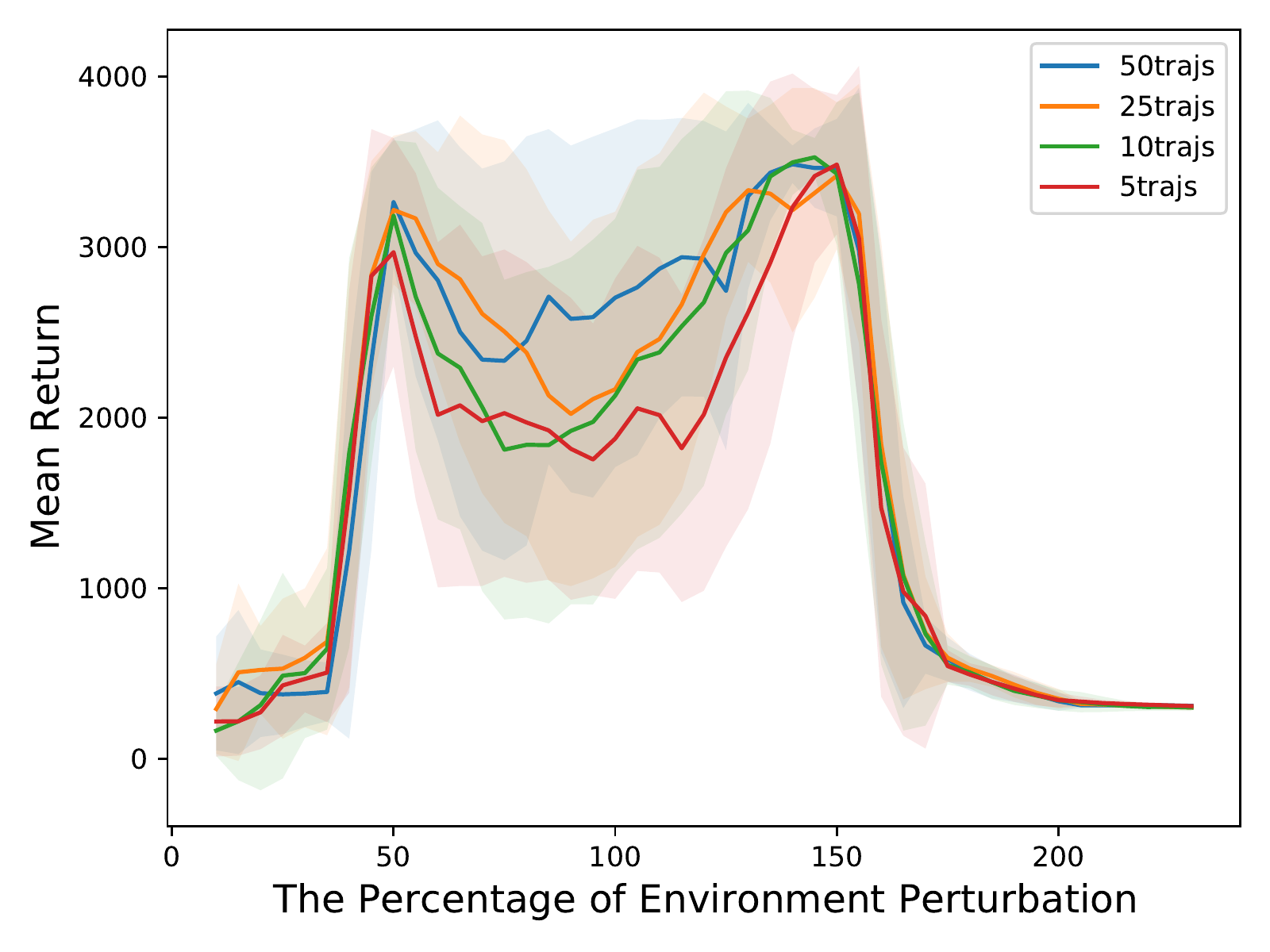}
        \captionsetup{justification=centering}
        \caption{Hopper+Gravity}
    \end{subfigure}
    \begin{subfigure}[b]{0.24\textwidth}
        \label{figure:biased_envs_results_hog_compare}
        \centering
        \includegraphics[width=\textwidth]{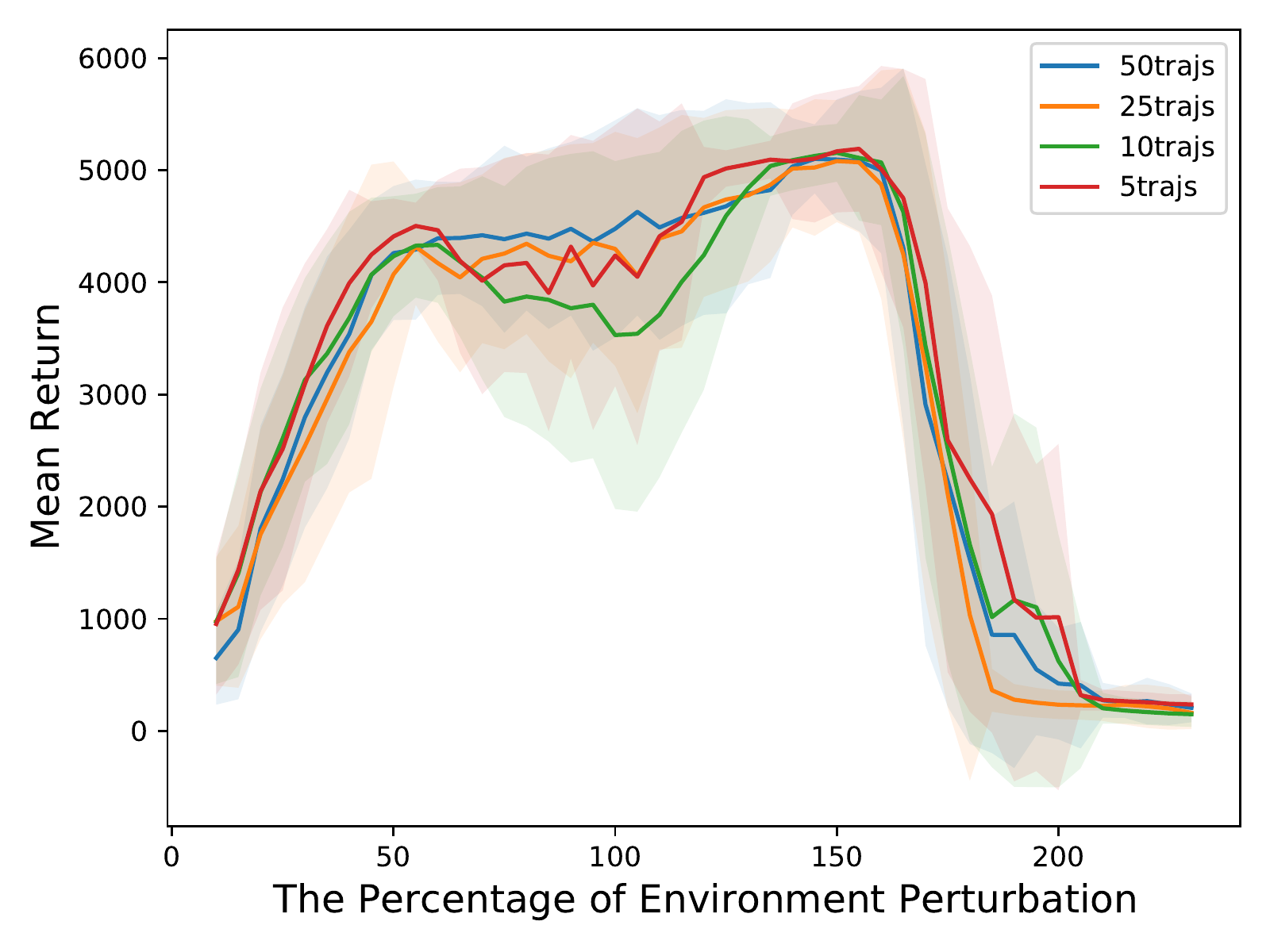}
        \captionsetup{justification=centering}
        \caption{Walker2d+Gravity}
    \end{subfigure}
    \begin{subfigure}[b]{0.24\textwidth}
        \centering
        \includegraphics[width=\textwidth]{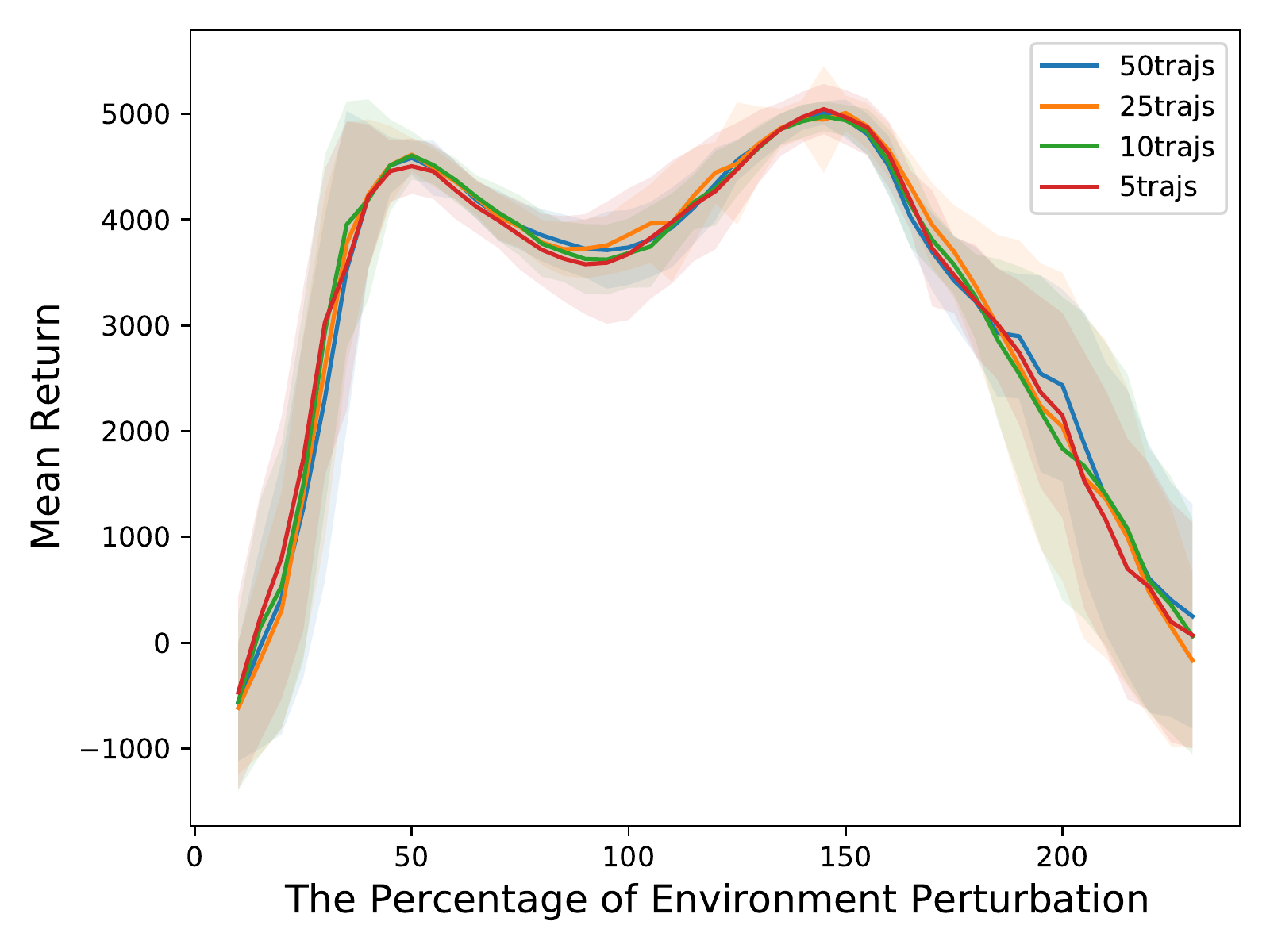}
        \captionsetup{justification=centering}
        \caption{HalfCheetah+Gravity}
    \end{subfigure}
    \begin{subfigure}[b]{0.24\textwidth}
        \centering
        \includegraphics[width=\textwidth]{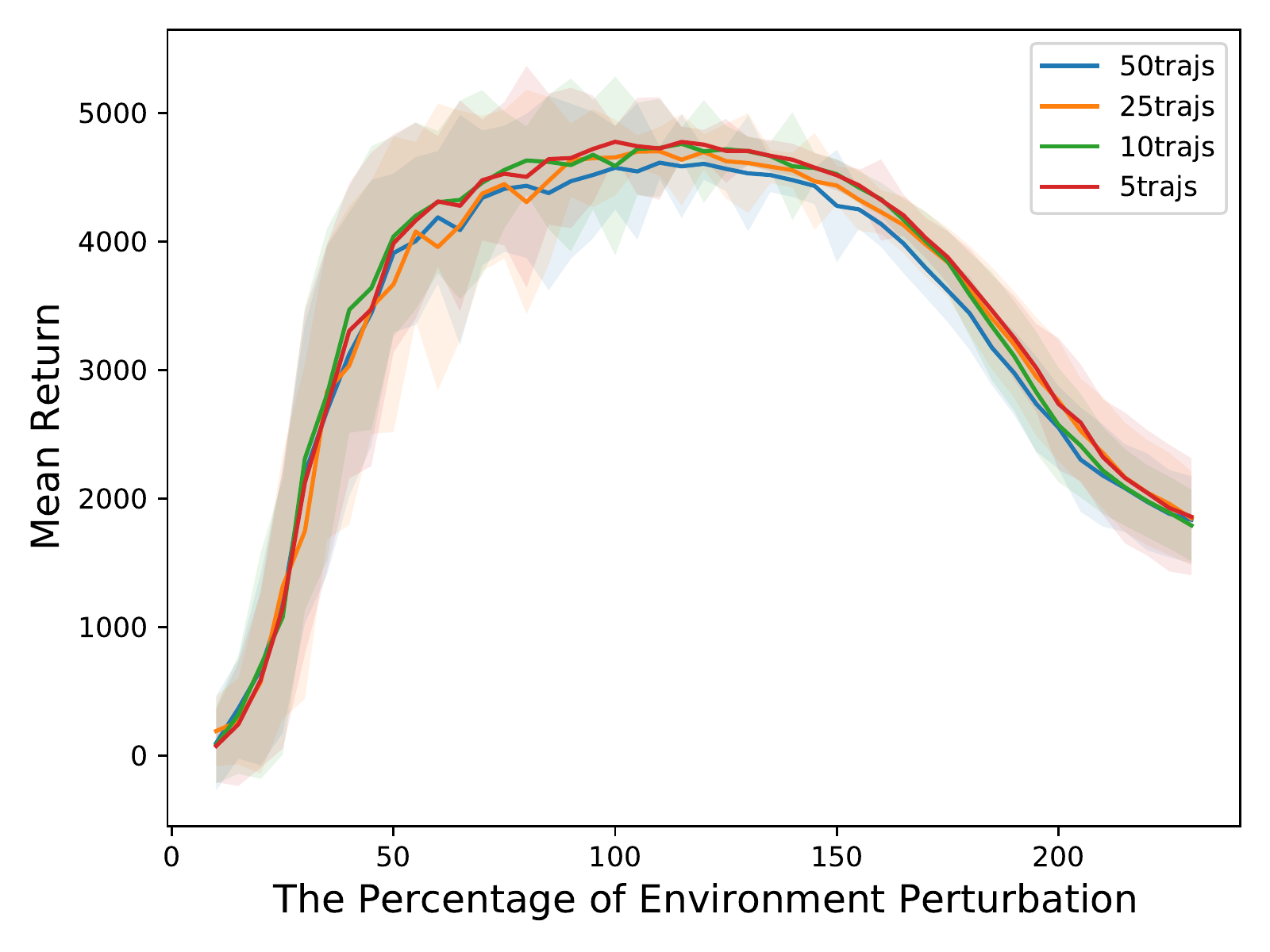}
        \captionsetup{justification=centering}
        \caption{Ant+Gravity}
    \end{subfigure}
    
    \begin{subfigure}[b]{0.24\textwidth}
        \centering
        \includegraphics[width=\textwidth]{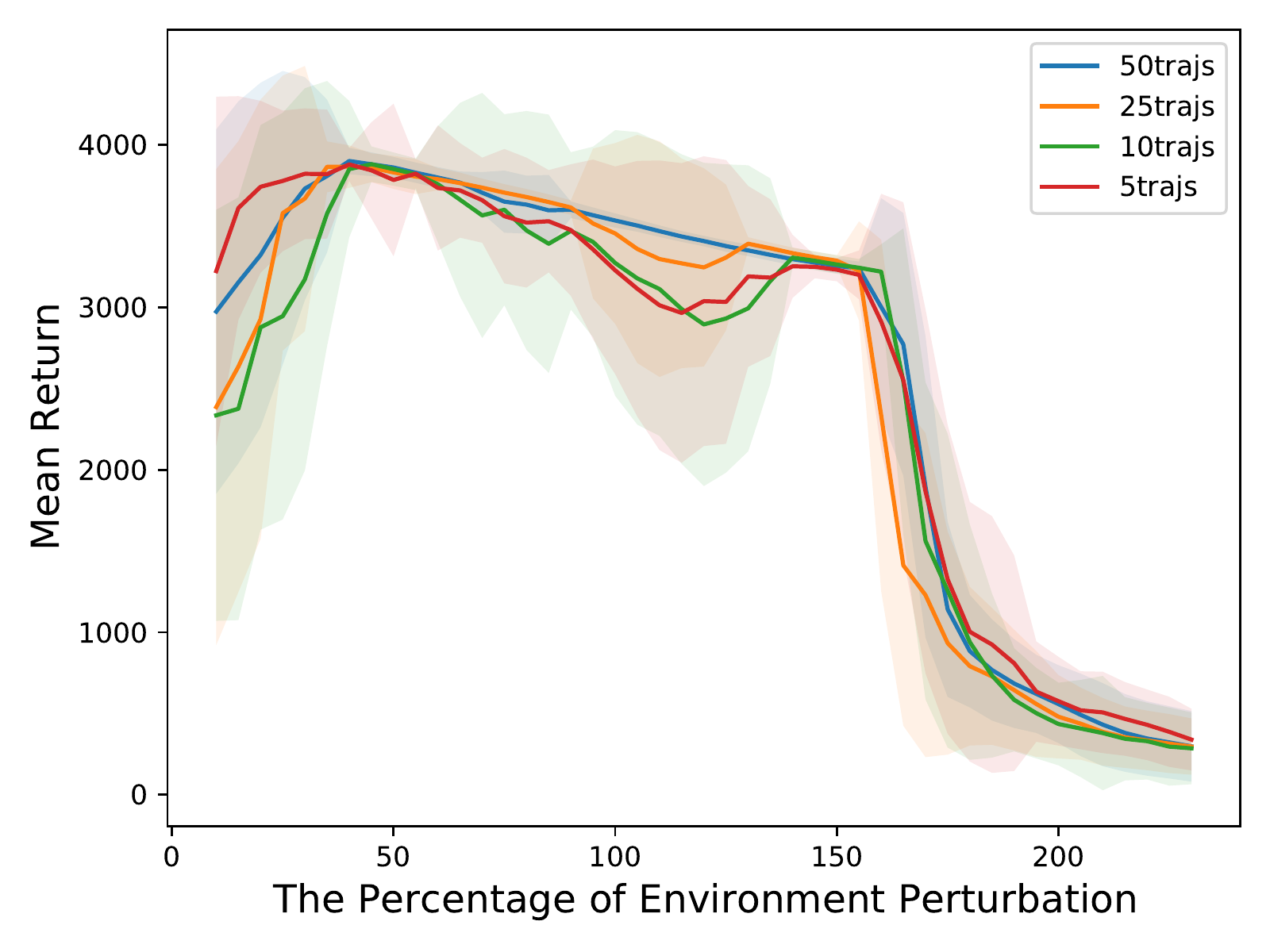}
        \captionsetup{justification=centering}
        \caption{Hopper+Mass}
    \end{subfigure}
    \begin{subfigure}[b]{0.24\textwidth}
        \label{figure:biased_envs_results_hog_compare}
        \centering
        \includegraphics[width=\textwidth]{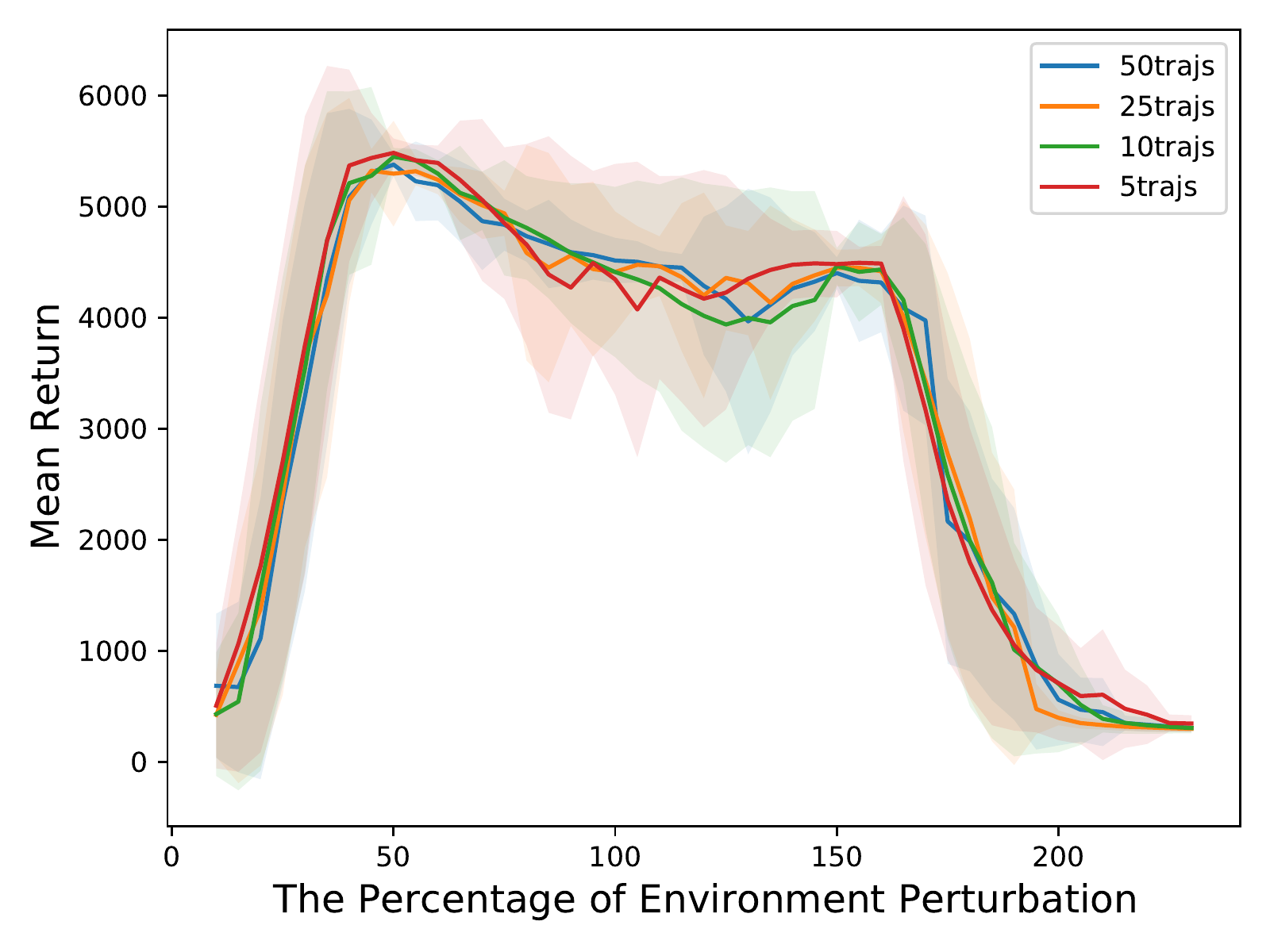}
        \captionsetup{justification=centering}
        \caption{Walker2d+Mass}
    \end{subfigure}
    \begin{subfigure}[b]{0.24\textwidth}
        \centering
        \includegraphics[width=\textwidth]{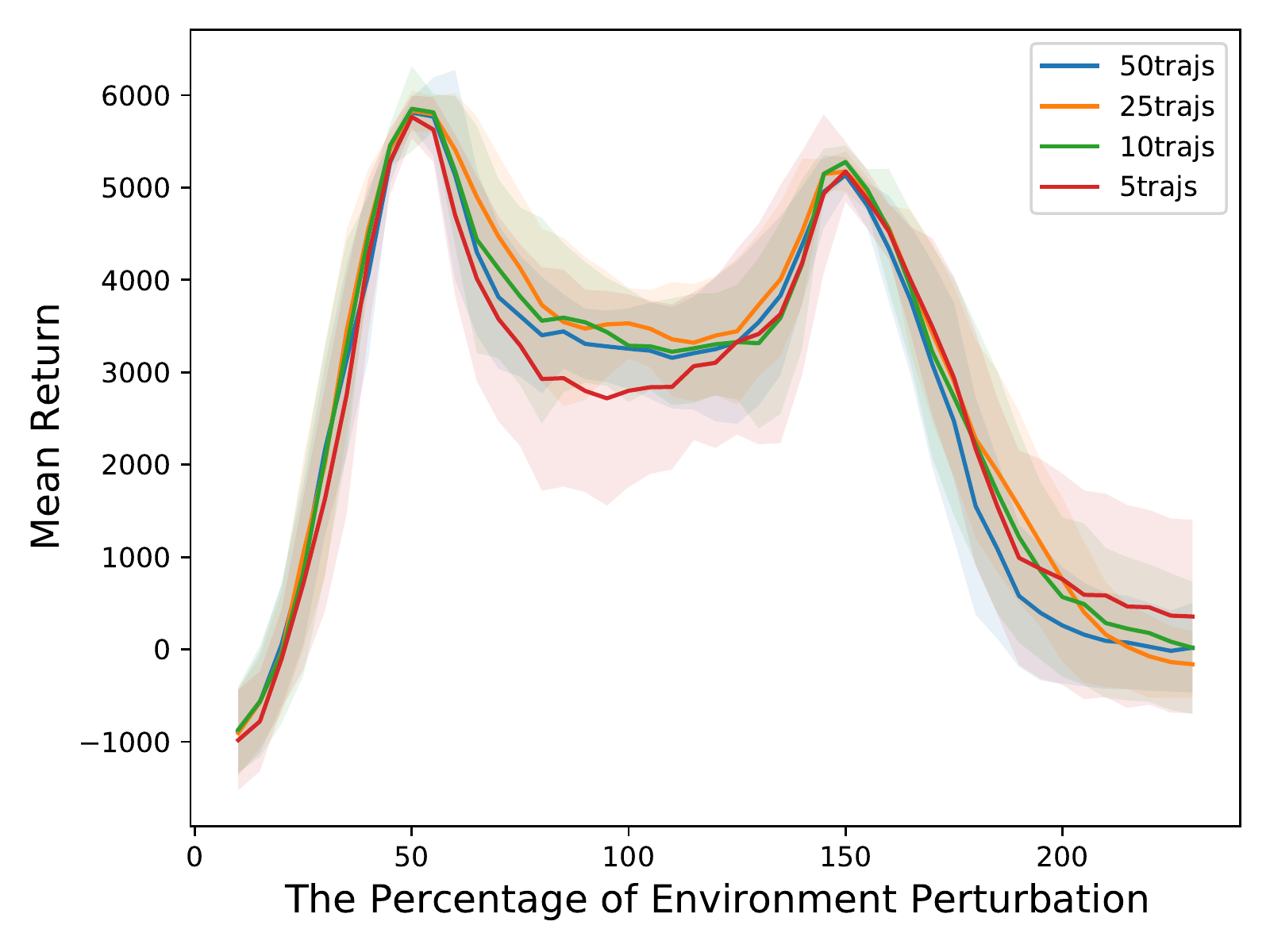}
        \captionsetup{justification=centering}
        \caption{HalfCheetah+Mass}
    \end{subfigure}
    \begin{subfigure}[b]{0.24\textwidth}
        \centering
        \includegraphics[width=\textwidth]{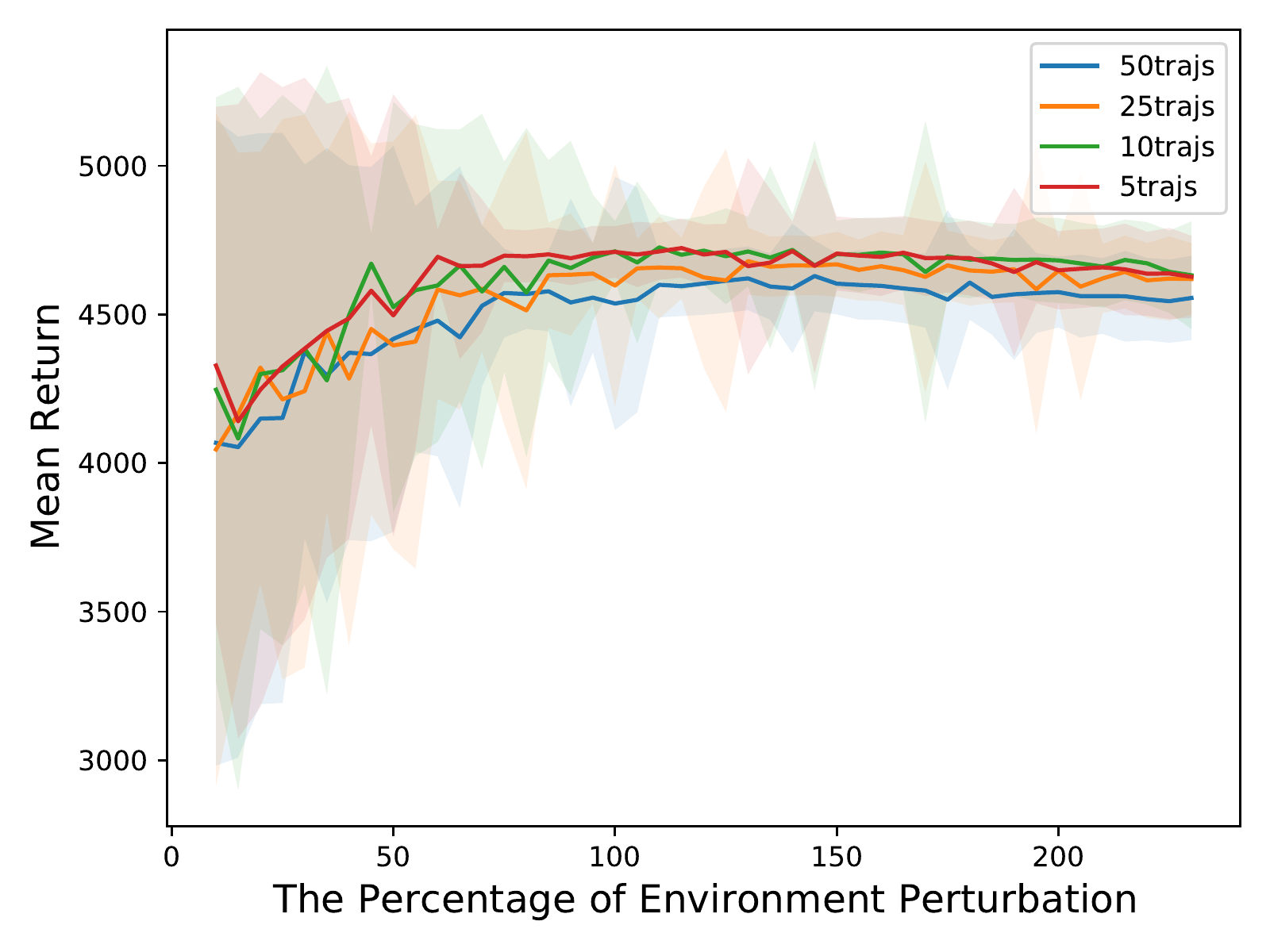}
        \captionsetup{justification=centering}
        \caption{Ant+Mass}
    \end{subfigure}
    \caption{The performance of RIME trained in the $N=2$ sampled environment setting with various sizes of expert demonstrations. \label{appendix:figure:results_over_trajectories}}
\end{figure}

\newpage

\section{Additional Experimental Results}
\label{Appendix:additional_experiments}

\subsection{Results in the $N=2$ Sampled Environment Setting ($50\%\zeta_0,~150\%\zeta_0$)}
\label{Appendix:results_in_2_learning_environmnets}

Here we provide all result plots in the 2 sampled environment setting for our algorithm and the baseline algorithms.

\begin{figure}[ht]
    \begin{subfigure}[b]{0.24\textwidth}
        \centering
        \includegraphics[width=\textwidth]{RIME_figures/2env_RIME_hog_perf.pdf}
        \captionsetup{justification=centering}
        \caption{Hopper+Gravity:\\performance}
    \end{subfigure}
    \begin{subfigure}[b]{0.24\textwidth}
        \centering
        \includegraphics[width=\textwidth]{RIME_figures/2env_compare_hog_compare.pdf}
        \captionsetup{justification=centering}
        \caption{Hopper+Gravity:\\comparisons}
    \end{subfigure}
    \begin{subfigure}[b]{0.24\textwidth}
        \centering
        \includegraphics[width=\textwidth]{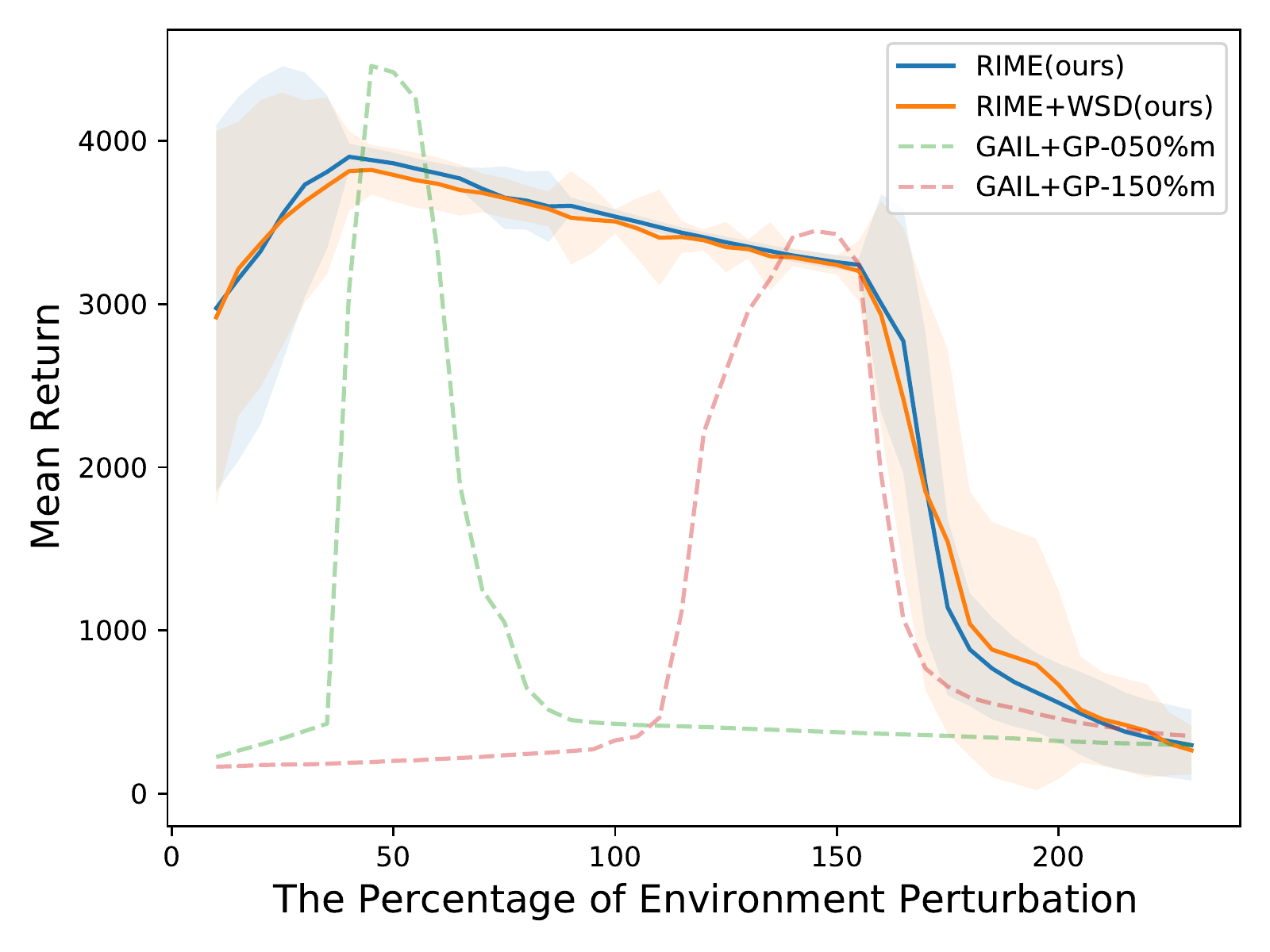}
        \captionsetup{justification=centering}
        \caption{Hopper+Mass:\\performance}
    \end{subfigure}
    \begin{subfigure}[b]{0.24\textwidth}
        \centering
        \includegraphics[width=\textwidth]{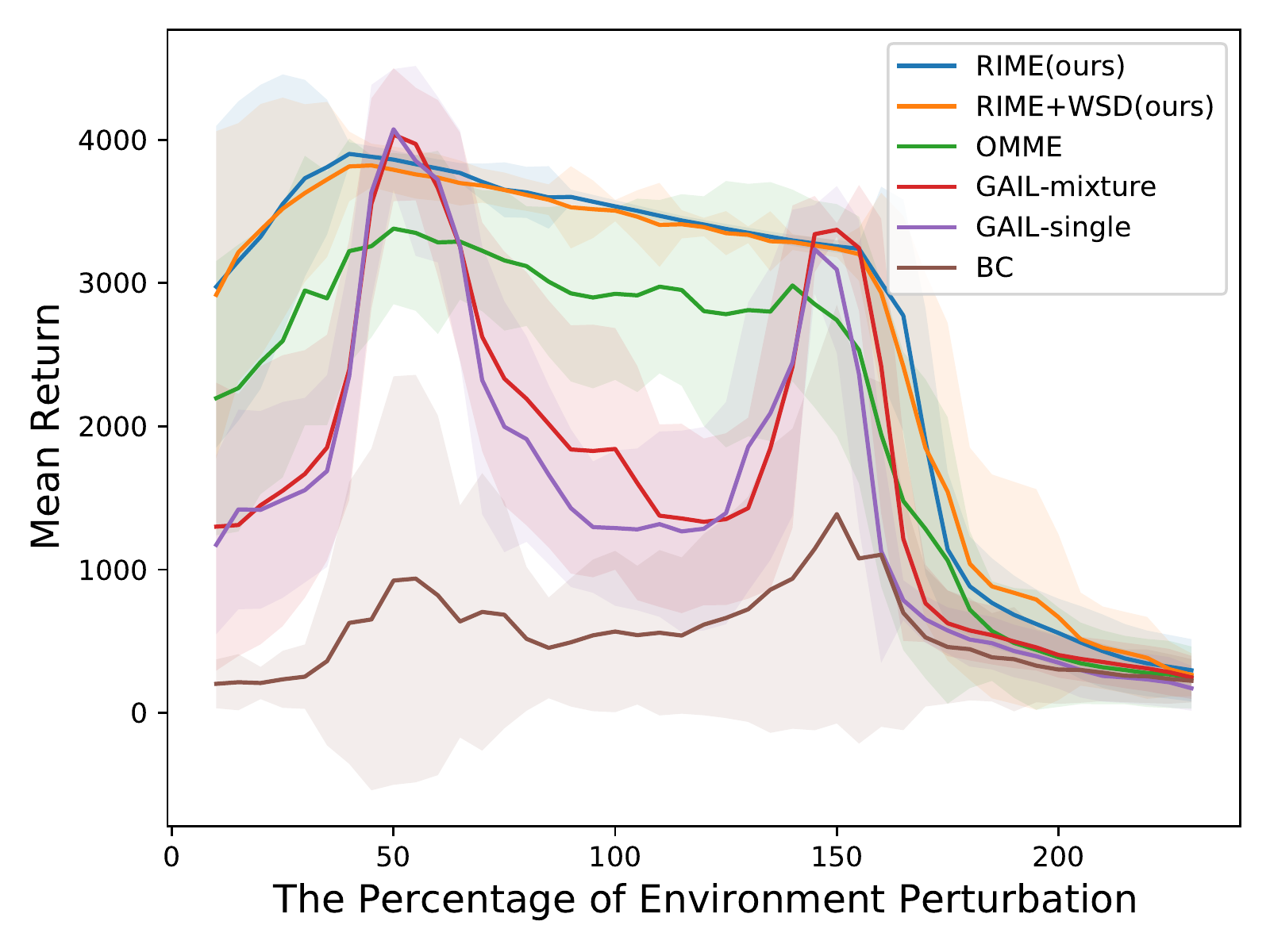}
        \captionsetup{justification=centering}
        \caption{Hopper+Mass:\\comparisons}
    \end{subfigure}
    
    \begin{subfigure}[b]{0.24\textwidth}
        \centering
        \includegraphics[width=\textwidth]{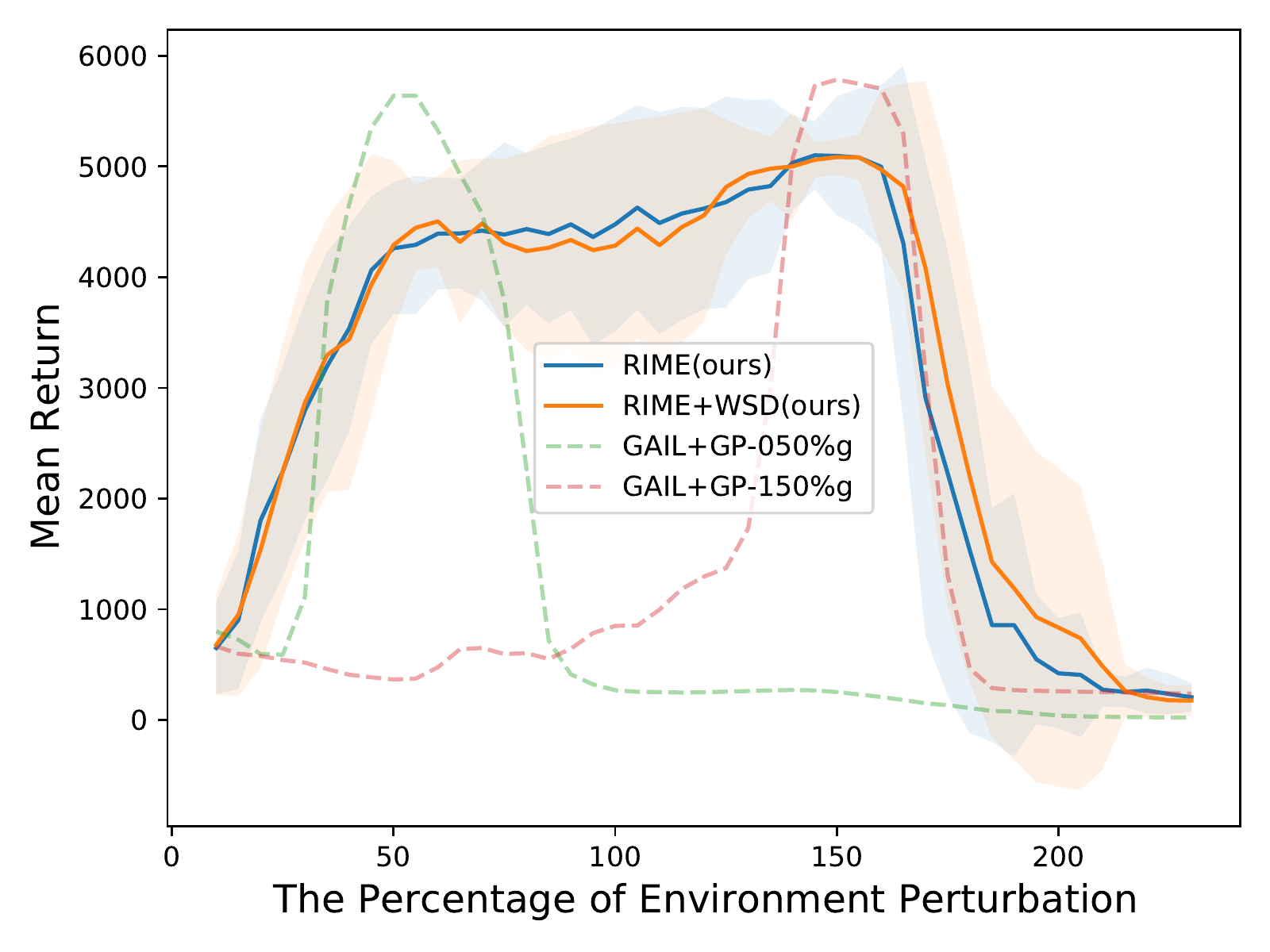}
        \captionsetup{justification=centering}
        \caption{Walker2d+Gravity:\\performance}
    \end{subfigure}
    \begin{subfigure}[b]{0.24\textwidth}
        \centering
        \includegraphics[width=\textwidth]{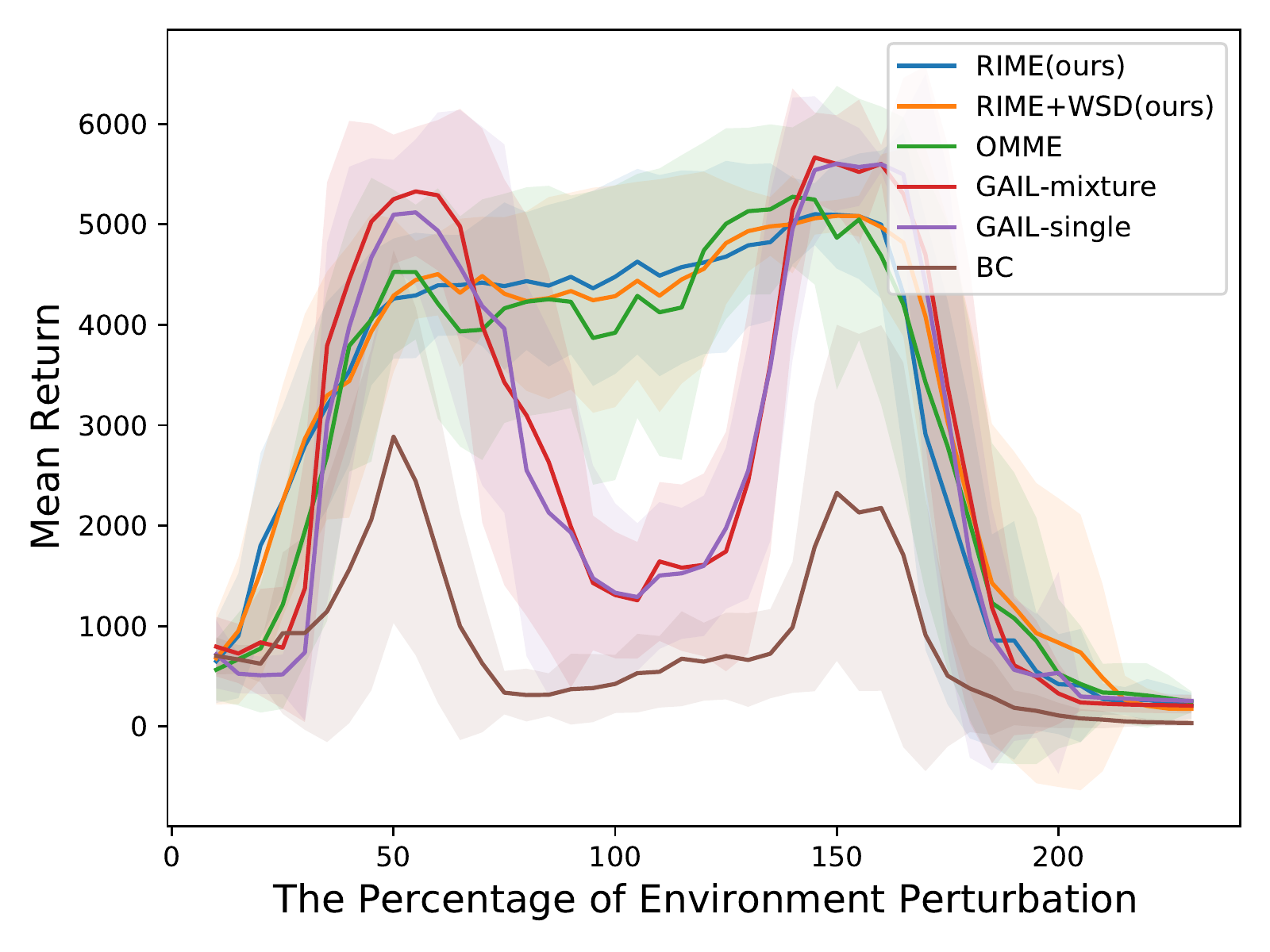}
        \captionsetup{justification=centering}
        \caption{Walker2d+Gravity:\\comparisons}
    \end{subfigure}
    \begin{subfigure}[b]{0.24\textwidth}
        \centering
        \includegraphics[width=\textwidth]{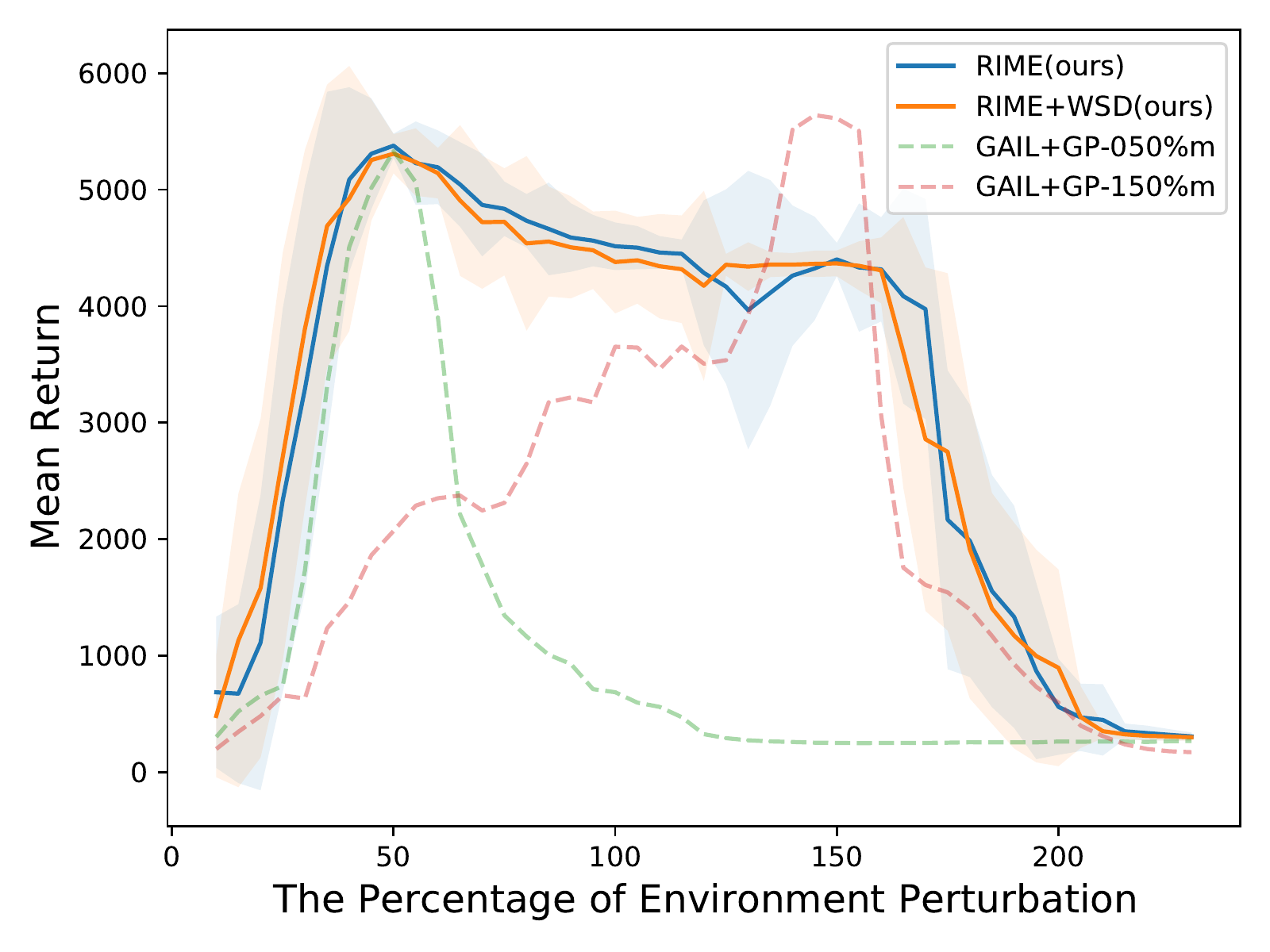}
        \captionsetup{justification=centering}
        \caption{Walker2d+Mass:\\performance}
    \end{subfigure}
    \begin{subfigure}[b]{0.24\textwidth}
        \centering
        \includegraphics[width=\textwidth]{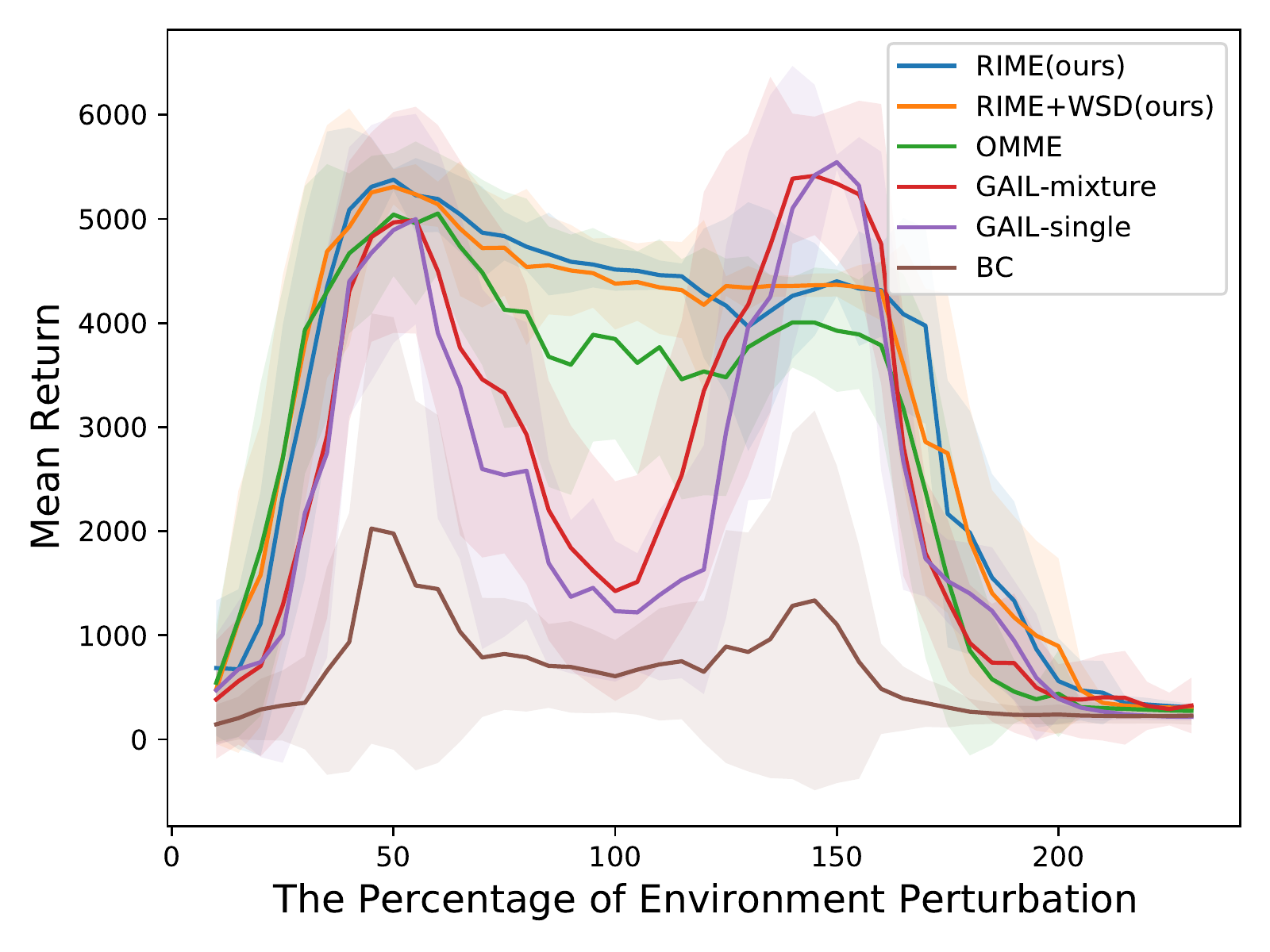}
        \captionsetup{justification=centering}
        \caption{Walker2d+Mass:\\comparisons}
    \end{subfigure}
    
    \begin{subfigure}[b]{0.24\textwidth}
        \centering
        \includegraphics[width=\textwidth]{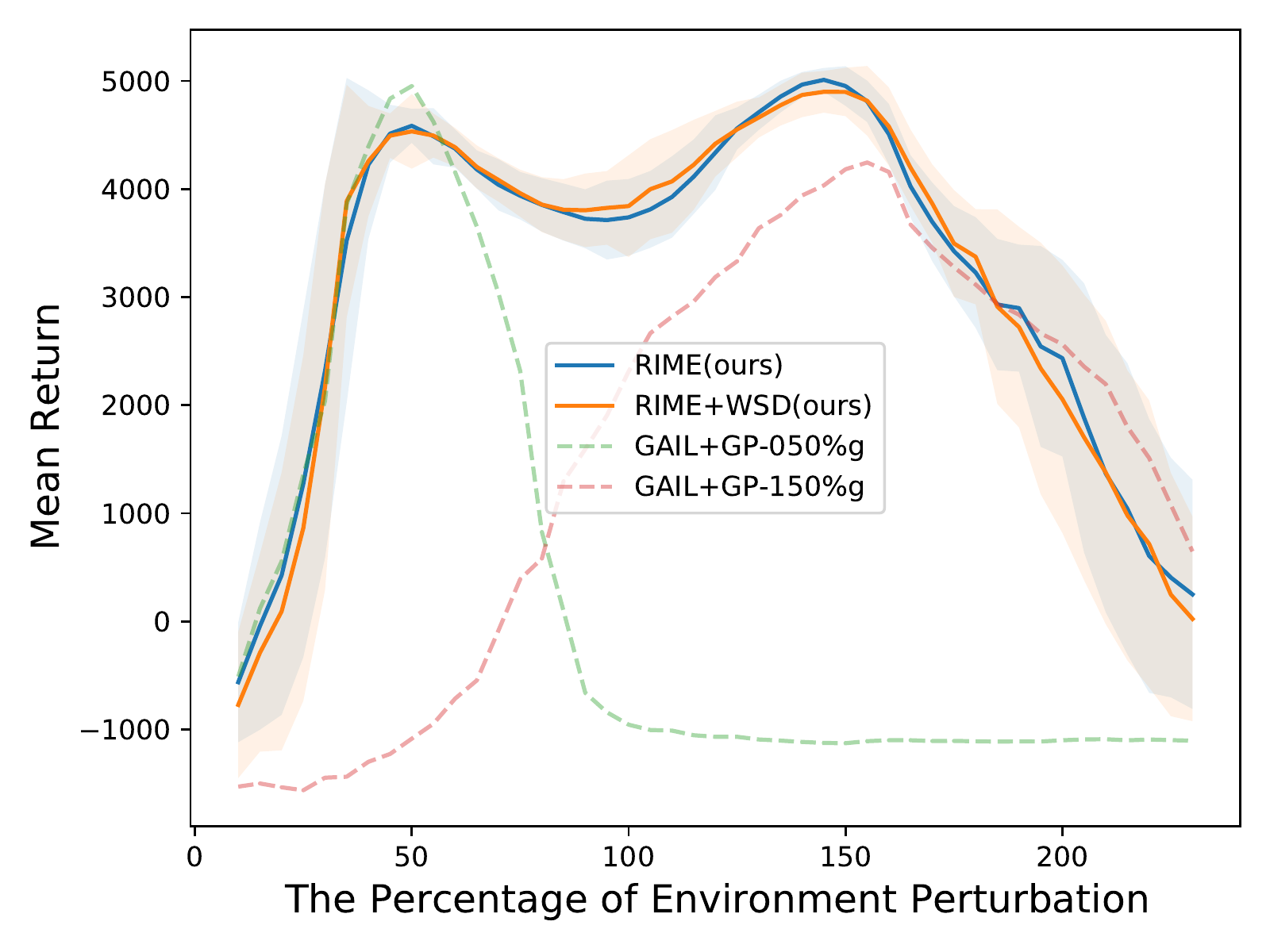}
        \captionsetup{justification=centering}
        \caption{HalfCheetah+Gravity:\\performance}
    \end{subfigure}
    \begin{subfigure}[b]{0.24\textwidth}
        \centering
        \includegraphics[width=\textwidth]{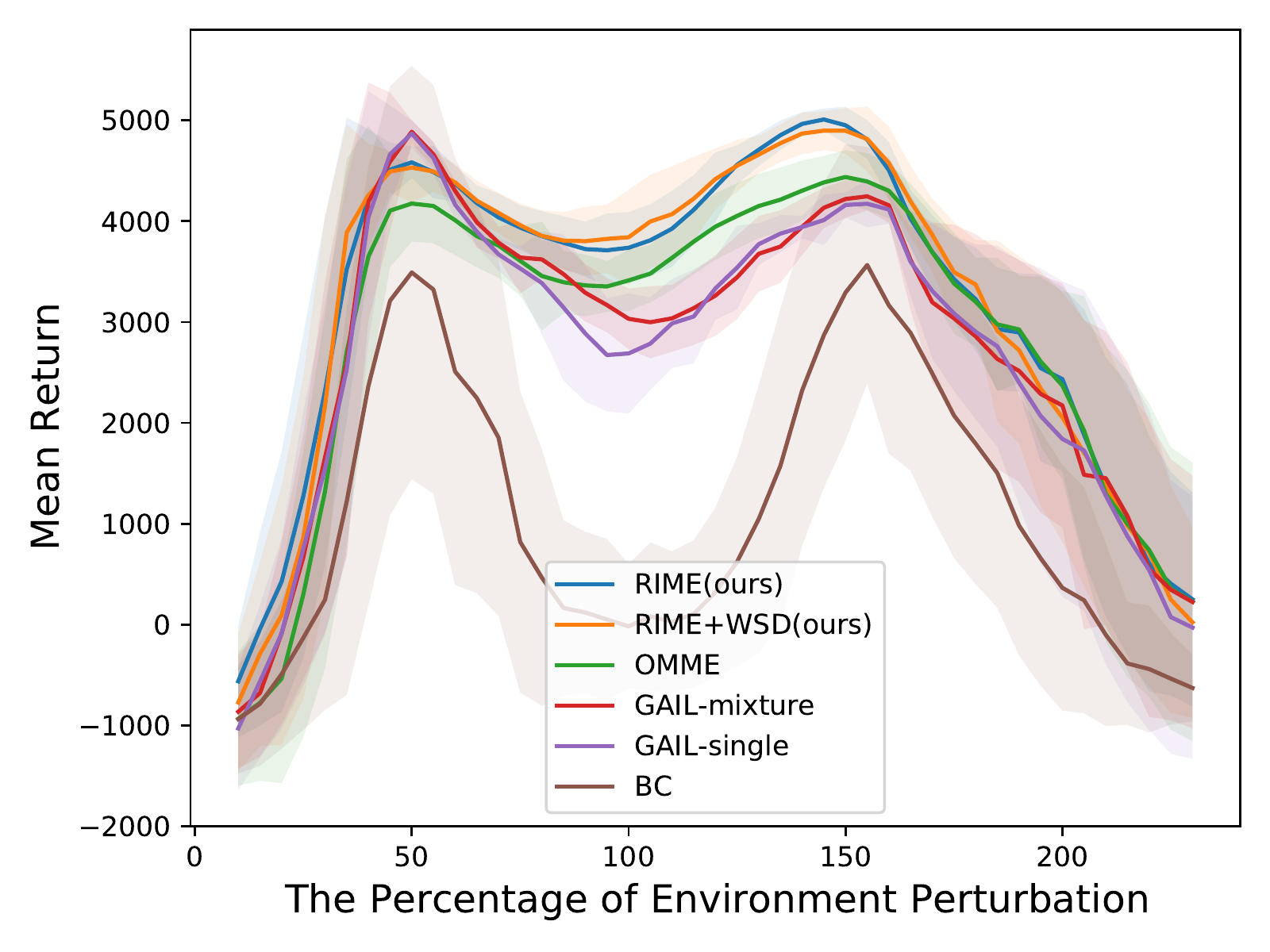}
        \captionsetup{justification=centering}
        \caption{HalfCheetah+Gravity:\\comparisons}
    \end{subfigure}
    \begin{subfigure}[b]{0.24\textwidth}
        \centering
        \includegraphics[width=\textwidth]{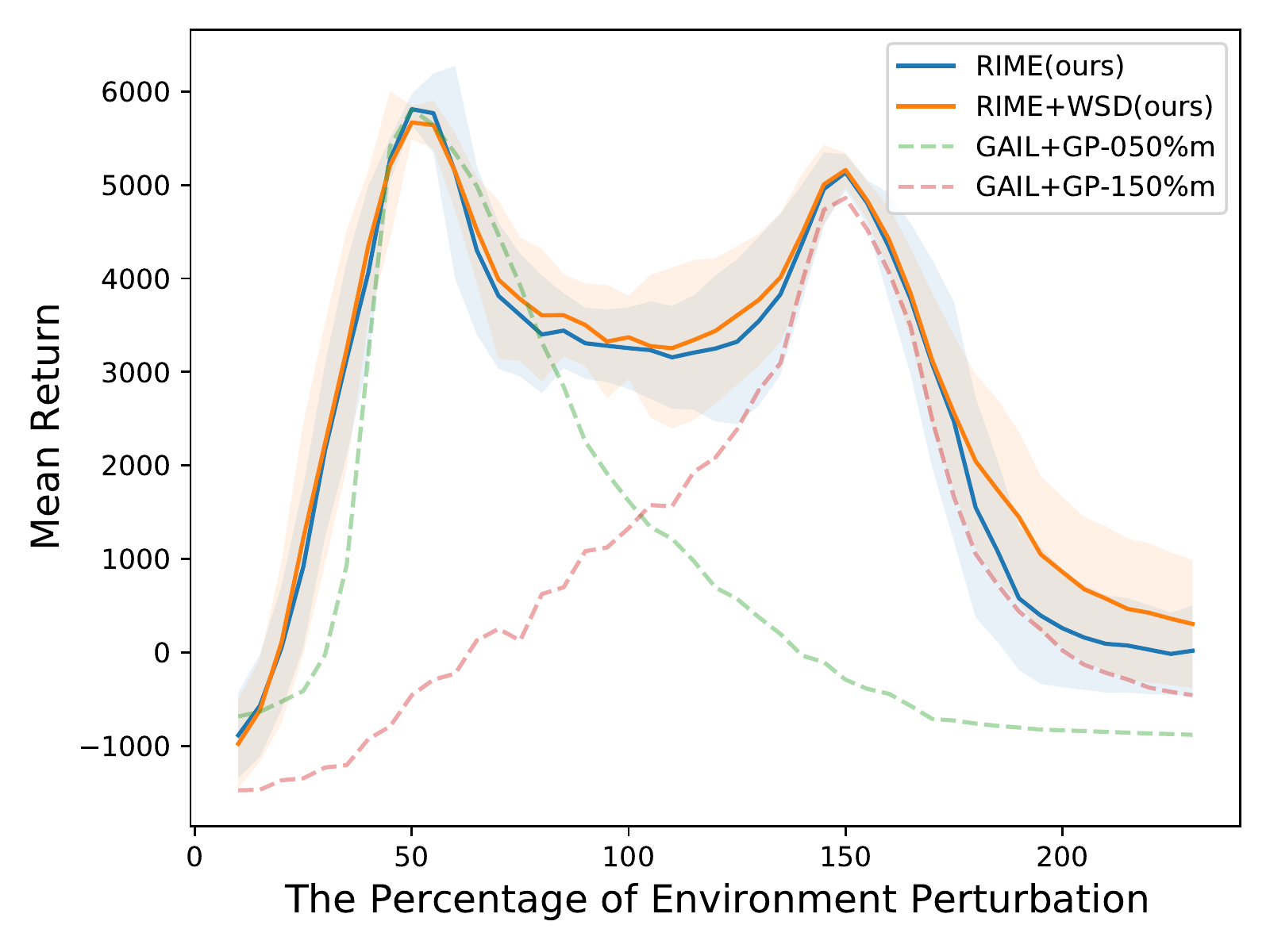}
        \captionsetup{justification=centering}
        \caption{HalfCheetah+Mass:\\performance}
    \end{subfigure}
    \begin{subfigure}[b]{0.24\textwidth}
        \centering
        \includegraphics[width=\textwidth]{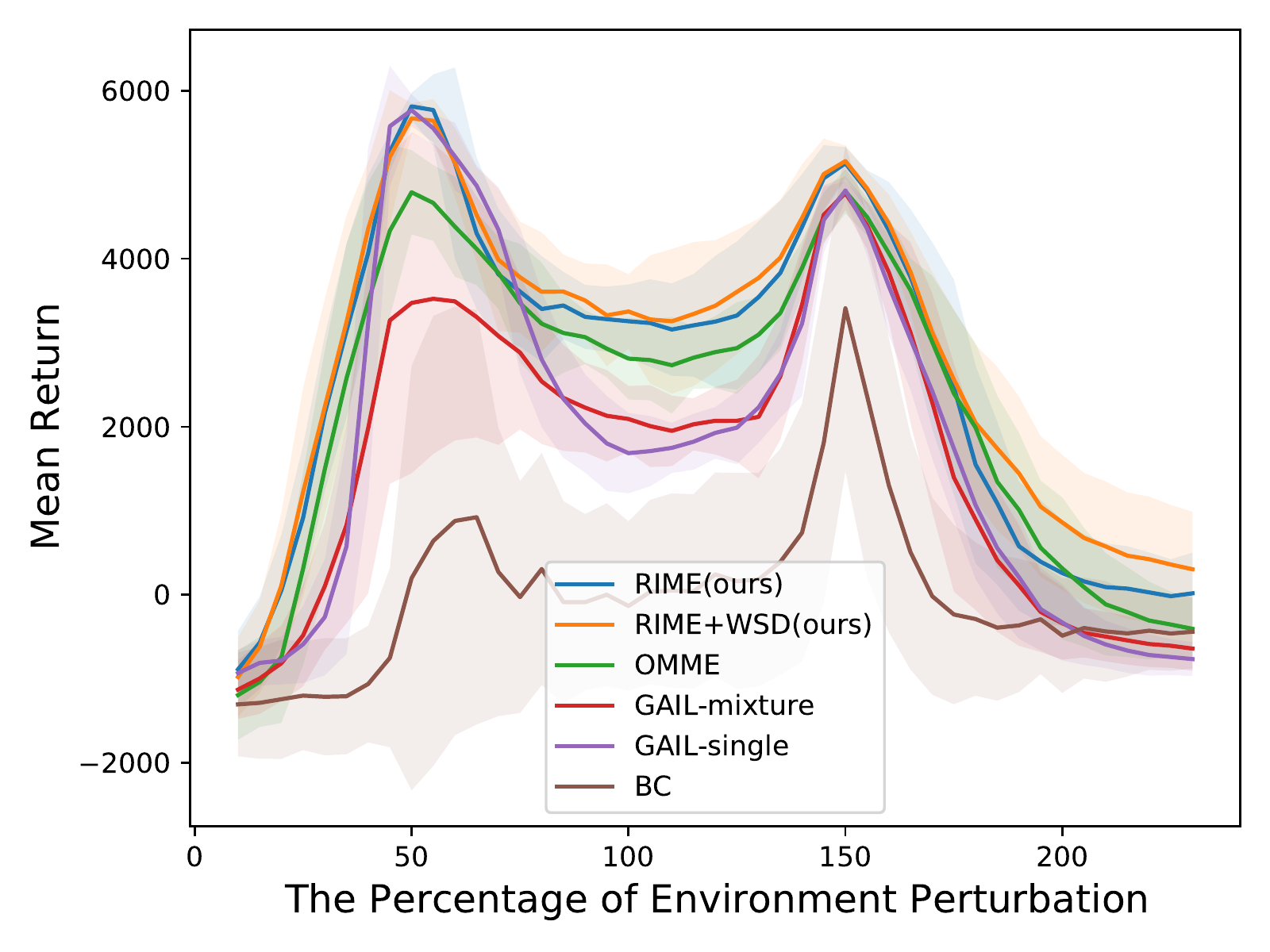}
        \captionsetup{justification=centering}
        \caption{HalfCheetah+Mass:\\comparisons}
    \end{subfigure}

    \begin{subfigure}[b]{0.24\textwidth}
        \centering
        \includegraphics[width=\textwidth]{RIME_figures/2env_RIME_ag_perf.pdf}
        \captionsetup{justification=centering}
        \caption{Ant+Gravity:\\performance}
    \end{subfigure}
    \begin{subfigure}[b]{0.24\textwidth}
        \centering
        \includegraphics[width=\textwidth]{RIME_figures/2env_compare_ag_compare.pdf}
        \captionsetup{justification=centering}
        \caption{Ant+Gravity:\\comparisons}
    \end{subfigure}
    \begin{subfigure}[b]{0.24\textwidth}
        \centering
        \includegraphics[width=\textwidth]{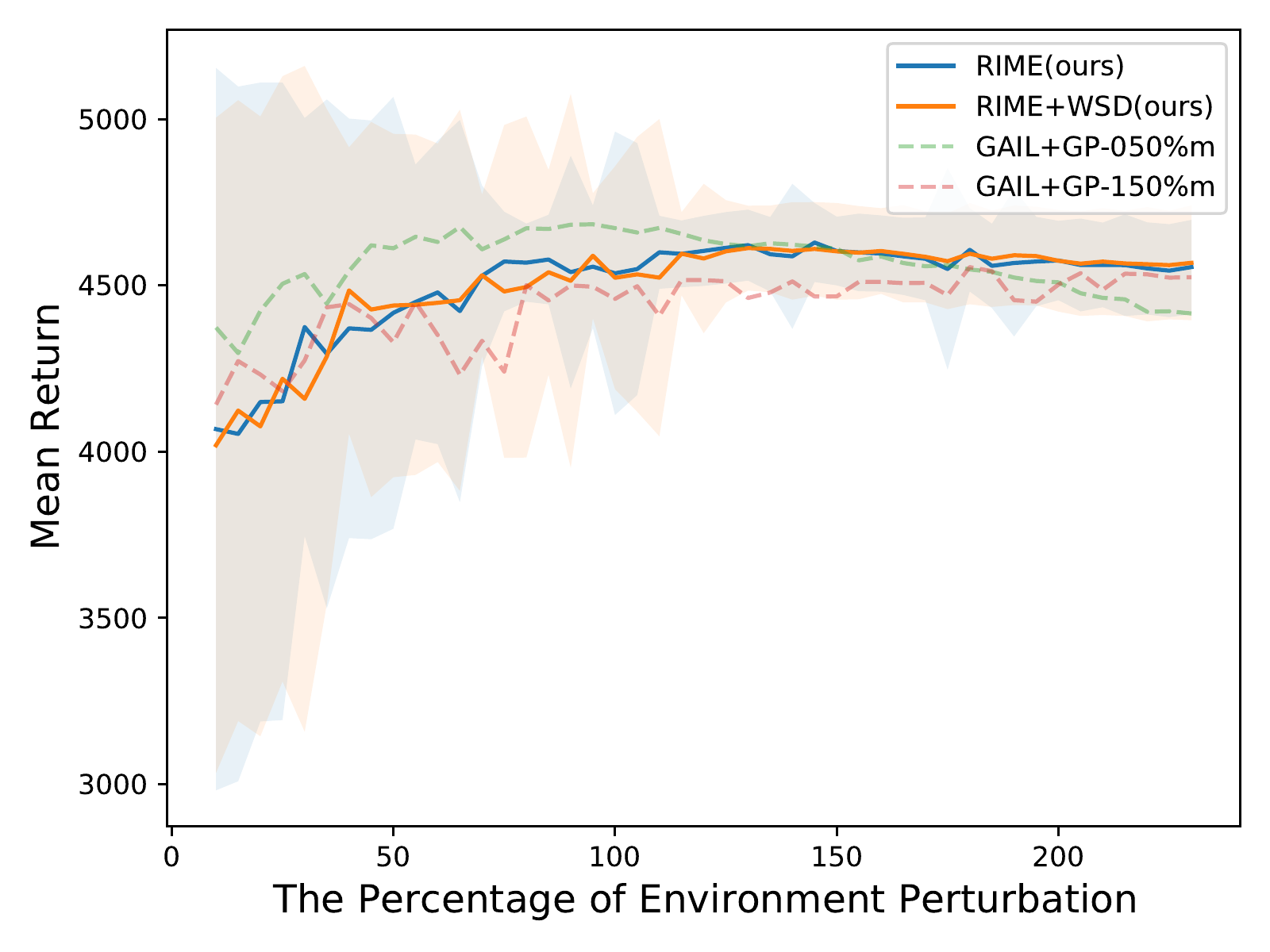}
        \captionsetup{justification=centering}
        \caption{Ant+Mass:\\performance}
    \end{subfigure}
    \begin{subfigure}[b]{0.24\textwidth}
        \centering
        \includegraphics[width=\textwidth]{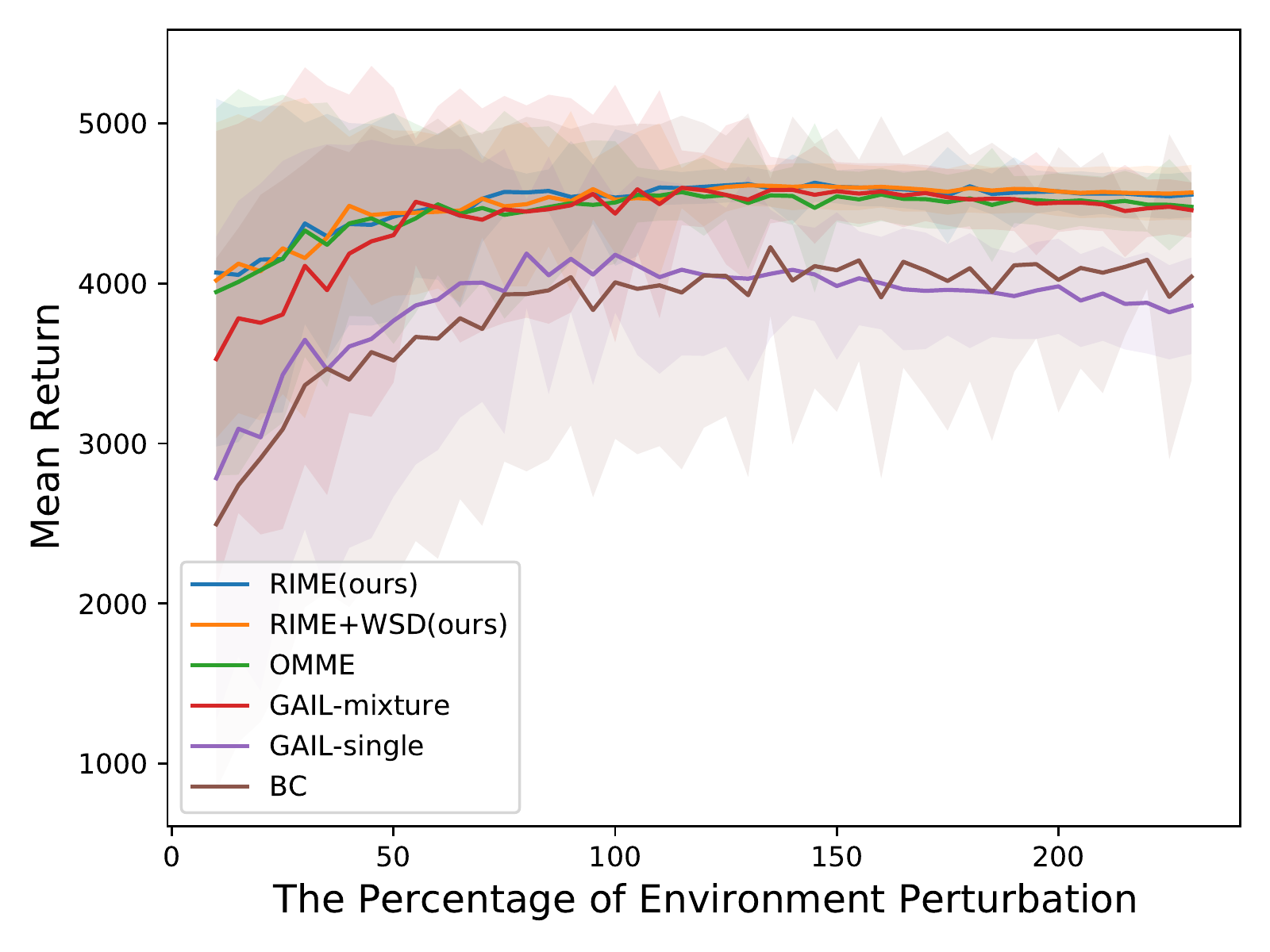}
        \captionsetup{justification=centering}
        \caption{Ant+Mass:\\comparisons}
    \end{subfigure}
    \caption{All experimental results in the $N=2$ sampled environment setting ($50\%\zeta_0, ~150\%\zeta_0$).\label{appendix:figures:2learningenvironments}}
\end{figure}

\newpage

\subsection{Results in the $N=3$ Sampled Environment Setting ($050\%\zeta_0,~100\%\zeta_0,~150\%\zeta_0$)}
\label{Appendix:results_in_3_learning_environmnets}

Here we provide all result plots in the 3 sampled environment setting for our algorithm and the baseline algorithms.

\begin{figure}[ht]
    \begin{subfigure}[b]{0.24\textwidth}
        \centering
        \includegraphics[width=\textwidth]{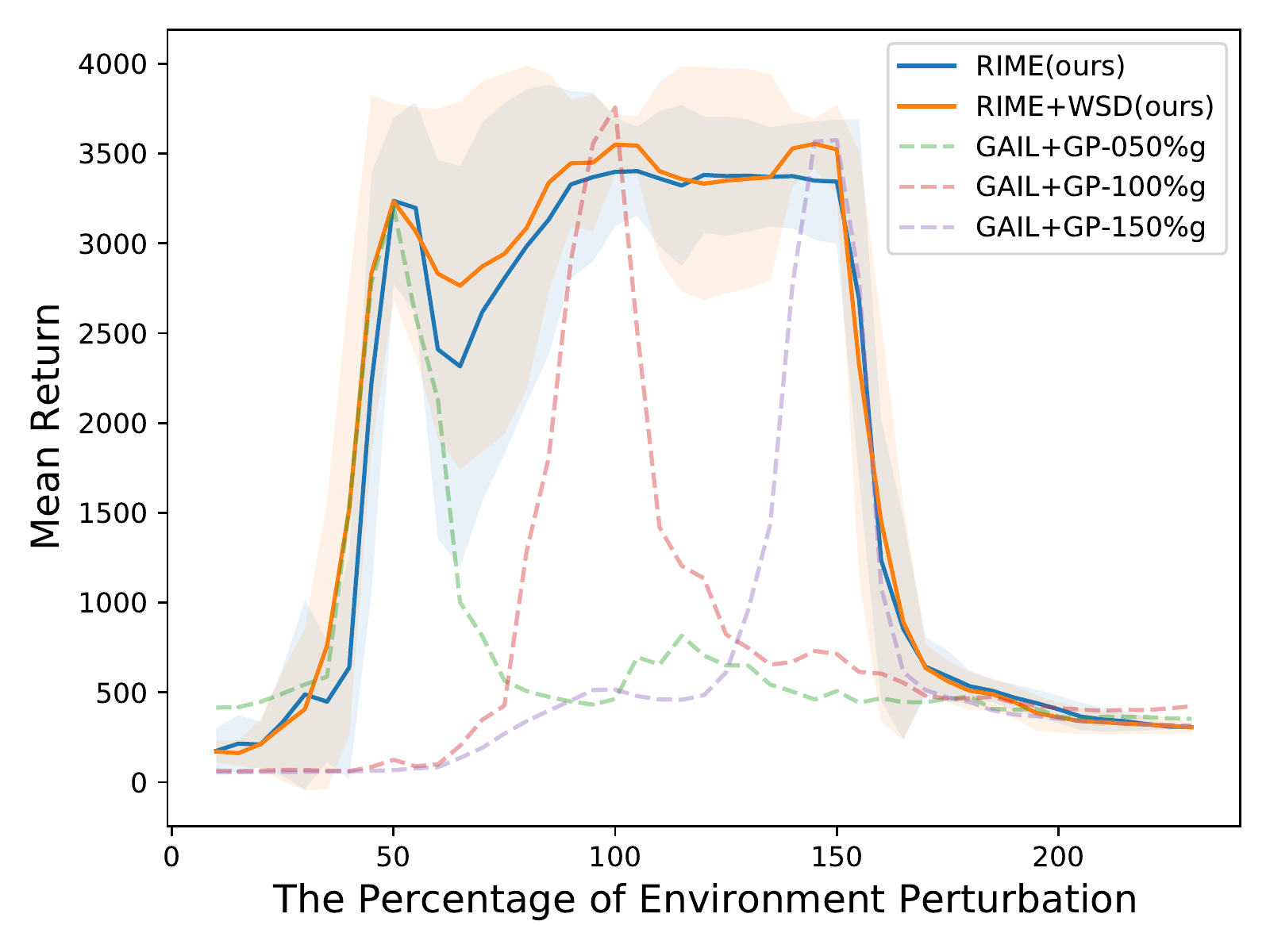}
        \captionsetup{justification=centering}
        \caption{Hopper+Gravity:\\performance}
    \end{subfigure}
    \begin{subfigure}[b]{0.24\textwidth}
        \centering
        \includegraphics[width=\textwidth]{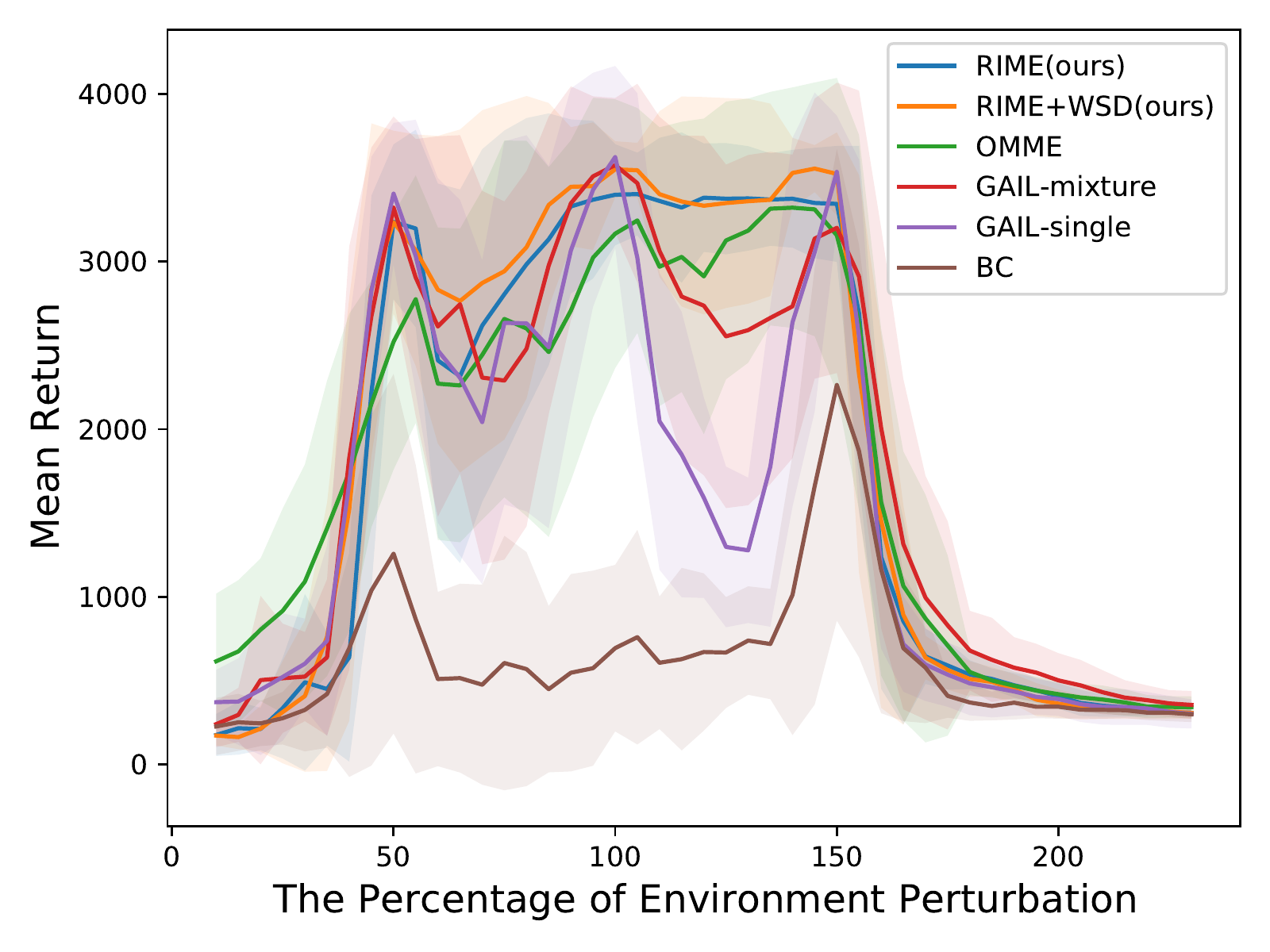}
        \captionsetup{justification=centering}
        \caption{Hopper+Gravity:\\comparisons}
    \end{subfigure}
    \begin{subfigure}[b]{0.24\textwidth}
        \centering
        \includegraphics[width=\textwidth]{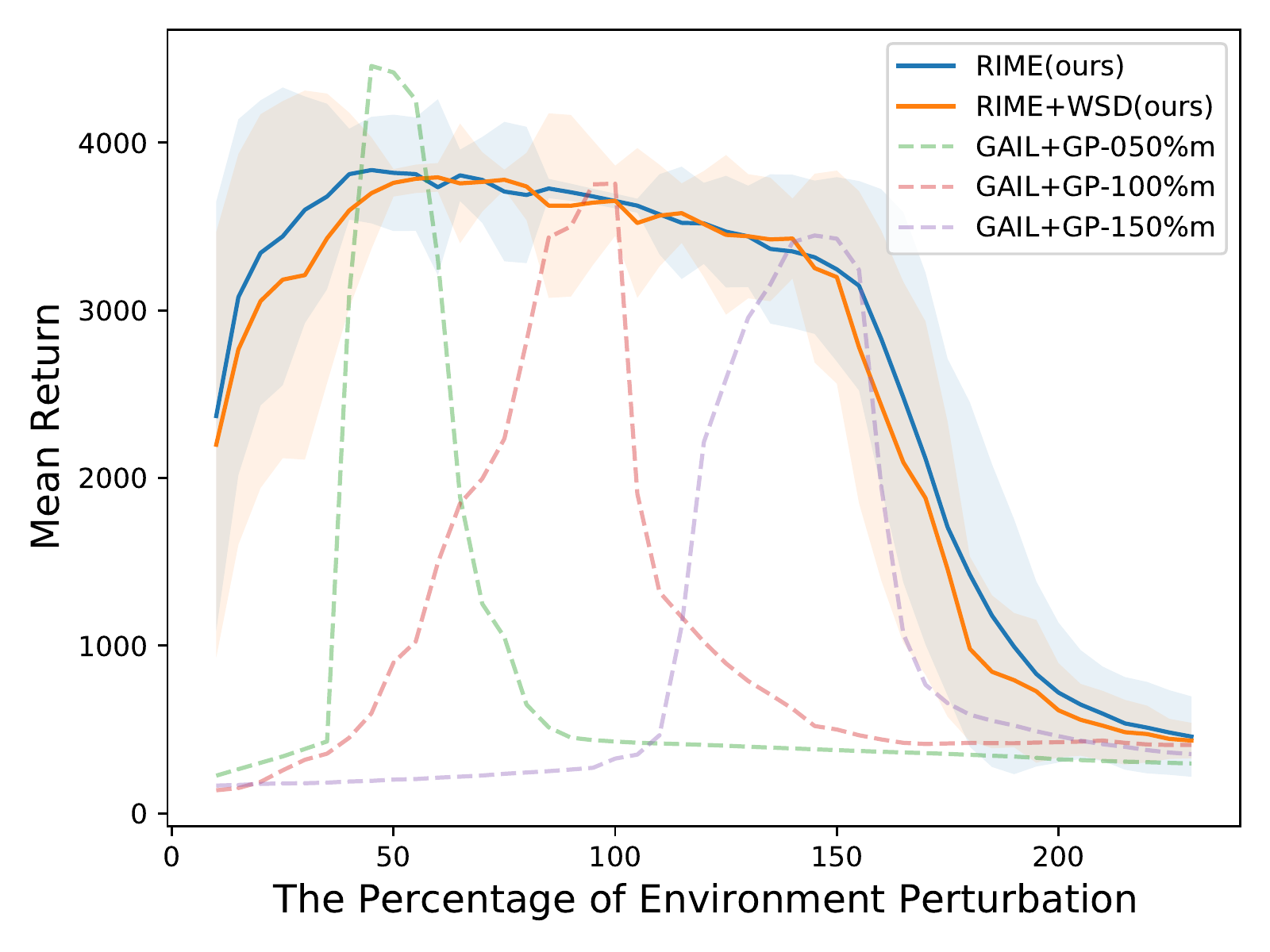}
        \captionsetup{justification=centering}
        \caption{Hopper+Mass:\\performance}
    \end{subfigure}
    \begin{subfigure}[b]{0.24\textwidth}
        \centering
        \includegraphics[width=\textwidth]{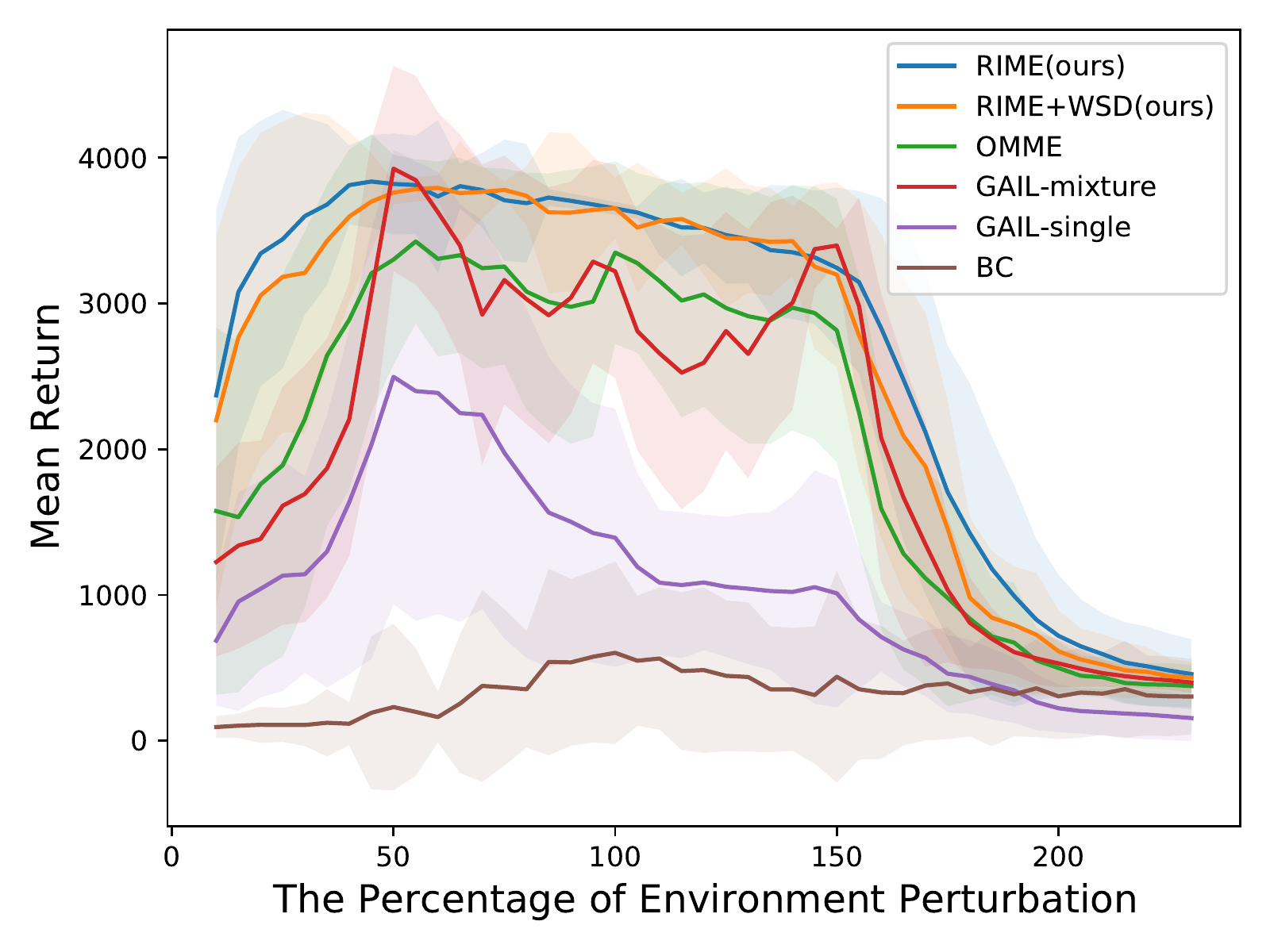}
        \captionsetup{justification=centering}
        \caption{Hopper+Mass:\\comparisons}
    \end{subfigure}
    
    \begin{subfigure}[b]{0.24\textwidth}
        \centering
        \includegraphics[width=\textwidth]{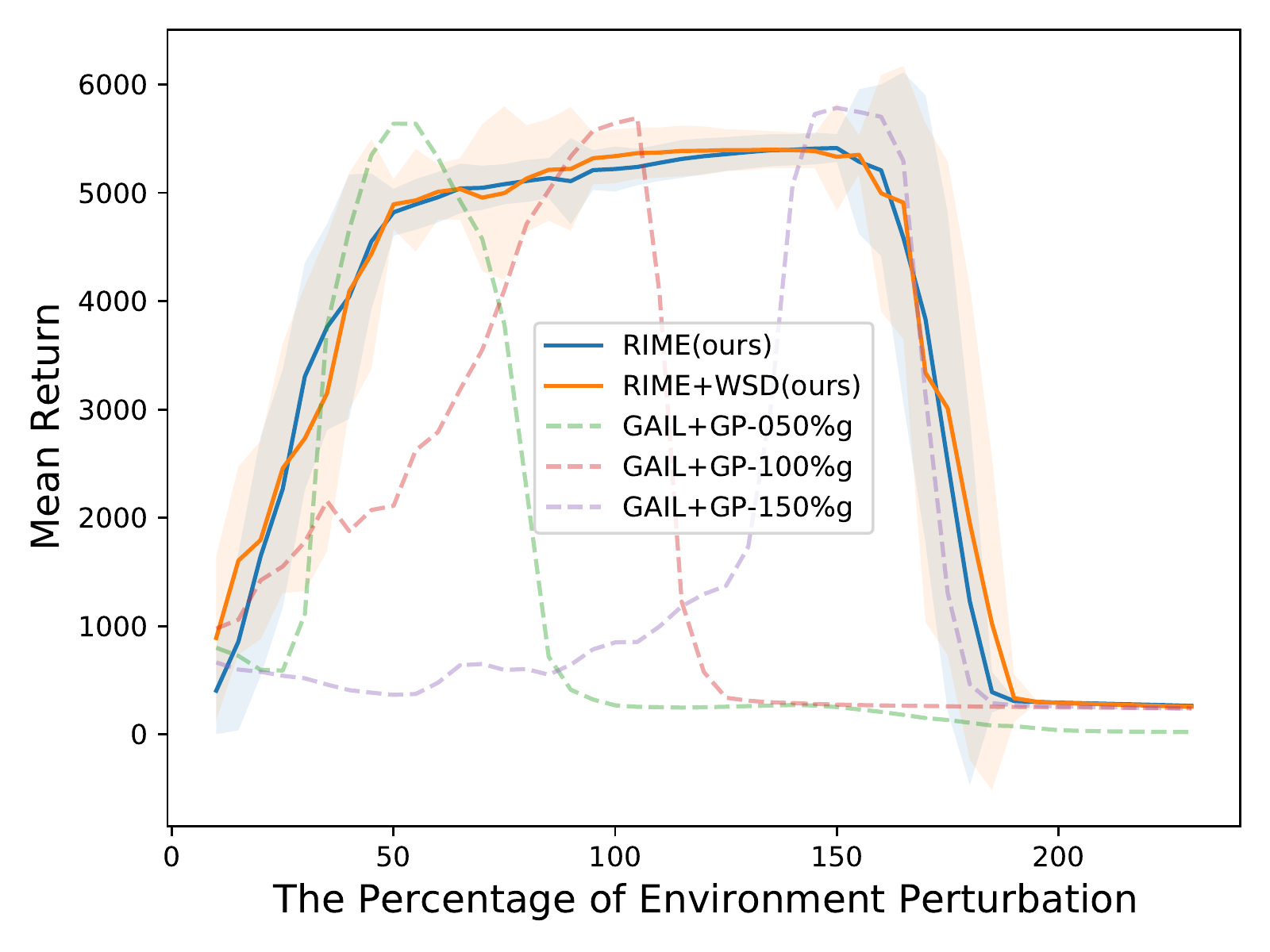}
        \captionsetup{justification=centering}
        \caption{Walker2d+Gravity:\\performance}
    \end{subfigure}
    \begin{subfigure}[b]{0.24\textwidth}
        \centering
        \includegraphics[width=\textwidth]{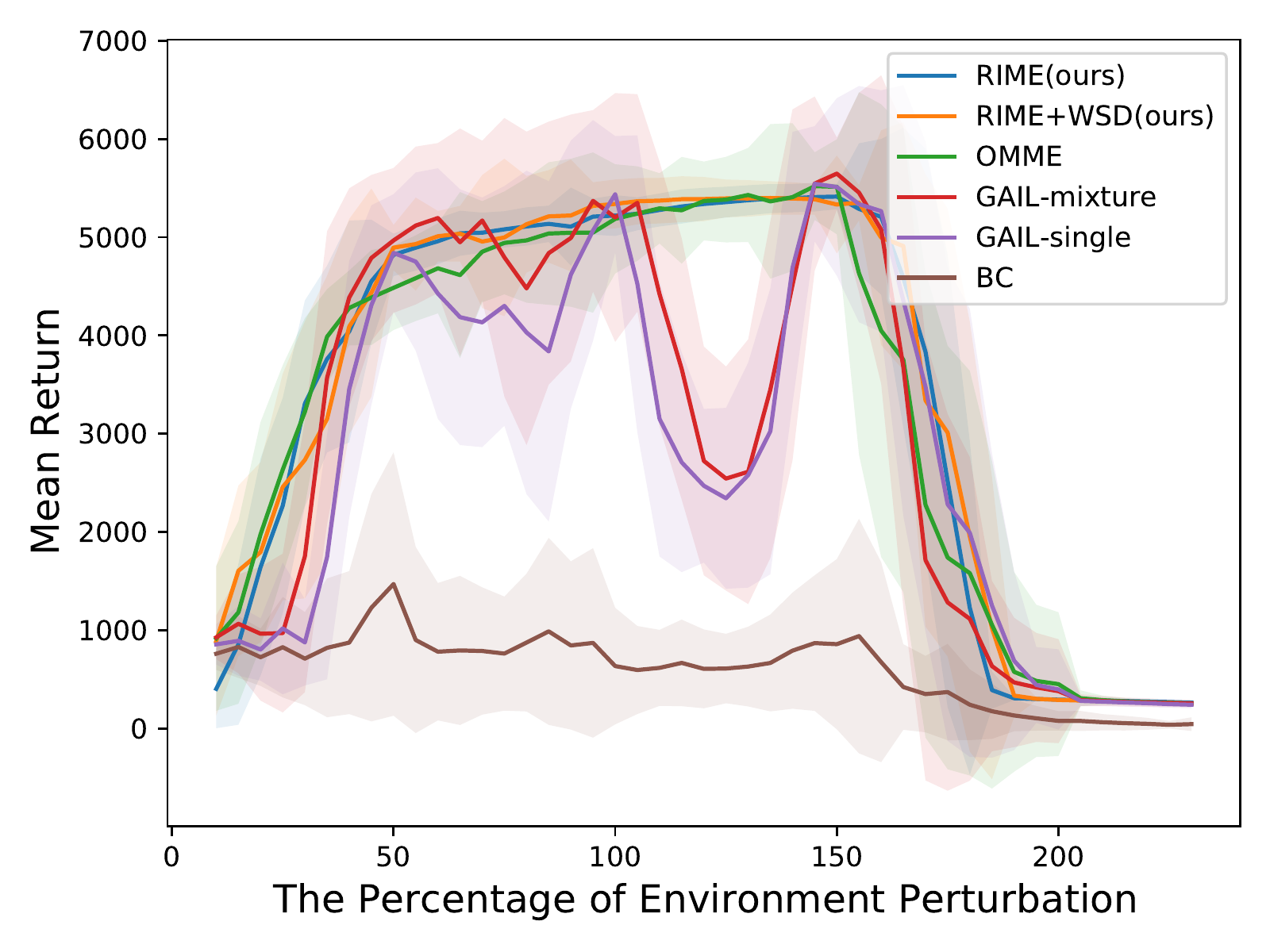}
        \captionsetup{justification=centering}
        \caption{Walker2d+Gravity:\\comparisons}
    \end{subfigure}
    \begin{subfigure}[b]{0.24\textwidth}
        \centering
        \includegraphics[width=\textwidth]{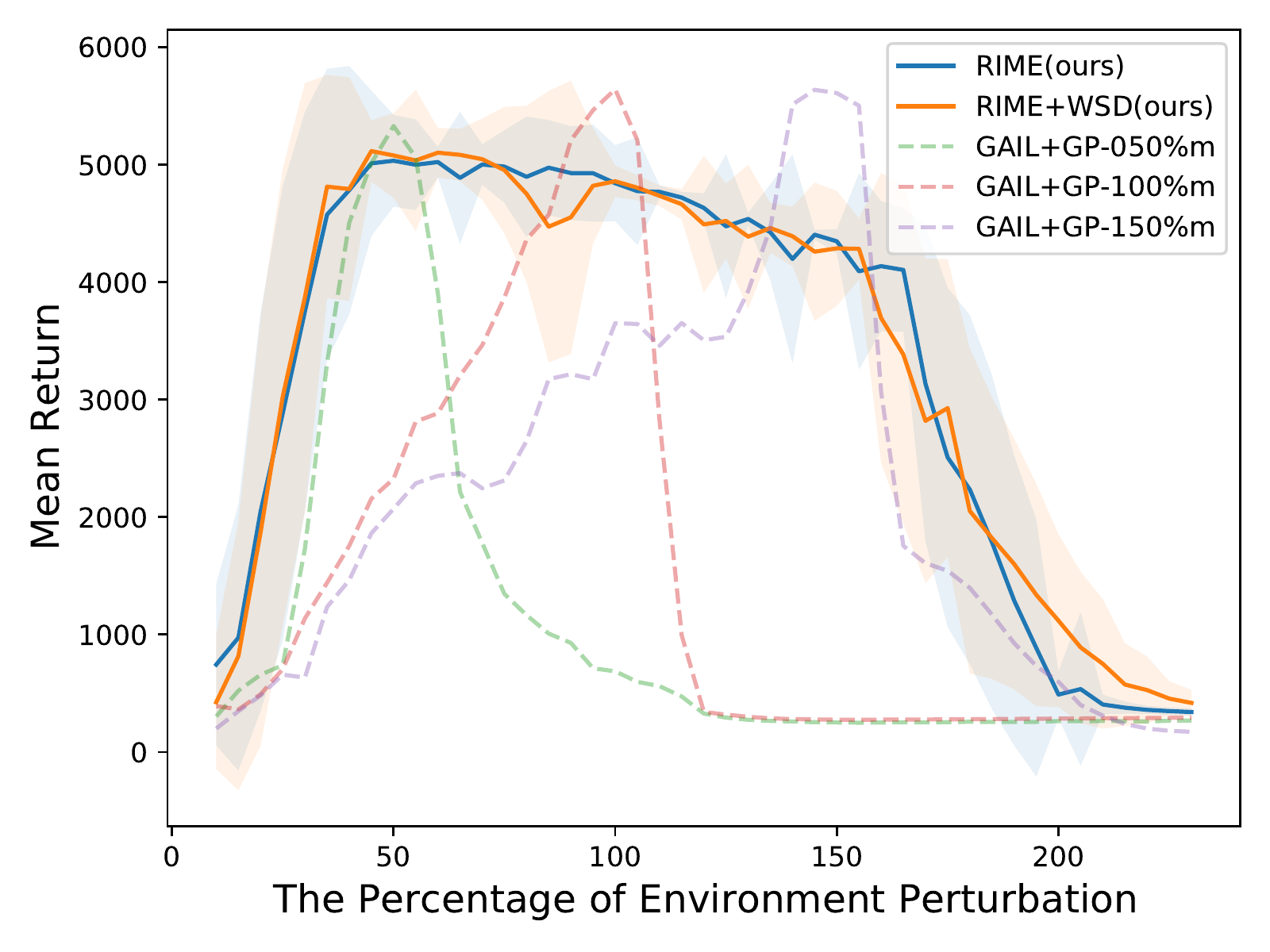}
        \captionsetup{justification=centering}
        \caption{Walker2d+Mass:\\performance}
    \end{subfigure}
    \begin{subfigure}[b]{0.24\textwidth}
        \centering
        \includegraphics[width=\textwidth]{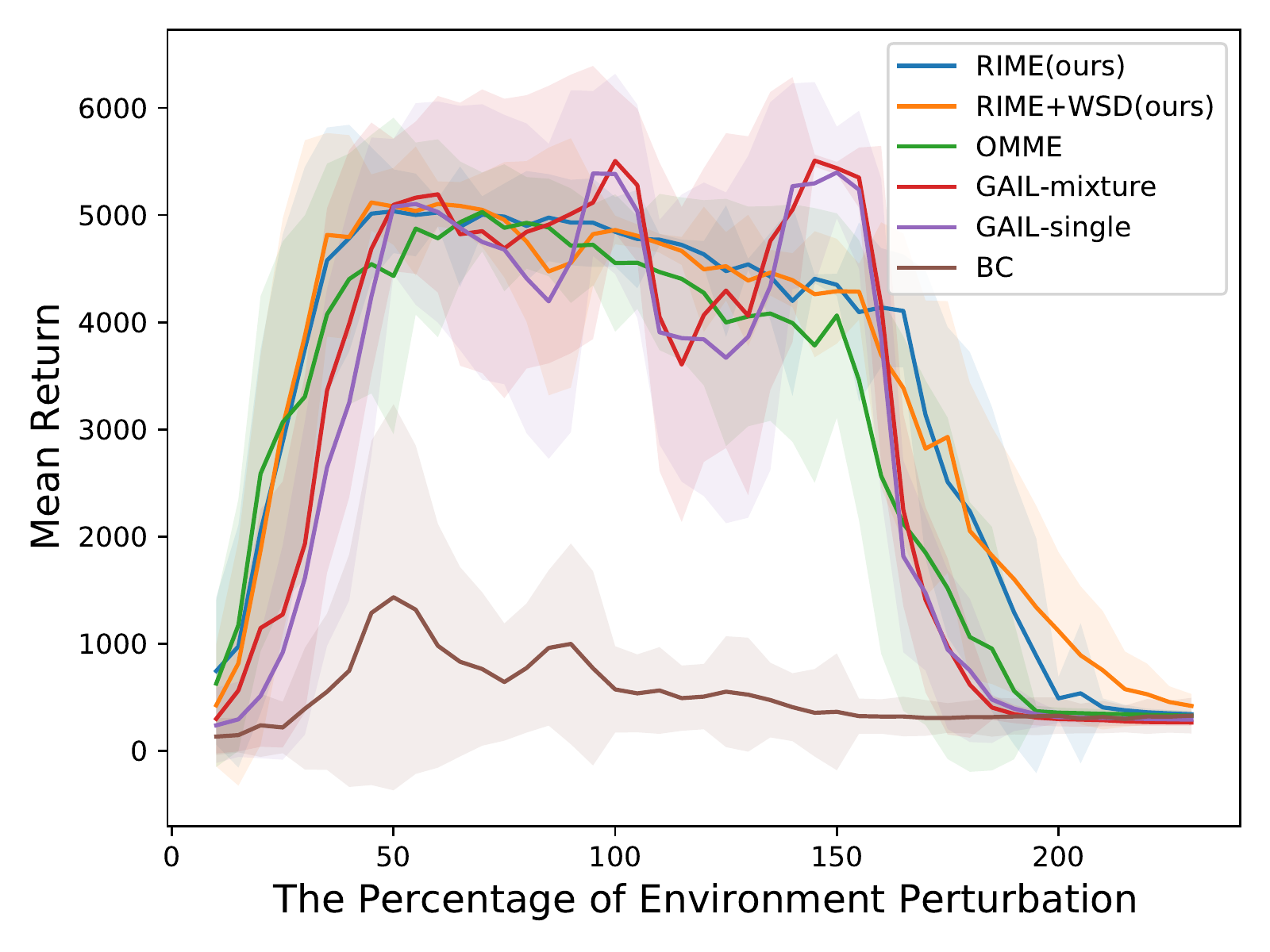}
        \captionsetup{justification=centering}
        \caption{Walker2d+Mass:\\comparisons}
    \end{subfigure}
    
    \begin{subfigure}[b]{0.24\textwidth}
        \centering
        \includegraphics[width=\textwidth]{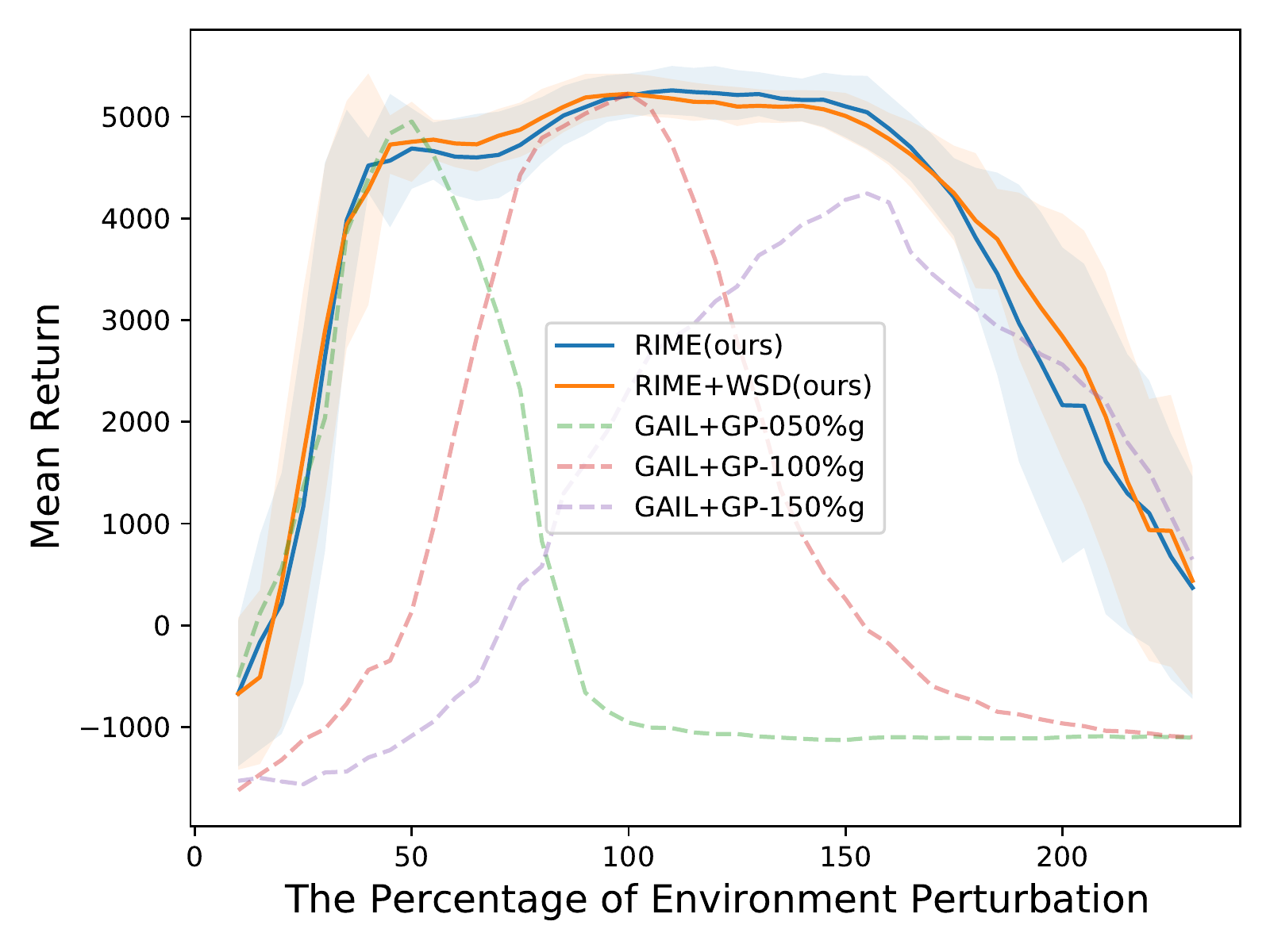}
        \captionsetup{justification=centering}
        \caption{HalfCheetah+Gravity:\\performance}
    \end{subfigure}
    \begin{subfigure}[b]{0.24\textwidth}
        \centering
        \includegraphics[width=\textwidth]{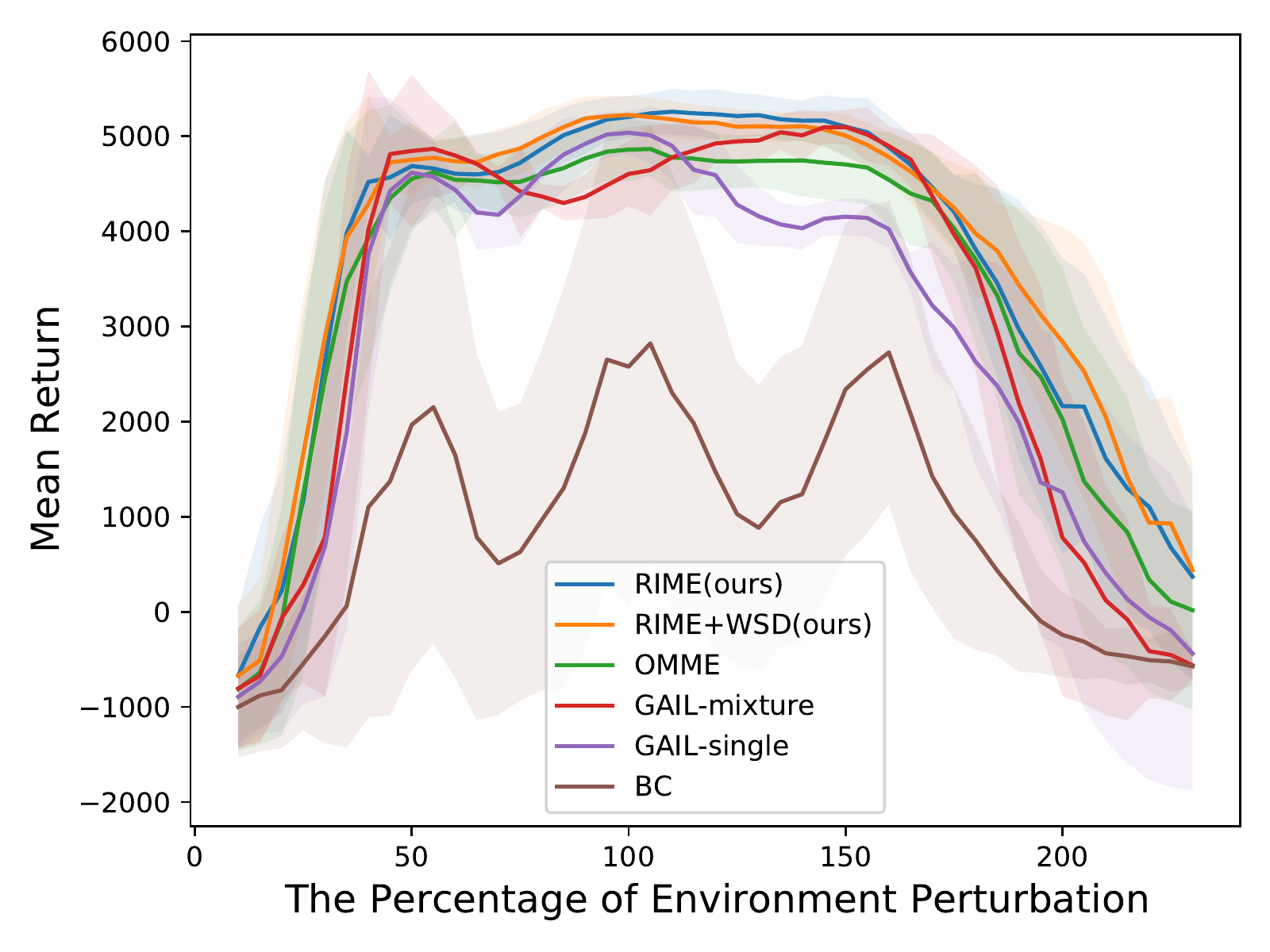}
        \captionsetup{justification=centering}
        \caption{HalfCheetah+Gravity:\\comparisons}
    \end{subfigure}
    \begin{subfigure}[b]{0.24\textwidth}
        \centering
        \includegraphics[width=\textwidth]{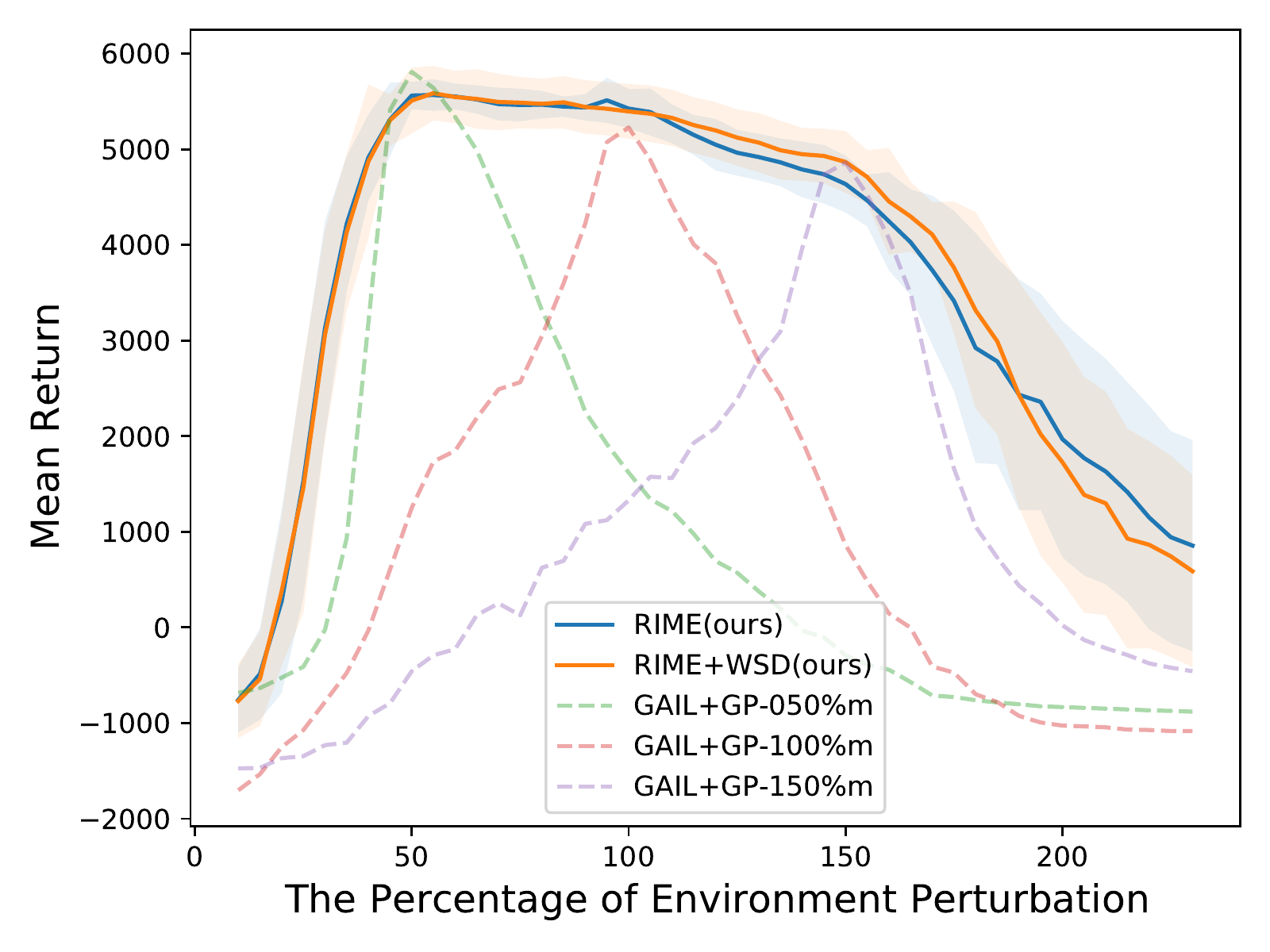}
        \captionsetup{justification=centering}
        \caption{HalfCheetah+Mass:\\performance}
    \end{subfigure}
    \begin{subfigure}[b]{0.24\textwidth}
        \centering
        \includegraphics[width=\textwidth]{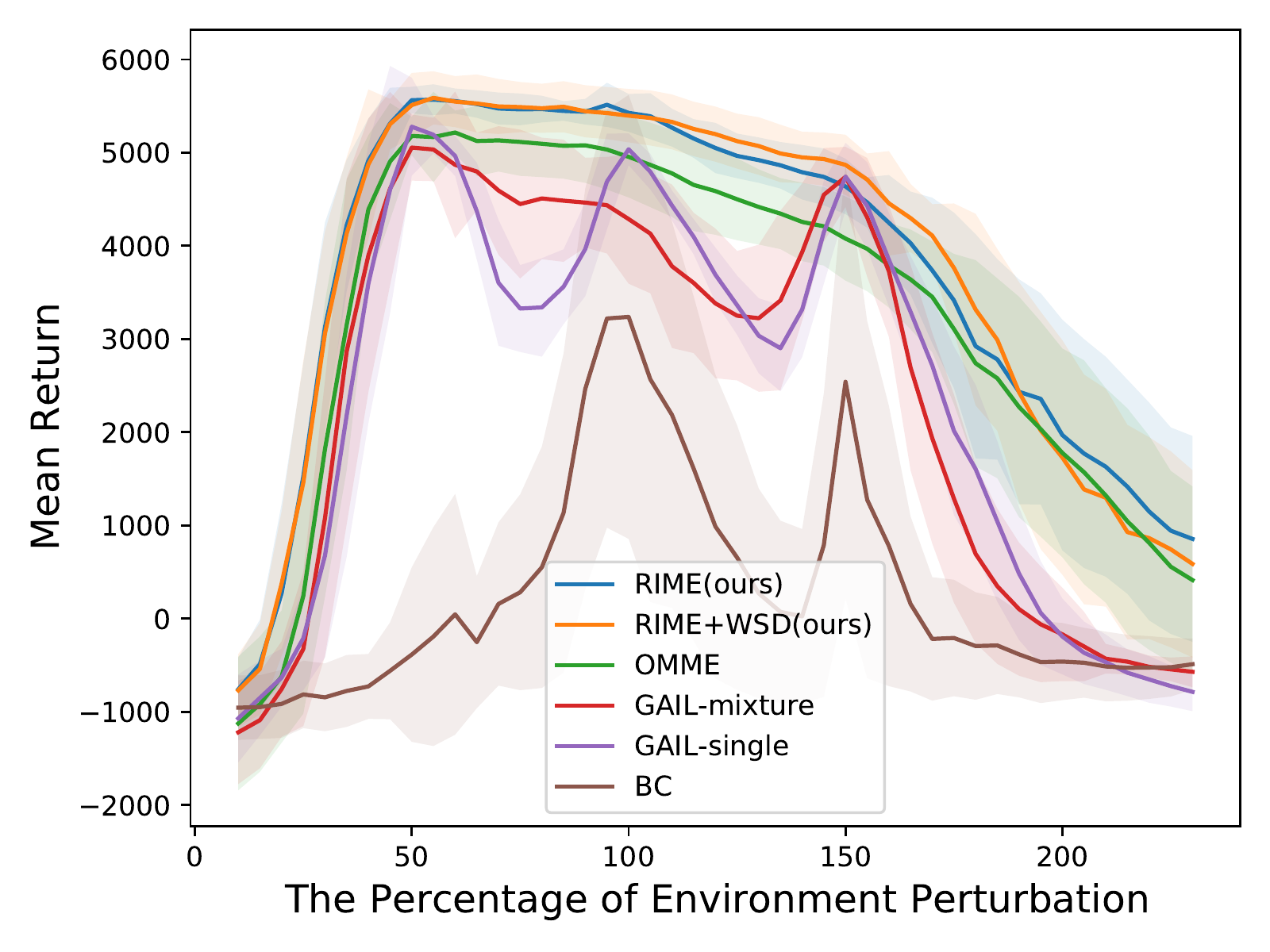}
        \captionsetup{justification=centering}
        \caption{HalfCheetah+Mass:\\comparisons}
    \end{subfigure}
    
    \begin{subfigure}[b]{0.24\textwidth}
        \centering
        \includegraphics[width=\textwidth]{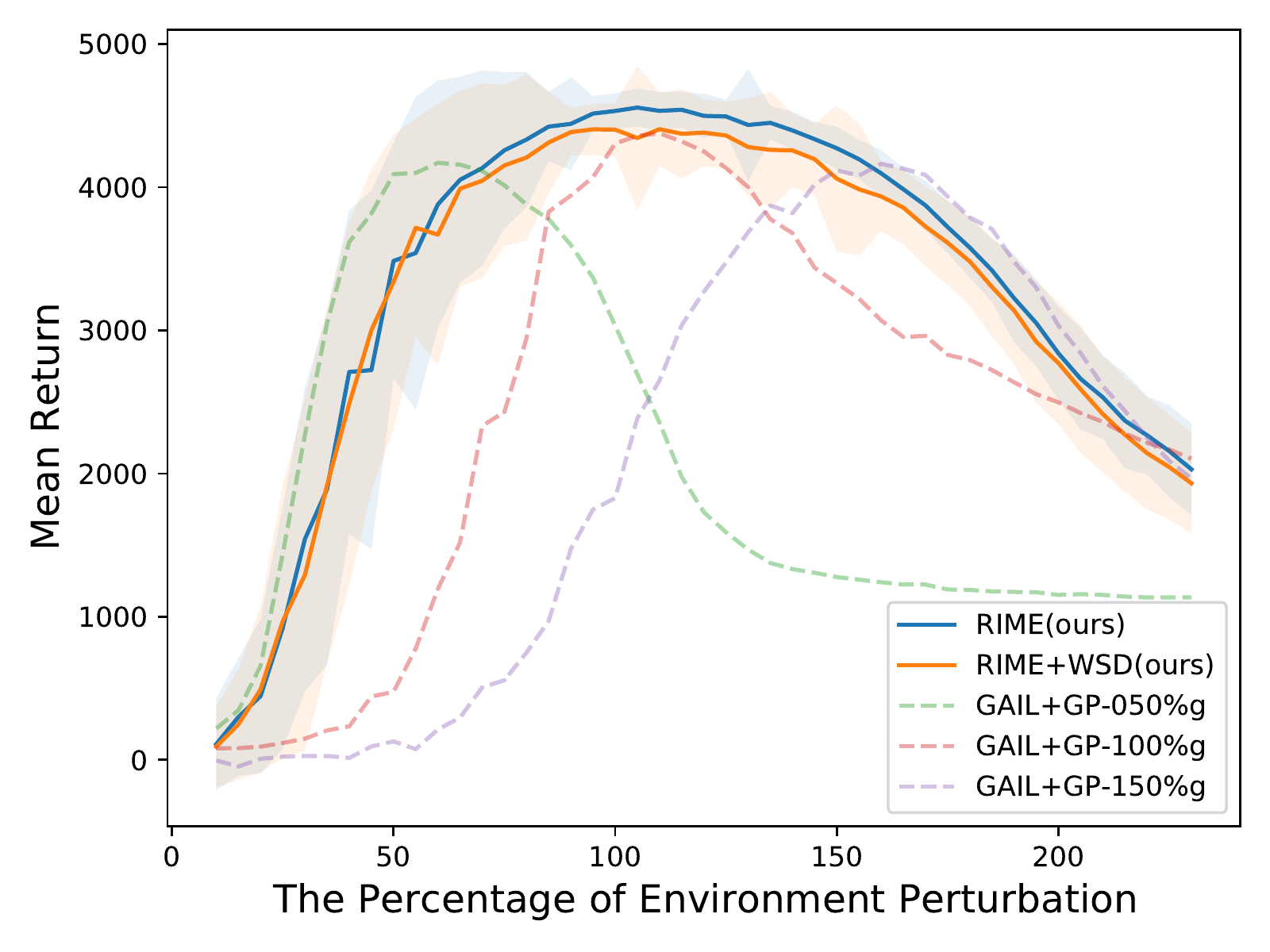}
        \captionsetup{justification=centering}
        \caption{Ant+Gravity:\\performance}
    \end{subfigure}
    \begin{subfigure}[b]{0.24\textwidth}
        \centering
        \includegraphics[width=\textwidth]{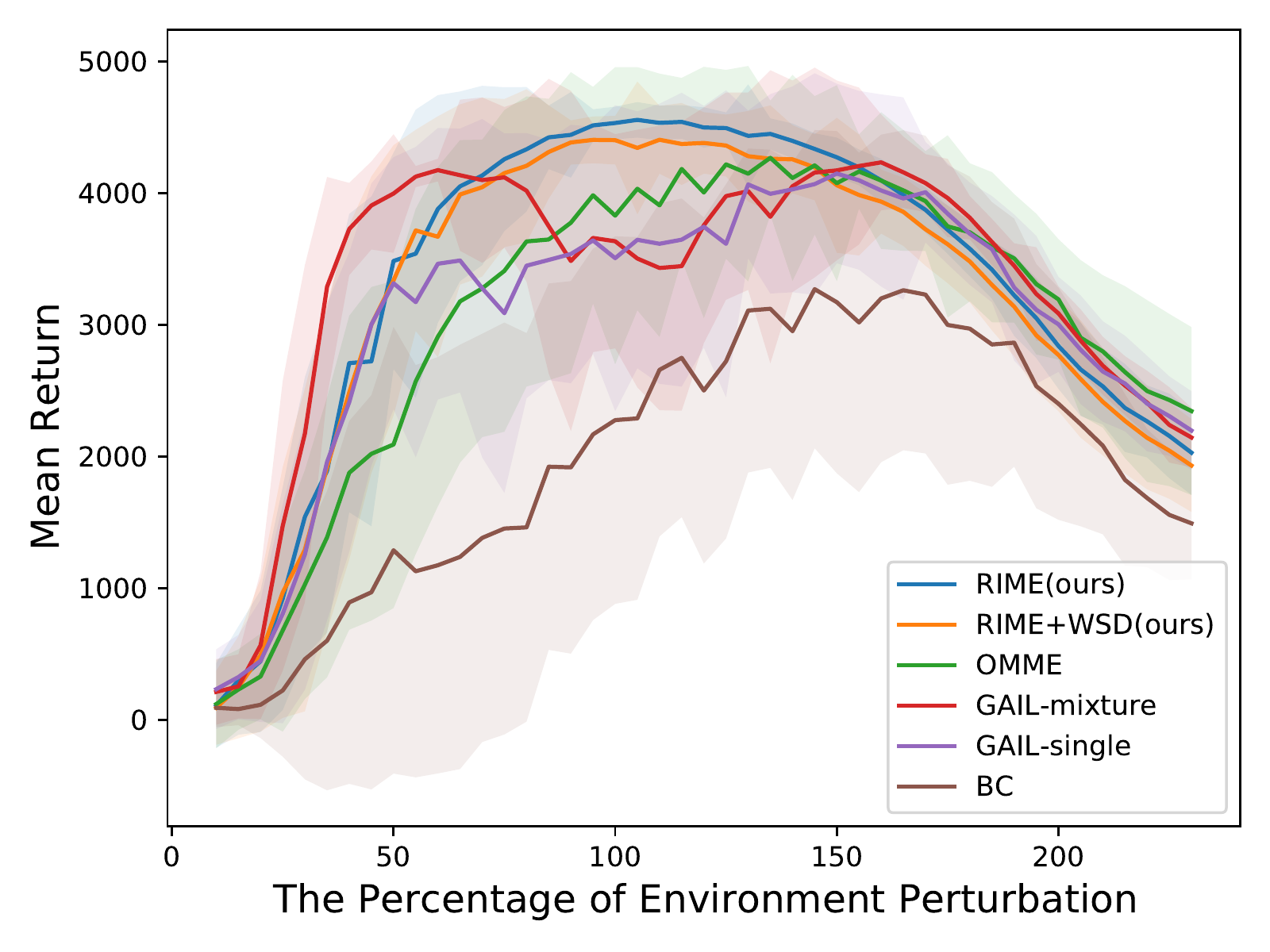}
        \captionsetup{justification=centering}
        \caption{Ant+Gravity:\\comparisons}
    \end{subfigure}
    \begin{subfigure}[b]{0.24\textwidth}
        \centering
        \includegraphics[width=\textwidth]{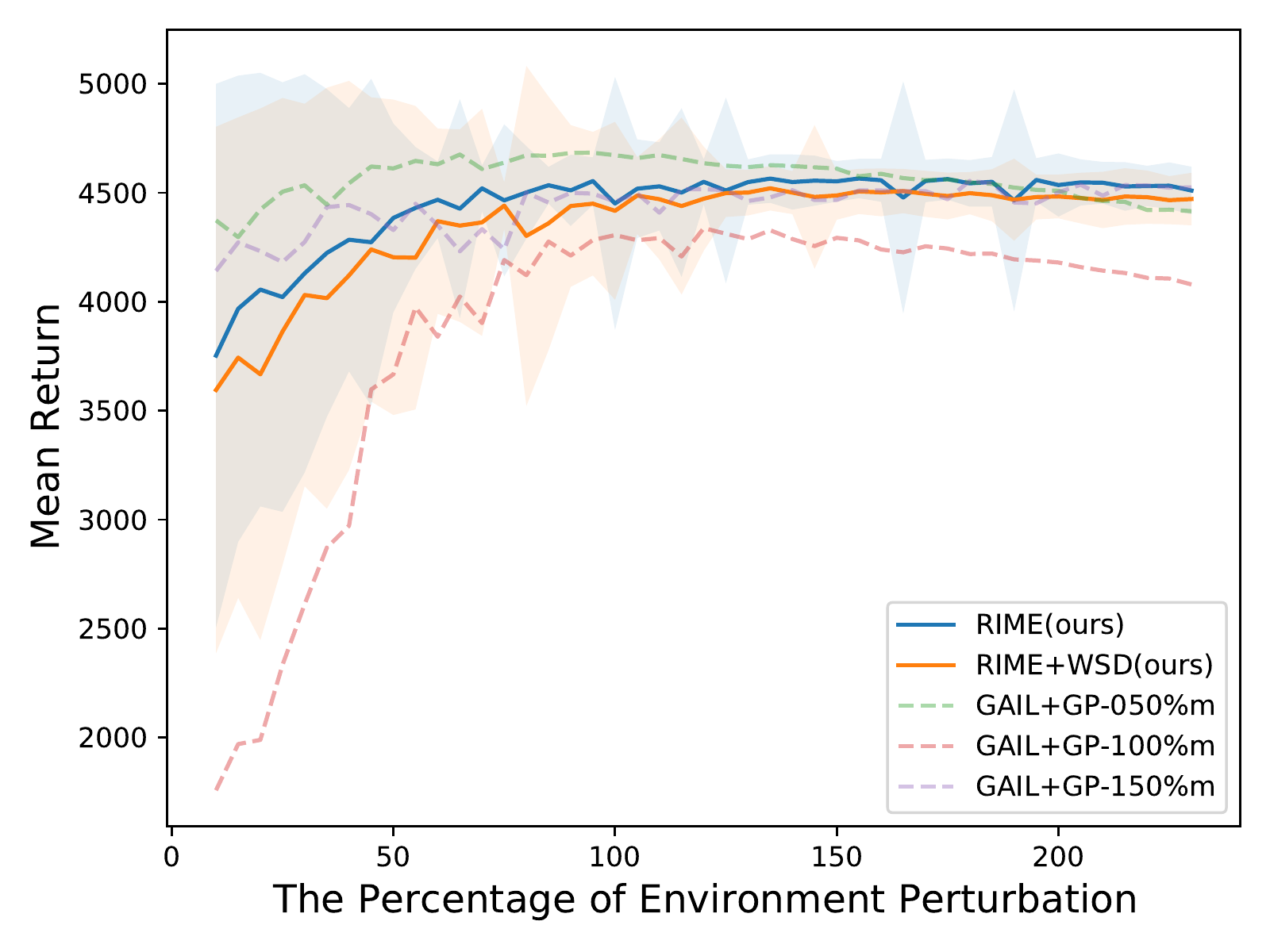}
        \captionsetup{justification=centering}
        \caption{Ant+Mass:\\performance}
    \end{subfigure}
    \begin{subfigure}[b]{0.24\textwidth}
        \centering
        \includegraphics[width=\textwidth]{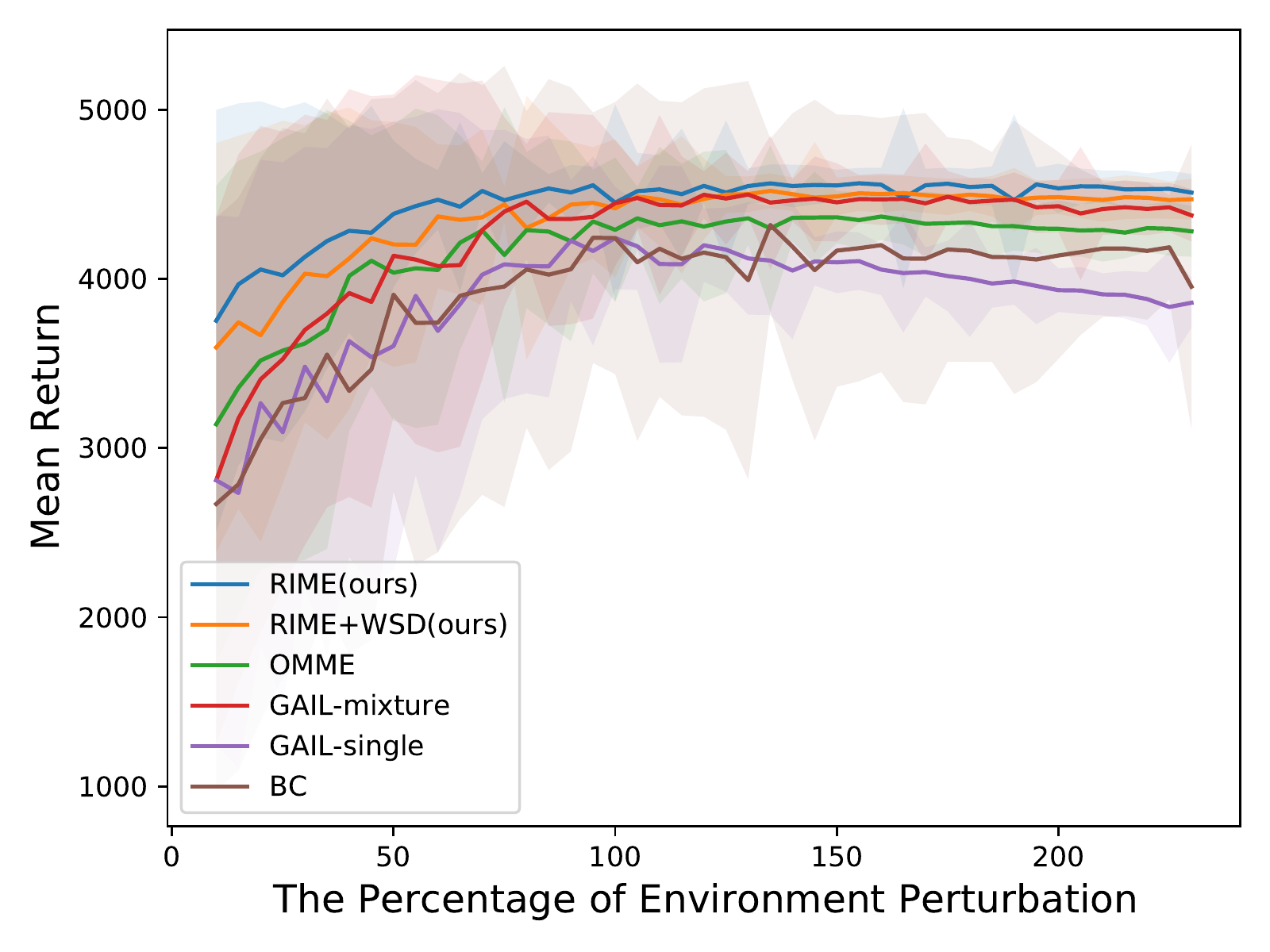}
        \captionsetup{justification=centering}
        \caption{Ant+Mass:\\comparisons}
    \end{subfigure}
    \caption{All experimental results in the $N=3$ sampled environment setting ($50\%\zeta_0,~100\%\zeta_0,~150\%\zeta_0$)\label{appendix:figures:3learningenvironments}}
\end{figure}

\newpage

\subsection{Results in the 2-D Perturbation Case}
\label{Appendix:results_in_4_learning_environmnets}

Here we provide all result plots in the 2-D perturbation case for our algorithm and the baselines.

\begin{figure*}[h]
    \centering
    \begin{subfigure}[b]{\textwidth}
        \centering
        \includegraphics[width=\textwidth]{RIME_figures/env_2par_4env_hopper.pdf}
        \captionsetup{justification=centering}
        \caption{Hopper+2-dim(G\&M)}
    \end{subfigure}
    \begin{subfigure}[b]{\textwidth}
        \centering
        \includegraphics[width=\textwidth]{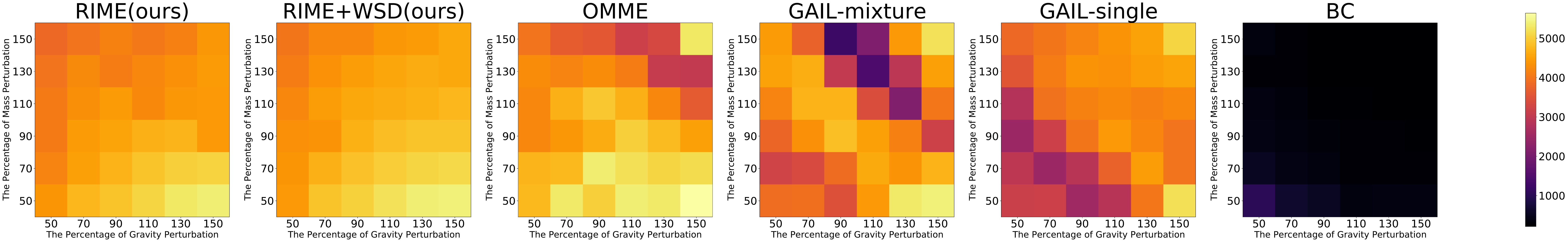}
        \captionsetup{justification=centering}
        \caption{Walker2d+2-dim(G\&M)}
    \end{subfigure}
    \begin{subfigure}[b]{\textwidth}
        \centering
        \includegraphics[width=\textwidth]{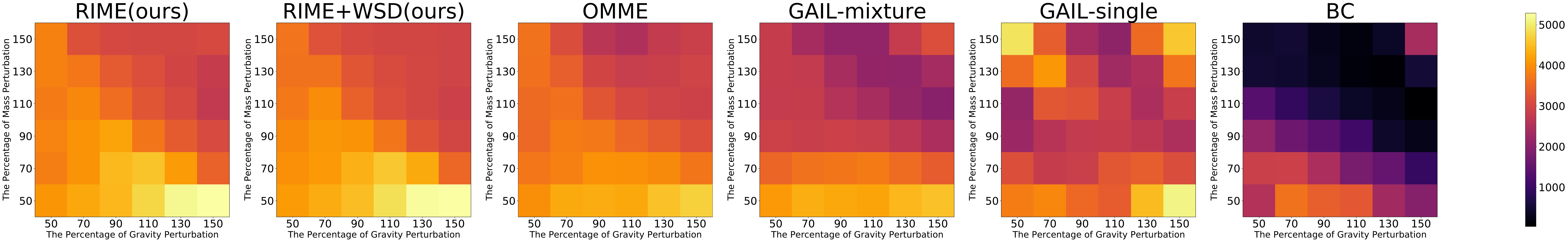}
        \captionsetup{justification=centering}
        \caption{HalfCheetah+2-dim(G\&M)}
    \end{subfigure}
    \begin{subfigure}[b]{\textwidth}
        \centering
        \includegraphics[width=\textwidth]{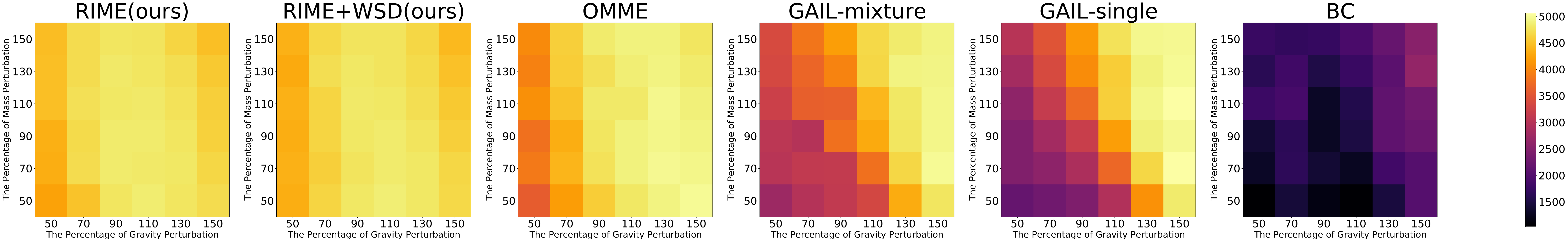}
        \captionsetup{justification=centering}
        \caption{Ant+2-dim(G\&M)}
    \end{subfigure}
    \caption{Experimental results in the $N=4$ sampled environments case with 2-dim perturbation parameters (Gravity+Mass)}
\end{figure*}

\newpage

\section{Experimental Information}
\label{appendix:experimental_information}

\subsection{Detailed Description of Other IL Baselines}
\label{Appendix:explanation_baselines}

We considered the following 4 IL baselines in \cref{section:experiments}. Here is the detailed description of these IL baselines.
\begin{enumerate}
    \item Behavior Cloning (BC): All expert demonstrations are split out $70\%$ training dataset and $30\%$ validation dataset. The policy is trained by supervised learning  until validation errors of each expert demonstration stop decreasing.
    
    \item GAIL-mixture: It is a variant of GAIL+GP applied directly to the multiple interaction environment setting. There is a single discriminator, and this discriminator distinguishes between all  $\bar{\rho}_{\pi}^i$'s and all  $\bar{\rho}_E^j$'s. The objective function of this algorithm is \eqref{appendixF:eq:GAIL-mixture}, and the objective function without GP term can be reduced to $\min_{\pi}\mathcal{D}_{JS}(\sum_i\bar{\rho}_{\pi}^i/N,\sum_j\bar{\rho}_{E}^j/N)$. 
    It minimizes the divergence between the mixture of the normalized occupancy distributions of the policy and the experts so that the mixtures are close. Thus, we call this algorithm GAIL-mixture.
    {\small
    \begin{align}
        \label{appendixF:eq:GAIL-mixture}
        \min_{\pi}\max_{D}\sum_{i=1}^{N}\Big\{\mathbb{E}_{\rho_{\pi}^i}\left[\log(1-D(s,a))\right]+\mathbb{E}_{\rho_E^i}\left[\log(D(s,a))\right]+\kappa\mathbb{E}_{\hat{x}}\left(\|\nabla_{\hat{x}} D(\hat{x})\|_2-1\right)^2\Big\}
    \end{align}
    }
    
    \item GAIL-single: It is another variant of GAIL+GP applied directly to the multiple interaction environment setting. Unlike GAIL-mixture, there are multiple discriminators. The objective function of this algorithm is \eqref{appendixF:eq:GAIL-single}, and the objective function without GP term can be reduced to $\min_{\pi}\sum_i\mathcal{D}_{JS}(\bar{\rho}_{\pi}^i,\bar{\rho}_E^i)$. 
    It minimizes the divergence between $\bar{\rho}_{\pi}^i$ and $\bar{\rho}_E^i$, which makes $\rho_{\pi}^i$ close to $\rho_E^i$, for each $i$. Thus, we call it GAIL-single.
    {\small
    \begin{align}
        \label{appendixF:eq:GAIL-single}
        \min_{\pi}\sum_{i=1}^{N}\max_{D_i}\Big\{\mathbb{E}_{\rho_{\pi}^i}\left[\log(1-D_i(s,a))\right]+\mathbb{E}_{\rho_E^i}\left[\log(D_i(s,a))\right]+\kappa\mathbb{E}_{\hat{x}}\left(\|\nabla_{\hat{x}} D_{i}(\hat{x})\|_2-1\right)^2\Big\}
    \end{align}
    }
    
    \item Occupancy measures Matching in Multiple Environments (OMME): This algorithm is a method obtained by matching occupancy measures in a different way from GAIL-mixture and GAIL-single. As mentioned in \cref{subsection:comparison_amom} \& \cref{appendix:description_for_a_matching_occupancy_measures}, if we match occupancy measures, the objective function for the policy would be $\min_{\pi}\sum_i\max_j\mathbb{E}_{\rho_{\pi}^i}[\log(1-D_{ij})]$, not $\min_{\pi}\sum_i\mathbb{E}_{\rho_{\pi}^i}[\max_j\log(1-D_{ij})]$. Except the objective function for the policy, this algorithm is the same as our algorithm.
\end{enumerate}

\subsection{Model Architecture}
\label{Appendix:model_architecture}

We developed our code based on \cite{code:pytorchppo}. In our experiments, 
we used MLP that consists of two layers with 64 cells in each layer, and this network is used for the policy. For the discriminators, we used MLP that consists of two layers with 100 cells in each layer.
We used PPO as the algorithm for updating the policy. The batch size is set to 2048, the number of update epochs for the policy at one iteration is set to 4, and the number of update epochs for the discriminator at one iteration is set to 5. Finally, the coefficient of the GP term is set to 10, and the coefficient of entropy for PPO is 0. The rest of the hyper-parameters are the same as those in \cite{RL7:PPO,RL8:GAE}.

\newpage

\subsection{Environments \& Experts}
\label{Appendix:environments_n_experts}
\begin{table}[h]
  \caption{Interaction Environments \& Expert Demonstrations}
  \label{expert_table1}
  \centering
  {\small
  \begin{tabular}{l|c|c|c|c}
    \toprule
    Task & Observation Space & Action Space & Environment Perturbation & Expert Performance \\
    \midrule
    \multirow{5}{*}{Hopper-v2} & \multirow{5}{*}{11 (Continuous)} & \multirow{5}{*}{3 (Continuous)} & 100\%g \& 100\%m & 3817.7 $\pm$ 21.9 \\
     & & & 50\%g (Gravity) & 3717.0 $\pm$ 72.1 \\
     & & & 150\%g & 3620.1 $\pm$ 4.3 \\
     & & & 50\%m (Mass) & 4415.7 $\pm$ 16.6 \\
     & & & 150\%m & 3442.0 $\pm$ 2.5 \\
    \midrule
    \multirow{5}{*}{Walker2d-v2} & \multirow{5}{*}{17 (Continuous)} & \multirow{5}{*}{6 (Continuous)} & 100\%g \& 100\%m & 5617.6 $\pm$ 16.3 \\
     & & & 50\%g & 5612.4 $\pm$ 22.1 \\
     & & & 150\%g & 5791.8 $\pm$ 46.1 \\
     & & & 50\%m & 5359.0 $\pm$ 118.9 \\
     & & & 150\%m & 5616.7 $\pm$ 14.6 \\
    \midrule
    \multirow{5}{*}{HalfCheetah-v2} & \multirow{5}{*}{17 (Continuous)} & \multirow{5}{*}{6 (Continuous)} & 100\%g \& 100\%m & 6106.3 $\pm$ 44.9 \\
     & & & 50\%g & 5396.7 $\pm$ 55.0 \\
     & & & 150\%g & 6171.4 $\pm$ 20.7 \\
     & & & 50\%m & 6159.7 $\pm$ 43.3 \\
     & & & 150\%m & 5889.9 $\pm$ 54.5 \\
    \midrule
    \multirow{5}{*}{Ant-v2} & \multirow{5}{*}{111 (Continuous)} & \multirow{5}{*}{8 (Continuous)} & 100\%g \& 100\%m & 4136.5 $\pm$ 66.1 \\
     & & & 50\%g & 3618.8 $\pm$ 102.0 \\
     & & & 150\%g & 4182.6 $\pm$ 114.7 \\
     & & & 50\%m & 4255.3 $\pm$ 97.0 \\
     & & & 150\%m & 4399.8 $\pm$ 38.3 \\
    \midrule\midrule
    Task & Observation Space & Action Space & Environment Perturbation & Expert Performance \\
    \midrule
    \multirow{4}{*}{Hopper-v2} & \multirow{4}{*}{11 (Continuous)} & \multirow{4}{*}{3 (Continuous)} & 50\%g+50\%m & 4042.7 $\pm$ 49.6 \\
     & & & 50\%g+150\%m & 3473.4 $\pm$ 41.4 \\
     & & & 150\%g+50\%m & 3789.8 $\pm$ 6.9 \\
     & & & 150\%g+150\%m & 3386.8 $\pm$ 1.6 \\
    \midrule
    \multirow{4}{*}{Walker2d-v2} & \multirow{4}{*}{17 (Continuous)} & \multirow{4}{*}{6 (Continuous)} & 50\%g+50\%m & 6478.3 $\pm$ 84.2 \\
     & & & 50\%g+150\%m & 4526.3 $\pm$ 33.0 \\
     & & & 150\%g+50\%m & 5841.6 $\pm$ 49.9 \\
     & & & 150\%g+150\%m & 5580.6 $\pm$ 12.5 \\
    \midrule
    \multirow{4}{*}{HalfCheetah-v2} & \multirow{4}{*}{17 (Continuous)} & \multirow{4}{*}{6 (Continuous)} & 50\%g+50\%m & 5015.4 $\pm$ 63.0 \\
     & & & 50\%g+150\%m & 5466.5 $\pm$ 34.5 \\
     & & & 150\%g+50\%m & 6051.8 $\pm$ 53.7 \\
     & & & 150\%g+150\%m & 5843.9 $\pm$ 63.3 \\
    \midrule
    \multirow{4}{*}{Ant-v2} & \multirow{4}{*}{111 (Continuous)} & \multirow{4}{*}{8 (Continuous)} & 50\%g+50\%m & 3818.2 $\pm$ 97.4 \\
     & & & 50\%g+150\%m & 3846.5 $\pm$ 106.6 \\
     & & & 150\%g+50\%m & 4316.7 $\pm$ 114.9 \\
     & & & 150\%g+150\%m & 4516.0 $\pm$ 77.3 \\
     \bottomrule
  \end{tabular}
  }
\end{table}


\end{document}